%% file: main.tex
\documentclass[11pt,letterpaper]{article}
\usepackage[margin=1in]{geometry}
\usepackage{authblk}
\usepackage{natbib}
\setcitestyle{authoryear,round,citesep={;},aysep={,},yysep={;}}
\renewcommand{\topfraction}{0.95}
\renewcommand{\textfraction}{0.05}

\usepackage{amsmath,amssymb,amsthm,mathtools}
\usepackage{graphicx}
\usepackage{placeins}
\usepackage{booktabs}
\usepackage{microtype}
\usepackage{xcolor}
\usepackage{tikz}
\usetikzlibrary{fit,positioning}
\usepackage{hyperref}
\usepackage{xurl}
\definecolor{linkblue}{RGB}{30,60,140}
\hypersetup{colorlinks=true,linkcolor=linkblue,citecolor=linkblue,urlcolor=linkblue}

\newcommand{\indep}{\perp\!\!\!\perp}
\newcommand{\R}{\mathbb{R}}
\newcommand{\E}{\mathbb{E}}
\newcommand{\Prob}{\mathbb{P}}
\newcommand{\indic}[1]{\mathbf{1}\{#1\}}
\newcommand{\cA}{\mathcal{A}}

\newcommand{\cF}{\mathcal{F}}
\newcommand{\cG}{\mathcal{G}}

\newcommand{\inner}[2]{\langle #1,\,#2\rangle}
\newcommand{\norm}[1]{\lVert #1\rVert}

\newcommand{\dd}{\mathrm{d}}
\DeclareMathOperator*{\argmax}{arg\,max}
\DeclareMathOperator{\sign}{sign}
\DeclareMathOperator{\softmax}{softmax}
\DeclareMathOperator{\spn}{span}

\theoremstyle{plain}
\newtheorem{theorem}{Theorem}
\newtheorem{lemma}[theorem]{Lemma}
\newtheorem{proposition}[theorem]{Proposition}

\theoremstyle{definition}

\theoremstyle{remark}
\newtheorem{remark}{Remark}

\newcommand{\titleblockfont}{\fontencoding{T1}\fontfamily{qtm}\selectfont}

\makeatletter
\renewcommand{\@maketitle}{
  \newpage
  \null
  \begingroup
  \centering
  \titleblockfont
  \hrule height 2pt
  \vspace{0.22in}
  {\fontsize{20}{25}\selectfont\scshape\@title\par}
  \vspace{0.20in}
  \hrule height 2pt
  \vspace{0.42in}
  {\fontsize{12}{15}\selectfont\@author\par}
  \vspace{0.36in}
  {\fontsize{12}{15}\selectfont\@date\par}
  \endgroup
  \vspace{0.10in}
}
\makeatother

\title{Learning Style, Forgetting Semantics:\\
A case study of SFT and RFT on Classification Tasks}

\author{Haodong Liang}
\author{Yanhao Jin}
\author{Krishnakumar Balasubramanian}
\author{Lifeng Lai}

\affil{University of California, Davis \\ 
\texttt{\fontencoding{T1}\selectfont\{hdliang, yahjin, kbala, lflai\}@ucdavis.edu}}
\date{October 1, 2026}

\begin{document}
\maketitle

\begin{abstract}
Why does supervised fine-tuning (SFT) lead to more forgetting than reinforcement fine-tuning (RFT), even when all teacher demonstrations are semantically correct? We study this question on classification tasks where tokens within each semantic class express the same semantic answer in different styles. The tasks share an underlying semantic rule but differ in their prompt distributions and teachers' stylistic preferences. Using a tractable linear-softmax policy, we derive an exact decomposition of the updates into semantic and style components. We show that, at a common policy and prompt, SFT and RFT have parallel semantic updates but differ in their style dynamics. Starting from a policy with no within-class style preference, RFT with exact policy gradients preserves this symmetry, whereas SFT with a nonuniform teacher develops off-axis style drift along a nonzero task mean under population updates. We use this drift to establish a separation under explicit conditions: for population updates from a common perfectly fitted checkpoint, SFT forgetting admits a strictly positive lower bound over a finite training interval, while RFT retains zero semantic error. Simulations over task sequences support these theoretical predictions.
\end{abstract}

\section{Introduction}\label{sec:intro}
Fine-tuning plays a central role in adapting pretrained foundation models
to downstream tasks and enhancing their task-specific capabilities
\citep{howard2018universal,devlin2019bert,raffel2020exploring}.
For example, instruction tuning on diverse demonstrations improves language models'
ability to follow instructions and generalize to unseen tasks
\citep{ouyang2022training,chung2024scaling}; reward-based fine-tuning improves
mathematical and coding performance by optimizing
model-generated responses against verifiable outcomes
\citep{shao2024deepseekmath,guo2025deepseekr1}.
However, as models learn new tasks, they may lose previously acquired
capabilities, a phenomenon known as
\emph{forgetting} \citep{wang2024comprehensive,luo2025empirical,chen2026retaining}. Preserving existing capabilities while adapting to new tasks is therefore an important consideration, highlighting the need to understand when and why fine-tuning can degrade them.

Two widely used fine-tuning methods are supervised fine-tuning (SFT),
which fits teacher demonstrations, and reinforcement fine-tuning (RFT),
which updates the model using rewards for its own sampled responses.
Empirical studies find that RFT
often preserves existing capabilities better than SFT
\citep{chen2026retaining,lai2026reinforcement}. This gap can arise even
with correct supervision: \citet{chu2025sft} observe that SFT on correct
expert demonstrations degrades existing performance under input
distribution shifts that preserve the task rules, whereas RFT from the
same checkpoint improves it. Existing theories offer distributional
explanations for such differences. \citet{shenfeld2026razor} identify
an implicit bias of online RL toward the solution closest to the
initial policy in KL divergence among those that solve the new task.
\citet{balasubramanian2026quantitative} quantify how forward- and
reverse-KL objectives affect the preservation of the
old-task component in a two-component Gaussian mixture. These accounts, however, leave a specific
mechanistic question unanswered: how does fitting a distribution of \emph{correct}
responses induce errors on previously mastered tasks?

The ParityMNIST experiments of \citet{shenfeld2026razor} make this
issue concrete: the two SFT baselines, trained on different distributions over equally correct labels, exhibit different retention trade-offs, but both preserve prior-task performance less effectively than RFT. Related
language-model studies show that SFT can teach response formats
\citep{zhou2023lima}, but that imitation can accompany declines in factual accuracy
\citep{gudibande2024false}. To understand this phenomenon, we isolate the role of these preferences
in a classification case study, separating an answer's \emph{semantic
content} from its particular expression, which we call its \emph{style}.
For example,
``4,'' ``four,'' and ``The answer is 4'' are all correct responses to
$2+2$, but a teacher may prefer one expression over the others. In our
setting, SFT fits this preference, whereas RFT rewards semantic
correctness regardless of style. This distinction motivates our
central question:
\begin{center}
\emph{Can fitting the teacher's style preference alone cause semantic
forgetting, even when every demonstration is correct?}
\end{center}

We study this question in a tractable classification model with a linear-softmax policy, where tasks share the same semantic rule but differ in their input distributions and teachers’ stylistic preferences. For this policy class, we derive an exact decomposition of the updates into semantic and style components, allowing us to trace how style fitting induces errors on a previously mastered task. Under explicit conditions and population updates, we prove that SFT causes strictly positive semantic forgetting even when every demonstration is correct, whereas RFT preserves zero error.

Our main contributions are:
\begin{itemize}
\item \textbf{An exact decomposition of the update.}
We split the policy parameters into a semantic channel and a style
channel and characterize the update direction in each channel under
SFT and RFT. The semantic updates are parallel, whereas the style
updates differ: from a policy with no style preference, RFT keeps the
style channel exactly zero throughout training, while SFT drifts along
the task mean (Theorems~\ref{thm:rate} and~\ref{thm:style}).
\item \textbf{A provable population retention gap between SFT and RFT.}
Starting from a checkpoint that perfectly fits the previous task and
has no style preference, we prove a positive lower bound on SFT
forgetting under population updates when the next task has a nonzero
prompt mean and the teacher's style preference is stronger in one semantic
class than in the other, whereas RFT retains perfect semantic accuracy
(Theorem~\ref{thm:separation}).
\item \textbf{A finite-sample bound on off-axis energy.} On top of the population-level analysis, we account for the additional
effects of finite-sample fluctuations on the off-axis parameter energy
of SFT and RFT.
Our bounds cover both semantic and style components and show that RFT
admits a smaller upper bound on the expected off-axis energy than SFT
(Theorem~\ref{thm:offaxis-energy}), and the resulting bounds are consistent with our population-level analysis in the large-sample limit.
\end{itemize}

Together, these results suggest that fitting a teacher's style
preference is by itself sufficient to cause semantic forgetting under
SFT, even when the teacher is always correct, whereas RFT avoids it
because learning from its own on-policy samples induces no style drift.
In Section~\ref{sec:experiments}, we provide empirical evidence for this mechanism by simulating the
model over multi-task sequences under population and finite-sample
updates; the results
illustrate the mechanisms in Theorems~\ref{thm:style}
and~\ref{thm:separation} and the finite-sample effects studied in
Theorem~\ref{thm:offaxis-energy}. We also report ablation studies in Appendix~\ref{app:controls} to examine the model's behavior when some of the theoretical conditions are not satisfied.

\textbf{Additional related work}: We provide a detailed discussion of more related work in Appendix~\ref{sec:related}.

\section{Preliminaries}\label{sec:preliminaries}

\subsection{Problem setup}\label{sec:setting}

We take a case study on a family of tasks that share a binary-valued semantic rule but differ in their prompt distributions and in teachers‘ preferences among semantically equivalent outputs.

\paragraph{Semantic classes.} Let $\mathbf{x}\in\R^p$ denote
a prompt representation and $\cA$ a finite output vocabulary. A map
$\sigma:\cA\to\{0,1\}$ defines two \emph{semantic classes},
$\cA_b=\sigma^{-1}(b)$, with $|\cA_0|=|\cA_1|=M\ge2$.
Tokens within a class express the same semantic answer in different styles.
A unit direction $\mathbf{u}\in\R^p$, shared by all tasks, determines
the correct class and reward:
\begin{equation}\label{eq:reward}
b^\star(\mathbf{x}) = \indic{\inner{\mathbf{u}}{\mathbf{x}}>0},
\qquad
R(a,\mathbf{x}) = \indic{\sigma(a)=b^\star(\mathbf{x})}.
\end{equation}
Thus each prompt rewards exactly one class, regardless of expression style.
Figure~\ref{fig:reward-cells} illustrates this relationship. For RFT, we assume access to an accurate semantic reward signal, which may be provided by human correctness judgments or an automatic verifier.

\textit{Notations.} Throughout our analysis, we refer to components parallel to the shared
semantic direction $\mathbf u$ as \emph{on-axis} and those orthogonal
to it as \emph{off-axis}.
We write $\mathbf v^\perp:=\mathbf v-\inner{\mathbf v}{\mathbf u}\mathbf u$
for vectors and apply this projection rowwise to matrices, i.e.,
$\mathbf V^\perp:=\mathbf V(\mathbf I_p-\mathbf u\mathbf u^\top)$.

\begin{figure}[t]
\centering
\input{Figures/fig_reward_cells.tex}
\caption{\textbf{Semantic classes and expression styles.} \emph{(a)}
The same token space is partitioned into semantic classes $\cA_0$ and $\cA_1$.
The prompt determines which class is rewarded; all tokens within a class
receive the same reward. Here, $\mathbf{x}$ denotes the prompt's vector
representation. \emph{(b)} Teachers favor different expression styles
across tasks.}
\label{fig:reward-cells}
\end{figure}
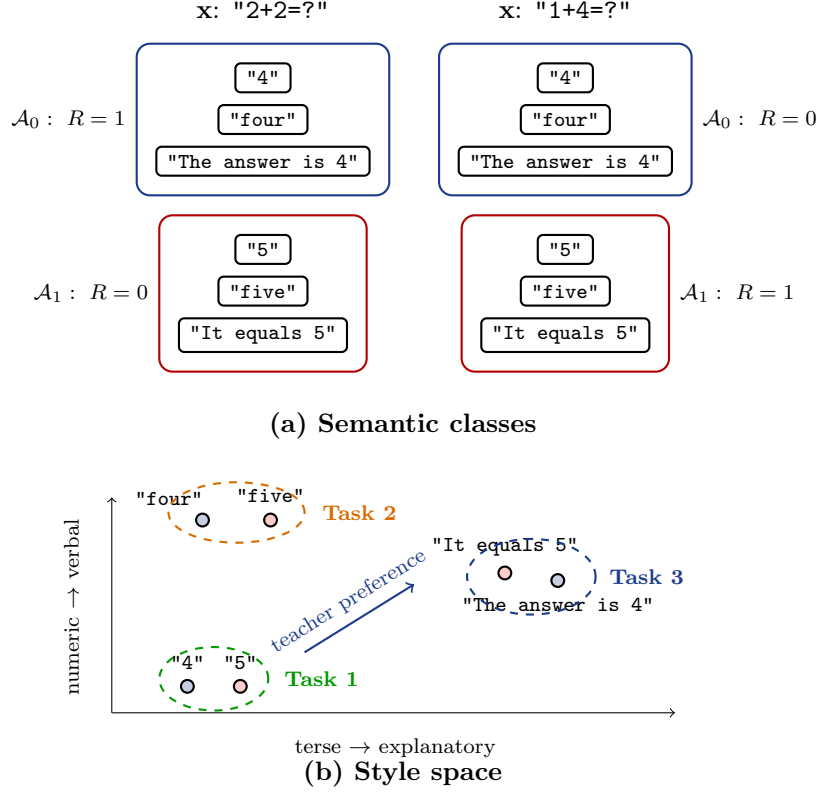

\paragraph{Tasks.} Each task $t$ specifies a prompt distribution $\mathbf{x}\sim N(\boldsymbol{\mu}_t,\mathbf{I}_p)$, and an SFT
teacher.
The teacher always selects a token from the correct semantic class, but
may favor particular styles of expressing the correct answer. For each task $t$ and semantic class $b\in\{0,1\}$, the teacher has a fixed within-class preference distribution $\rho_{t,b}$, which may differ across tasks.
The teacher distribution is
\begin{equation}\label{eq:teacher}
q_t(a\mid \mathbf{x}) =
\begin{cases}
\rho_{t,b^\star(\mathbf{x})}(a), & a\in\cA_{b^\star(\mathbf{x})},\\
0, & \text{otherwise}.
\end{cases}
\end{equation}
Thus the teacher is always semantically correct.
Throughout the analysis, we assume
$\inner{\boldsymbol{\mu}_t}{\mathbf{u}}=0$ for every task $t$, which allows task-specific prompt shifts while keeping the semantic coordinate
$\inner{\mathbf{u}}{\mathbf{x}}\sim N(0,1)$, therefore isolating the role of style
drift in semantic forgetting.

\paragraph{Policy.}
We use a linear-softmax policy, a standard parameterization in
policy-gradient reinforcement learning
\citep{sutton2000policy,kakade2001natural,agarwal2021theory}.
This linear-softmax form is also used for classification heads in supervised
fine-tuning of pretrained language models
\citep{radford2018improving,devlin2019bert}; our analysis holds the prompt
representations fixed.
The policy is parameterized by
$\mathbf{W}\in\R^{2M\times p}$, with rows $\mathbf{w}_a$:
\begin{equation}\label{eq:policy}
\pi_{\mathbf{W}}(a\mid \mathbf{x})
=\softmax(\mathbf{W}\mathbf{x})_a
=\frac{e^{\inner{\mathbf{w}_a}{\mathbf{x}}}}
{\sum_{a'\in\cA} e^{\inner{\mathbf{w}_{a'}}{\mathbf{x}}}}.
\end{equation}
The policy outputs the highest-scoring token, with semantic class
$\hat b_{\mathbf{W}}(\mathbf{x})
=\sigma\bigl(\argmax_a\inner{\mathbf{w}_a}{\mathbf{x}}\bigr)$.

Under this setup, SFT learns to match the teacher's distribution over
individual tokens, whereas RFT optimizes a reward that depends only on
the semantic class. We next specify the corresponding learning rules.

\subsection{Learning rules}\label{sec:learning-rules}

On task $t$, SFT aims to match the teacher's token distribution by minimizing
cross-entropy, whereas RFT aims to maximize the expected semantic reward.
For a prompt $\mathbf{x}$, SFT and RFT minimize the following losses, respectively:
\begin{gather}
\mathcal{L}_{\mathrm{SFT}}(\mathbf{W};\mathbf{x})
=-\E_{a^\star\sim q_t(\cdot\mid\mathbf{x})}
\bigl[\log\pi_{\mathbf{W}}(a^\star\mid\mathbf{x})\bigr],
\label{eq:sft-loss}\\
\mathcal{L}_{\mathrm{RFT}}(\mathbf{W};\mathbf{x})
=-\E_{a'\sim\pi_{\mathbf{W}}(\cdot\mid\mathbf{x})}
\bigl[R(a',\mathbf{x})\bigr].
\label{eq:rft-loss}
\end{gather}
Both methods process one prompt at a time with
step size $\eta$, starting from $\mathbf{W}=\mathbf{0}$. SFT draws a
demonstration $a^\star\sim q_t(\cdot\mid\mathbf{x})$ and takes a stochastic
gradient step on \eqref{eq:sft-loss}; RFT takes a stochastic gradient step on
\eqref{eq:rft-loss} using its \emph{exact} policy gradient.\footnote{Here, "exact" means that we compute the expected reward
by summing over all actions under the current policy and take its
exact gradient, while prompt sampling remains stochastic. We discuss the sampled REINFORCE updates
in Appendix~\ref{app:sampled}.} For every $a\in\cA$,
\begin{equation}\label{eq:updates}
\begin{aligned}
\mathbf{w}_a^+
&=\mathbf{w}_a
+\eta\bigl(\indic{a=a^\star}-\pi_{\mathbf{W}}(a\mid\mathbf{x})\bigr)
\mathbf{x}
&&\text{(SFT)},\\
\mathbf{w}_a^+
&=\mathbf{w}_a
+\eta\,\pi_{\mathbf{W}}(a\mid\mathbf{x})
\bigl(R(a,\mathbf{x})-\bar R_{\mathbf{W}}(\mathbf{x})\bigr)\mathbf{x}
&&\text{(RFT)},
\end{aligned}
\end{equation}
where $\bar R_{\mathbf{W}}(\mathbf{x})
=\sum_{a'\in\cA}\pi_{\mathbf{W}}(a'\mid\mathbf{x})R(a',\mathbf{x})$
is the expected reward. A derivation of these update rules is given in
Lemma~\ref{lem:learning-rules}.
For task $t$, we write the \emph{population} versions of the updates in
\eqref{eq:updates} in terms of the parameter matrix $\mathbf{W}$:
\begin{equation}\label{eq:pop-updates}
\begin{aligned}
\mathbf{W}^+&=\mathbf{W}+\eta\,\E_t\bigl[(q_t(\cdot\mid\mathbf{x})-\pi_{\mathbf{W}}(\cdot\mid\mathbf{x}))\mathbf{x}^\top\bigr]
&&\text{(SFT)},\\
\mathbf{W}^+&=\mathbf{W}+\eta\,\E_t\bigl[\pi_{\mathbf{W}}(\cdot\mid\mathbf{x})
(R(\cdot,\mathbf{x})-\bar R_{\mathbf{W}}(\mathbf{x}))\mathbf{x}^\top\bigr]
&&\text{(RFT)}.
\end{aligned}
\end{equation}

\subsection{Continual learning}\label{sec:continual-learning}

We train the policy sequentially on tasks $t=1,\dots,T$ and denote
its parameters after task $t$ by $\mathbf{W}_t$. We evaluate both
its final performance across tasks and its retention of earlier tasks.

For task \(t\), we define the semantic accuracy and error under its prompt distribution as
\[
a_t(\mathbf{W})
=\Prob_{\mathbf{x}\sim N(\boldsymbol{\mu}_t,\mathbf{I}_p)}
\bigl(\hat b_{\mathbf{W}}(\mathbf{x})=b^\star(\mathbf{x})\bigr),
\qquad
\mathcal{L}_t(\mathbf{W})=1-a_t(\mathbf{W}).
\]
Following standard practice in continual learning
\citep{chaudhry2018riemannian, chaudhry2019efficient,lin2023theory}, we evaluate
the final checkpoint $\mathbf{W}_T$ using the average generalization error
across all tasks and the average forgetting over earlier tasks:
\begin{equation}\label{eq:experiment-metrics}
\cG^{(T)}
=\frac{1}{T}\sum_{t=1}^{T}\mathcal{L}_t(\mathbf{W}_T),
\qquad
\cF^{(T)}
=\frac{1}{T-1}\sum_{t=1}^{T-1}
\Bigl(\max_{1\le l\le T-1}a_t(\mathbf{W}_l)-a_t(\mathbf{W}_T)\Bigr).
\end{equation}
For each earlier task, forgetting compares its final accuracy with
its best accuracy over all preceding task checkpoints. Negative values of $\cF^{(T)}$ indicate
positive backward transfer.

\section{Main Results}\label{sec:main}
In this section, we compare SFT and RFT by decomposing the policy parameters
into semantic and style channels (Section~\ref{sec:channels}).
We first show that, at the same policy and prompt, the two methods update
the semantic channel in the same direction (Theorem~\ref{thm:rate}).
Despite this alignment, their style dynamics differ: starting from a
policy uniform within each class, RFT preserves zero style deviation, whereas SFT
develops population style drift along a nonzero task mean when the
teacher has within-class style preferences (Theorem~\ref{thm:style}).
We then show that this style drift can produce a strict separation in
forgetting between SFT and RFT at the population level
(Theorem~\ref{thm:separation}).
Finally, we complement the population analysis with finite-sample bounds on the off-axis drifting energy of both methods under fresh sampling
(Theorem~\ref{thm:offaxis-energy}).

\subsection{Semantic--style decomposition}\label{sec:channels}

We decompose the policy parameters into a \emph{semantic channel}
$\boldsymbol{\alpha}=(\boldsymbol{\alpha}^{(0)},\boldsymbol{\alpha}^{(1)})$,
which represents the mean weight associated with each semantic class,
and a \emph{style channel}
$\boldsymbol{\beta}=(\boldsymbol{\beta}_a)_{a\in\cA}$,
which captures the policy's preferences among semantically equivalent
tokens within each class.
Specifically, for $b\in\{0,1\}$,
\begin{equation}\label{eq:channels}
\boldsymbol{\alpha}^{(b)} = \frac1{M}\sum_{a\in\cA_b} \mathbf{w}_a,
\qquad
\boldsymbol{\beta}_a = \mathbf{w}_a - \boldsymbol{\alpha}^{(\sigma(a))},
\qquad \text{so } \sum_{a\in\cA_b}\boldsymbol{\beta}_a = \mathbf{0}.
\end{equation}
Each token's score then splits as 
\[
\inner{\mathbf{w}_a}{\mathbf{x}}
=\inner{\boldsymbol{\alpha}^{(\sigma(a))}}{\mathbf{x}}
+\inner{\boldsymbol{\beta}_a}{\mathbf{x}}.
\]
Using this decomposition, the policy 
\eqref{eq:policy} can then be written as
\begin{equation}\label{eq:channel-policy}
\pi_{\mathbf{W}}(a\mid\mathbf{x})
=\frac{\exp\!\bigl(\inner{\boldsymbol{\alpha}^{(\sigma(a))}}{\mathbf{x}}
+\inner{\boldsymbol{\beta}_a}{\mathbf{x}}\bigr)}
{\displaystyle\sum_{b=0}^1\sum_{a'\in\cA_b}
\exp\!\bigl(\inner{\boldsymbol{\alpha}^{(b)}}{\mathbf{x}}
+\inner{\boldsymbol{\beta}_{a'}}{\mathbf{x}}\bigr)}.
\end{equation}

Let
$\Delta\boldsymbol{\alpha}
:=\boldsymbol{\alpha}^{(1)}-\boldsymbol{\alpha}^{(0)}$.
For any $\mathbf{W}$, the decision rule is
\begin{gather}
D_{\mathbf{W}}(\mathbf{x})
= \inner{\Delta\boldsymbol{\alpha}}{\mathbf{x}} + \xi_{\boldsymbol{\beta}}(\mathbf{x}),
\qquad
\hat b_{\mathbf{W}}(\mathbf{x}) = \indic{D_{\mathbf{W}}(\mathbf{x})>0},\notag\\
\xi_{\boldsymbol{\beta}}(\mathbf{x})
= \max_{a\in\cA_1}\inner{\boldsymbol{\beta}_a}{\mathbf{x}}
             - \max_{a\in\cA_0}\inner{\boldsymbol{\beta}_a}{\mathbf{x}}.
\label{eq:decision}
\end{gather}
Thus, a within-class style preference can change the semantic prediction
when $\xi_{\boldsymbol{\beta}}(\mathbf{x})$ reverses the sign of
$\inner{\Delta\boldsymbol{\alpha}}{\mathbf{x}}$.
If $\Delta\boldsymbol{\alpha}, \boldsymbol{\beta}$ are on-axis, i.e.,
$\Delta\boldsymbol{\alpha}=\gamma\mathbf{u}$ and
$\boldsymbol{\beta}_a=c_a\mathbf{u}$ for every token $a$, the policy's
decision boundary coincides with the true semantic boundary whenever
$\gamma>2\max_a|c_a|$ (Lemma~\ref{lem:on-axis-classification}).
In particular, for a class-uniform policy ($\boldsymbol{\beta}=\mathbf{0}$),
this decision boundary is determined solely by $\Delta\boldsymbol{\alpha}$.
We next compare how SFT and RFT update these two channels, beginning with
the semantic component.

\subsection{The semantic channel: a shared update direction}\label{sec:semantic}

Define the cumulative probability assigned to semantic class $b$ by
\begin{equation}\label{eq:class-prob}
P_b(\mathbf{W},\mathbf{x}):=\sum_{a\in\cA_b}\pi_{\mathbf{W}}(a\mid\mathbf{x}).
\end{equation}
We first show that SFT and RFT update the semantic channel in
the same direction.

\begin{theorem}[Shared Semantic  direction]\label{thm:rate}
Under the setup of Section~\ref{sec:setting}, the following hold.
\begin{enumerate}
\item[\emph{(i)}] \emph{(per-sample identity)} For any $\mathbf{W}$ and
prompt $\mathbf{x}$, the semantic channel moves as
\begin{equation}\label{eq:semantic-steps}
\begin{aligned}
\Delta\boldsymbol{\alpha}^+_{\mathrm{SFT}} - \Delta\boldsymbol{\alpha}
&= \frac{2\eta}{M}\,\bigl(2b^\star(\mathbf{x})-1\bigr)
\bigl(1-P_{b^\star(\mathbf{x})}(\mathbf{W},\mathbf{x})\bigr)\mathbf{x},\\
\Delta\boldsymbol{\alpha}^+_{\mathrm{RFT}} - \Delta\boldsymbol{\alpha}
&= P_{b^\star(\mathbf{x})}(\mathbf{W},\mathbf{x})\cdot
\bigl(\Delta\boldsymbol{\alpha}^+_{\mathrm{SFT}}-\Delta\boldsymbol{\alpha}\bigr).
\end{aligned}
\end{equation}
\item[\emph{(ii)}] \emph{(population direction)} If $\boldsymbol{\beta}=\mathbf{0}$ and
$\Delta\boldsymbol{\alpha}=c\mathbf{u}$ for some $c\in\mathbb{R}$, then
the task-$t$ population semantic steps are parallel to $\mathbf{u}$:
\begin{equation}\label{eq:pop-alpha-task-mean}
\begin{aligned}
\E_t\bigl[\Delta\boldsymbol{\alpha}^+_{\mathrm{SFT}}-\Delta\boldsymbol{\alpha}\bigr]
&=\frac{2\eta}{M}\sqrt{\frac{2}{\pi}}
\left(\int_0^\infty\frac{z e^{-z^2/2}}{1+e^{cz}}\,\dd z\right)\mathbf{u},\\
\E_t\bigl[\Delta\boldsymbol{\alpha}^+_{\mathrm{RFT}}-\Delta\boldsymbol{\alpha}\bigr]
&=\frac{2\eta}{M}\sqrt{\frac{2}{\pi}}
\left(\int_0^\infty\frac{z e^{-z^2/2}e^{cz}}{(1+e^{cz})^2}\,\dd z\right)\mathbf{u}.
\end{aligned}
\end{equation}
\end{enumerate}
\end{theorem}
\noindent The proof of Theorem~\ref{thm:rate} is in Appendix~\ref{app:proof-rate}.

In particular, at a fixed policy and prompt, the SFT semantic update
in \eqref{eq:semantic-steps} does not depend on which style token the
teacher sampled, while the RFT update is rescaled by the probability
assigned to the correct class. Moreover, both coefficients in
\eqref{eq:pop-alpha-task-mean} are strictly positive, so both population
updates increase the semantic margin along $\mathbf{u}$. Despite this shared semantic update direction, the two methods
differ in how they update the style channel, as we show next.

\subsection{The style channel: RFT freezes, SFT drifts}\label{sec:style}

We next examine whether SFT and RFT preserve uniformity within each
semantic class. The following theorem shows that RFT preserves zero
style deviation, whereas SFT can develop population style drift
when the teacher has within-class preferences.

\begin{samepage}
\begin{theorem}[Style dynamics]\label{thm:style}
Under the setup of Section~\ref{sec:setting}, the following hold.
\begin{enumerate}
\item[\emph{(i)}] \emph{(RFT: style invariance)} The RFT updates
\eqref{eq:updates} preserve $\boldsymbol{\beta}\equiv\mathbf{0}$.
\item[\emph{(ii)}] \emph{(SFT: style drift)} For any $\mathbf{W}$ and prompt $\mathbf{x}$,
the SFT update \eqref{eq:updates} gives
\begin{equation}\label{eq:betastep}
\boldsymbol{\beta}_a^+ - \boldsymbol{\beta}_a
= \eta\left[
\indic{a=a^\star}
-\pi_{\mathbf{W}}(a\mid\mathbf{x})
-\frac{\indic{\sigma(a)=b^\star(\mathbf{x})}
-P_{\sigma(a)}(\mathbf{W},\mathbf{x})}{M}
\right]\mathbf{x}.
\end{equation}
If $\boldsymbol{\beta}=\mathbf{0}$,
the task-$t$ population SFT step
\eqref{eq:pop-updates} satisfies, for every $a\in\cA_b$,
\begin{equation}\label{eq:pop-drift-full}
\begin{aligned}
\E_t[\boldsymbol{\beta}_a^+-\boldsymbol{\beta}_a]
&=
\eta\Bigl(\rho_{t,b}(a)-\frac1{M}\Bigr)
\left(
\frac12\boldsymbol{\mu}_t+\frac{2b-1}{\sqrt{2\pi}}\mathbf{u}
\right),
\\
\text{so that}\qquad
\E_t[\boldsymbol{\beta}_a^+-\boldsymbol{\beta}_a]^{\perp}
&=\frac{\eta}{2}\Bigl(\rho_{t,b}(a)-\frac1{M}\Bigr)\boldsymbol{\mu}_t.
\end{aligned}
\end{equation}
\end{enumerate}
\end{theorem}
\end{samepage}
\noindent The proof of Theorem~\ref{thm:style} is in Appendix~\ref{app:proof-style}.

Theorem~\ref{thm:style} gives us intuition for why SFT can be more prone
to forgetting than RFT: starting from a class-uniform policy
($\boldsymbol{\beta}=\mathbf{0}$), SFT can develop
off-axis population drift along the task mean by fitting the teacher's style
preferences, while RFT preserves within-class symmetry.
A special case is when either teachers have no within-class style
preferences (i.e., $\rho_{t,b}(a)=1/M$) or task means are zero
(i.e., $\boldsymbol{\mu}_t=\mathbf{0}$), in which case both methods
avoid off-axis population style drift.
When drift does arise, it can shift the semantic decision boundary
through $\xi_{\boldsymbol{\beta}}$ in \eqref{eq:decision}.
In Section~\ref{sec:separation}, we show that this mechanism causes
forgetting under explicit conditions.

\subsection{From Style Drift to Semantic Error}\label{sec:separation}

We now compare how SFT and RFT retain a previously learned task when
trained on the next one. Both methods start from the same checkpoint,
which perfectly classifies the previous task and whose policy is uniform
within each semantic class. The following theorem establishes a
finite-time separation in semantic error on both the previous and
current tasks.

\begin{theorem}[Population error separation from a perfectly fitted checkpoint]
\label{thm:separation}
Under the setup of Section~\ref{sec:setting}, suppose that:
\begin{enumerate}
\item[\emph{(a)}] \emph{(Common start)} SFT and RFT start from the same
checkpoint $\mathbf{W}_1$, which fits the previous task perfectly and
has no style preference:
\[
\mathcal{L}_1(\mathbf{W}_1)=0,
\qquad \boldsymbol{\beta}(\mathbf{W}_1)=\mathbf{0}.
\]
\item[\emph{(b)}] \emph{(Next task)} The next task has a nonzero 
mean $\boldsymbol{\mu}_2\neq\mathbf{0}$, and its teacher's
style preference is
stronger in one semantic class than in the other, after relabeling the two
classes if necessary:
$\Delta_\rho:=\max_{a\in\cA_1}\rho_{2,1}(a)-\max_{a\in\cA_0}\rho_{2,0}(a)>0$.
\item[\emph{(c)}] \emph{(Training)} Each method runs $k_2\ge1$
population updates \eqref{eq:pop-updates} on task $2$ with step size
$\eta_2>0$; write $\tau_2:=k_2\eta_2$ for the effective training time and
$\mathbf{W}_2$ for the resulting checkpoint.
\end{enumerate}
\pagebreak[3]
Then:
\begin{enumerate}
\item[\emph{(i)}] \emph{(RFT preserves perfect accuracy on both tasks)}
For every $\eta_2>0$ and $k_2\ge1$,
\[
\mathcal{L}_1(\mathbf{W}_2)=\mathcal{L}_2(\mathbf{W}_2)=0.
\]
\item[\emph{(ii)}] \emph{(SFT incurs semantic error on both tasks)}
There is a threshold $\tau_0>0$, given in \eqref{eq:tau0} that depends only on
$\norm{\boldsymbol{\mu}_2}$, $\Delta_\rho$, and $M$, such that for every
$0<\tau_2\le\tau_0$ and $r\in\{1,2\}$,
\[
\mathcal{L}_r(\mathbf{W}_2)
\ge\int_0^\infty\phi(t)\Phi\!\left(
\frac{\inner{\boldsymbol{\mu}_r}{\boldsymbol{\mu}_2}}{\norm{\boldsymbol{\mu}_2}}
-\frac{4(\norm{\Delta\boldsymbol{\alpha}(\mathbf{W}_1)}+2\tau_2)t}{\norm{\boldsymbol{\mu}_2}\Delta_\rho\tau_2}
\right)\,\dd t>0,
\]
where $\phi$ and $\Phi$ denote the standard Gaussian density and distribution
function, respectively.
\end{enumerate}
\end{theorem}
\noindent The proof of Theorem~\ref{thm:separation} is in Appendix~\ref{app:proof-separation}.

Theorem \ref{thm:separation}(ii) implies that SFT
introduces errors on two tasks that were initially classified perfectly.
The lower bound measures the probability of a region where style drift
overturns the correct semantic prediction. This error region grows with stronger teacher's preference $\Delta_\rho$, smaller initial semantic margin $\norm{\Delta\boldsymbol{\alpha}(\mathbf{W}_1)}$, or longer training time $\tau_2$. 
Old-task prompts are more likely to fall into this region
when $\boldsymbol{\mu}_1$ has a larger projection onto the direction of the new task mean $\boldsymbol{\mu}_2$. 

Together with \eqref{eq:experiment-metrics}, these error bounds imply strictly
positive forgetting and generalization error for SFT when
$0<\tau_2\le\tau_0$, whereas both metrics remain zero for RFT. This separation is, however, a finite-time phenomenon: under continued
population SFT updates with a sufficiently small step size,
both task errors converge to zero as $\tau_2\to\infty$,
as shown in Appendices~\ref{app:limiting-behavior}
and~\ref{app:two-task-dynamics}.

\subsection{Finite-sample off-axis energy}\label{sec:offaxis-energy}

The separation result in Theorem~\ref{thm:separation} is established
under population updates. In finite-sample training, prompt sampling
can introduce additional deviations from the shared semantic direction
$\mathbf u$. In the next theorem, we provide a bound on the expected
off-axis energy $\E\norm{\mathbf W^\perp}_F^2$ of both SFT and RFT.

\begin{theorem}[Off-axis parameter energy of SFT and RFT]
\label{thm:offaxis-energy}
Under the setup of Section~\ref{sec:setting}, let $t(s)$ denote the task
index at step $s$. Run the SFT and RFT
updates \eqref{eq:updates} from $\mathbf W(0)=\mathbf0$ with constant
step size $\eta>0$.
At each step $s$, independently draw
$\mathbf x_s\sim N(\boldsymbol\mu_{t(s)},\mathbf I_p)$, with a fresh
teacher demonstration for SFT. For every $N\ge0$, both methods satisfy
\begin{equation}\label{eq:offaxis-energy-decomposition}
\norm{\mathbf W^\perp(N)}_F^2
=\norm{\boldsymbol\beta^\perp(N)}_F^2
+\frac M2\norm{\Delta\boldsymbol\alpha^\perp(N)}^2,
\end{equation}
and the following bounds hold:
\begin{enumerate}
\item[\emph{(i)}] \emph{(RFT)}
\begin{equation}\label{eq:rft-offaxis-energy}
\E\norm{\mathbf W_{\mathrm{RFT}}^\perp(N)}_F^2
\le\frac{\eta^2}{8M}\left(N(p-1)
+\sum_{s=0}^{N-1}\norm{\boldsymbol\mu_{t(s)}}^2\right).
\end{equation}
\item[\emph{(ii)}] \emph{(SFT)}
\begin{equation}\label{eq:offaxis-energy}
\E\norm{\mathbf W_{\mathrm{SFT}}^\perp(N)}_F^2
\le 2\eta^2\left(\sqrt{N(p-1)}
+\sum_{s=0}^{N-1}\norm{\boldsymbol\mu_{t(s)}}\right)^2.
\end{equation}
\end{enumerate}
\end{theorem}
\noindent The proof of Theorem~\ref{thm:offaxis-energy} is in
Appendix~\ref{app:proof-offaxis-energy}.

Theorem~\ref{thm:offaxis-energy} shows that, even under finite sampling,
RFT admits a smaller upper bound on the expected off-axis energy
than SFT.
The following remark discusses how these bounds reflect sampling
effects and relate to the population dynamics.

\begin{remark}[Sampling effects and population drift]
We discuss the following cases as $\eta\to0$ with effective training time
$\tau=N\eta$ fixed. This keeps the cumulative step size unchanged while
increasing the number of independent samples.
\begin{itemize}
\item[\emph{(i)}] \emph{Zero task means.}
If $\boldsymbol\mu_t=\mathbf0$ for all tasks, the RFT and SFT bounds
reduce to $\tau\eta(p-1)/(8M)$ and $2\tau\eta(p-1)$, respectively.
Both bounds vanish in the limit $\eta\to0$. This is consistent with
Lemma~\ref{lem:rotation-equivariance}, which shows that the expected
parameters remain on the semantic axis for both methods when all
task means are zero.
\item[\emph{(ii)}] \emph{Bounded nonzero task means.}
The RFT bound \eqref{eq:rft-offaxis-energy} remains $\mathcal{O}(\tau\eta)$
and vanishes, consistent with its on-axis population dynamics
(Theorems~\ref{thm:rate}(ii) and~\ref{thm:style}(i)). The SFT bound \eqref{eq:offaxis-energy}, however,
does not vanish, since the term $\eta\sum_{s=0}^{N-1}\norm{\boldsymbol\mu_{t(s)}}$
remains nonzero. This is consistent with the population style drift dynamics identified in
Theorem~\ref{thm:style}(ii).
\end{itemize}

\end{remark}

Together, these results show how fitting style can interfere with shared
semantic knowledge even when all demonstrations are semantically correct.
At a common policy and prompt, SFT and RFT share a semantic update
direction, but RFT with exact policy gradients preserves an initially zero
style channel, whereas SFT can develop off-axis population style drift
that causes finite-time forgetting.
In addition, when sampling is finite, it introduces additional off-axis fluctuations, for which
RFT admits a smaller energy upper bound than SFT.
We examine these mechanisms empirically in section \ref{sec:experiments}.

\section{Experiments}\label{sec:experiments}

The experiments simulate the model of Section~\ref{sec:setting} and
examine three aspects of the learning dynamics: (1) RFT maintains zero off-axis energy under population updates
(Theorem~\ref{thm:rate} and ~\ref{thm:style}(i)); (2) nonzero task means and nonuniform
teacher preferences induce SFT population off-axis style drift and can
lead to semantic forgetting (Theorems~\ref{thm:style}(ii)
and~\ref{thm:separation}); (3) finite-sample training produces additional off-axis parameter fluctuations, even when
population off-axis drift is absent; we measure the final off-axis energy, whose expectation under Gaussian prompts is bounded in Theorem~\ref{thm:offaxis-energy}.

\paragraph{Teacher parameterization.} To construct the teachers for the
simulations, we assign each token a fixed style embedding
$\mathbf{h}_a\in\R^d$ describing how its semantic answer is expressed.
A task- and class-dependent teacher preference vector
$\mathbf{v}_{t,b}\in\R^d$ determines the within-class distribution:
\begin{equation}\label{eq:rho}
\rho_{t,b}(a) =
\frac{e^{\inner{\mathbf{v}_{t,b}}{\mathbf{h}_a}}}
{\sum_{a'\in\cA_b} e^{\inner{\mathbf{v}_{t,b}}{\mathbf{h}_{a'}}}},
\qquad a\in\cA_b .
\end{equation}
Expressions with larger compatibility scores
$\inner{\mathbf{v}_{t,b}}{\mathbf{h}_a}$ receive higher probability;
$\mathbf{v}_{t,b}=\mathbf{0}$ gives a uniform teacher within the class.
Figure~\ref{fig:reward-cells} illustrates such teacher preferences schematically.

\paragraph{Setup.}
We train SFT and RFT from $\mathbf W=\mathbf0$ on $T=8$
sequential tasks, with prompt dimension $p=24$, two semantic
classes of $M=6$ tokens each, and style dimension $d=3$.
We draw $\mathbf u$ uniformly on the unit sphere, $M$ style
embeddings $\mathbf h_1,\dots,\mathbf h_M\sim N(\mathbf0,\mathbf I_3/3)$
assigned to the tokens of each class, and teacher preference vectors
$\mathbf v_{t,b}\sim N(\mathbf0,36\,\mathbf I_3)$.\footnote{The
variance $1/d$ normalizes the style embeddings to unit scale,
$\E\norm{\mathbf h_j}^2=1$, so that the within-class logits
$\inner{\mathbf v_{t,b}}{\mathbf h_a}$ in \eqref{eq:rho} have standard
deviation controlled by $\mathbf v_{t,b}$ alone, here about $6$.
This makes the teachers strongly nonuniform while leaving their
preferences varied across tasks and classes.}
Prompts are Gaussian as in Section~\ref{sec:setting}, with task means
$\boldsymbol\mu_t\sim N(\mathbf0,\mathbf I_p-\mathbf u\mathbf u^\top)$,
so that $\boldsymbol\mu_t\perp\mathbf u$.
These task parameters are sampled once and held fixed across runs.

\paragraph{Training and evaluation.}
For finite-sample runs, we sample $n=25$ prompts per task and process
them in a single pass in random order, averaging over $50$ seeds that
control the training prompts, demonstrations, and their order.
For population runs, we approximate \eqref{eq:pop-updates} by averaging
updates over $10{,}000$ fresh prompts at each step, take $25$ steps per
task to match the number of finite-sample updates, and average over
$10$ seeds that control the Monte Carlo prompts.
In both settings, we sweep $\eta\in[10^{-4},0.5]$ on a log-spaced grid
and evaluate each task on a fixed set of $1{,}000$ held-out prompts
from its own distribution.
We report generalization error $\cG^{(T)}$, forgetting $\cF^{(T)}$, and
the final off-axis parameter energy $\norm{\mathbf W_T^\perp}_F^2$,
all averaged across seeds.

\begin{figure}[t]
\centering
\includegraphics[width=\linewidth]{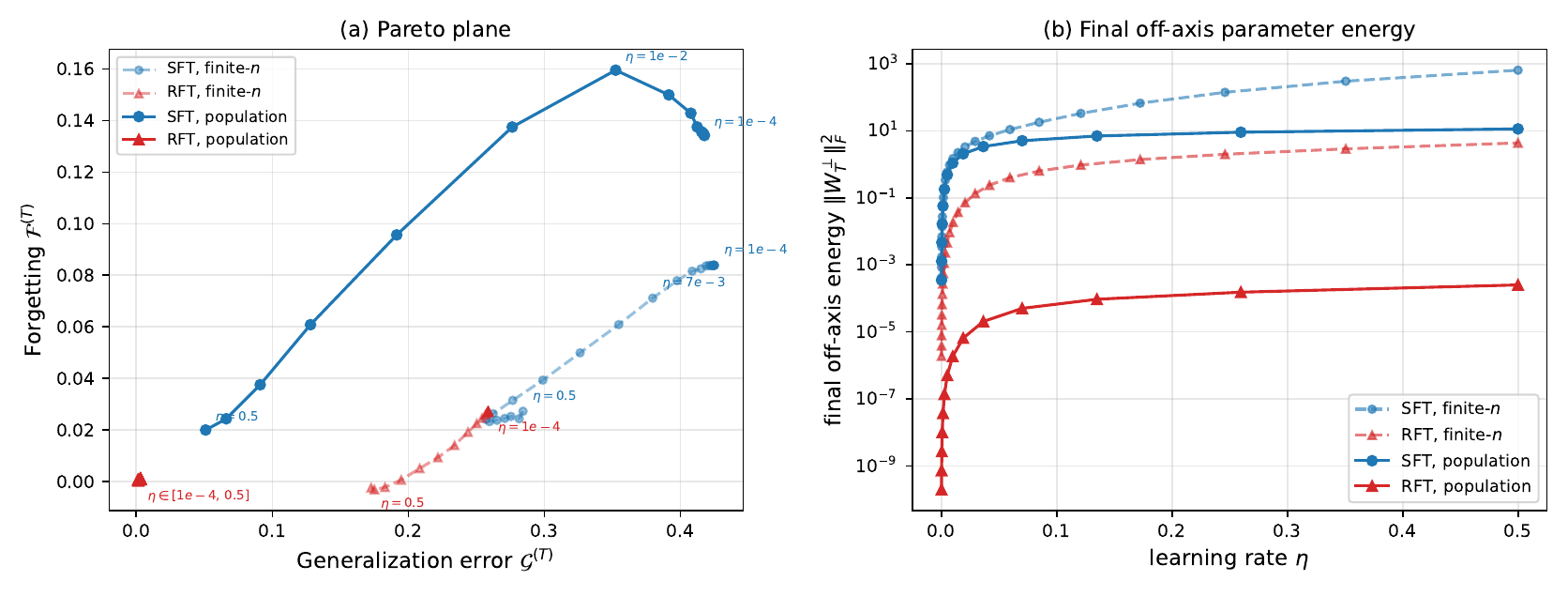}
\caption{\textbf{SFT and RFT under population and finite-sample updates.}
\emph{(a)} Generalization--forgetting plane, 
for finite-sample (dashed, faint) and population (solid) updates; selected
learning rates are annotated. \emph{(b)} Final off-axis parameter energy
$\norm{\mathbf W_T^\perp}_F^2$, averaged across seeds, versus learning rate,
with a logarithmic energy axis.}
\label{fig:pareto-pop}
\end{figure}

\paragraph{Results.} The main results are shown in Figure~\ref{fig:pareto-pop}.
In panel~(a), population RFT maintains zero generalization error
and forgetting throughout the learning-rate sweep, whereas population SFT
exhibits strictly positive values at every tested learning rate.
The persistence of this separation under population updates is consistent
with off-axis style drift driven by nonzero task means and nonuniform
teacher preferences (Theorem~\ref{thm:separation}).
In the finite-sample runs, prompt sampling introduces off-axis
fluctuations that also raise RFT's error
(Theorem~\ref{thm:offaxis-energy}), yet RFT still achieves lower
generalization error and forgetting than SFT at every matched learning rate.
Panel~(b) gives a complementary view: population SFT develops off-axis
parameter energy, while population RFT remains close to zero (Theorem~\ref{thm:rate} and Theorem~\ref{thm:style}), with small
residual energy from Monte Carlo prompt sampling.
Finite-sample updates produce additional off-axis energy, which grows
with the learning rate and is one to two orders of magnitude larger for
SFT than for RFT.
The observed energy gap qualitatively agrees with the smaller upper
bound available for RFT in Theorem~\ref{thm:offaxis-energy}.

\paragraph{Control and ablation experiments.}
We provide additional experiments in Appendix~\ref{app:controls} to
isolate the roles of task means and teacher style preferences and to
relax the assumptions behind RFT's style invariance.
With either zero task means (Figure~\ref{fig:pareto-pop-zero}) or
uniform teachers (Figure~\ref{fig:pareto-pop-uniform}), the population
SFT--RFT gap vanishes, while finite-sample SFT still accumulates more
off-axis energy and higher generalization error than RFT; replacing
Gaussian prompts by Laplace prompts leaves the results nearly unchanged
(Figure~\ref{fig:pareto-pop-laplace}).
We also compare the methods under corrupted rewards and under REINFORCE
with finitely many sampled actions, and find that population RFT retains
low forgetting in both cases; see Figure~\ref{fig:ablation-forgetting}
for details.

\section{Conclusion}\label{sec:conclusion}
This work provides intuition on why supervised fine-tuning forgets more
than reinforcement fine-tuning, even when every demonstration is
semantically correct, through a case study on classification tasks with
a linear-softmax policy. We show that semantically correct supervision can still induce forgetting when training also fits task-specific style preferences. In our model, fitting these preferences can shift the semantic decision boundary and disrupt previously learned predictions, even in the absence of sampling noise. RFT avoids this source of interference by preserving within-class symmetry under exact policy gradients from a symmetric initialization. The finite-sample bounds and simulations further characterize the off-axis deviations introduced by prompt sampling.
\paragraph{Limitations.}
Our theoretical analysis focuses on a tractable setting with a
single-token linear-softmax model and a shared semantic rule across
tasks. Future work could extend the analysis to more general settings such as nonlinear policies, multi-token outputs, and tasks with evolving semantic rules.

\bibliography{references}
\bibliographystyle{references}

\clearpage
\appendix
\numberwithin{equation}{section}
\makeatletter
\@addtoreset{theorem}{section}
\@addtoreset{figure}{section}
\@addtoreset{table}{section}
\makeatother
\setcounter{theorem}{0}
\renewcommand{\thetheorem}{\thesection.\arabic{theorem}}
\setcounter{figure}{0}
\renewcommand{\thefigure}{\thesection.\arabic{figure}}
\setcounter{table}{0}
\renewcommand{\thetable}{\thesection.\arabic{table}}

\section{Literature Review}\label{sec:related}
\paragraph{Empirical findings on SFT and RFT.}
Recent studies compare how supervised and reinforcement fine-tuning
affect generalization and forgetting in foundation models.
\citet{chu2025sft} find better generalization under RL to changes in
reasoning rules and visual conditions, while \citet{lai2026reinforcement}
report less forgetting of earlier tasks and general capabilities under
RFT across sequential multimodal tasks.
Studies of on-policy and model-generated supervision further highlight
the role of the supervision distribution. \citet{chen2026retaining}
show that approximately on-policy supervision reduces forgetting under
SFT. \citet{zhang2026data} link the retention benefits of SFT on filtered
model-generated trajectories to the magnitude and direction of
example-induced updates. In self-distillation fine-tuning,
\citet{shenfeld2026selfdistillation} use a demonstration-conditioned
teacher to supervise student-generated trajectories without an explicit
reward function, reducing forgetting during skill learning.

\paragraph{Distributional explanations of the SFT--RFT gap.}
\citet{shenfeld2026razor} relate forgetting to
the KL divergence from the initial policy and contrast SFT's fitting
of an external supervision distribution with online RL's
information-projection bias toward small policy changes.
Building on the output-mixture simulations of \citet{chen2026retaining},
\citet{balasubramanian2026quantitative} compare divergence-based
post-training objectives by quantifying loss of old-component mass
and drift within an old component. Their forward-KL objective fits
new-mode data, whereas their reverse-KL objective fits a mixture that
includes old behavior; the forgetting comparison therefore depends
on both the objectives and their target distributions.

\paragraph{Forgetting in continual-learning models.}
Complementary theoretical work examines how task structure,
optimization, and architecture affect forgetting.
In overparameterized linear regression,
\citet{lin2023theory} derive expressions for expected forgetting and
generalization error that reveal the effects of model size, task
similarity, and task order. For sequential stochastic gradient descent
in linear regression with a shared ground-truth parameter,
\citet{ding2024understanding} quantify how task covariances, ordering,
and step size affect forgetting. In linear classification with
jointly separable tasks, \citet{evron2023continual} connect weakly
regularized training to sequential max-margin learning and show that
forgetting can arise despite joint separability. In an overparameterized
mixture-of-experts regression setting, \citet{li2025theory} relate
expert specialization and routing to forgetting and generalization.

\paragraph{Explicit methods for mitigating forgetting.}
Replay methods maintain access to earlier experience:
\citet{rolnick2019experience}
combine on-policy learning with off-policy updates and behavioral
cloning from a replay buffer to reduce forgetting in continual
reinforcement learning. Parameter regularization penalizes changes
to weights estimated to be important for past tasks. Elastic weight
consolidation uses the diagonal Fisher information to weight these
penalties \citep{kirkpatrick2017overcoming}, while synaptic intelligence
estimates parameter importance from contributions to loss reduction
along the learning trajectory \citep{zenke2017continual}.
For parameter isolation, \citet{mallya2018packnet} iteratively prune a
network, freeze weights allocated to earlier tasks, and train subsequent
tasks using the remaining capacity.
They provide complementary ways to control forgetting beyond the
choice of SFT or RFT.

\section{Supporting lemmas}\label{app:supporting}

\begin{lemma}[SFT and RFT gradient updates]\label{lem:learning-rules}
Fix a prompt $\mathbf{x}$. A stochastic-gradient step on
\eqref{eq:sft-loss} using $a^\star\sim q_t(\cdot\mid\mathbf{x})$ and an
exact gradient step on \eqref{eq:rft-loss}, both with step size $\eta$,
yield the SFT and RFT updates in \eqref{eq:updates}, respectively.
\end{lemma}

\begin{proof}
For SFT, given a sampled demonstration $a^\star\sim q_t(\cdot\mid\mathbf{x})$, define
\[
\ell_{\mathrm{SFT}}(\mathbf{W};\mathbf{x},a^\star)
=-\log\pi_{\mathbf{W}}(a^\star\mid\mathbf{x}).
\]
Its expectation over the teacher is exactly \eqref{eq:sft-loss}; hence its
gradient is an unbiased stochastic gradient of
$\mathcal{L}_{\mathrm{SFT}}$. To compute its gradient, first expand the
log-softmax as
\[
\log\pi_{\mathbf{W}}(a'\mid\mathbf{x})
=\inner{\mathbf{w}_{a'}}{\mathbf{x}}
-\log\sum_{\widetilde a\in\cA}
\exp\!\bigl(\inner{\mathbf{w}_{\widetilde a}}{\mathbf{x}}\bigr).
\]
Differentiating with respect to $\mathbf{w}_a$ gives
\begin{align}
\nabla_{\mathbf{w}_a}\log\pi_{\mathbf{W}}(a'\mid\mathbf{x})
&=\indic{a=a'}\mathbf{x}
-\frac{\exp\!\bigl(\inner{\mathbf{w}_a}{\mathbf{x}}\bigr)}
{\sum_{\widetilde a\in\cA}
\exp\!\bigl(\inner{\mathbf{w}_{\widetilde a}}{\mathbf{x}}\bigr)}
\mathbf{x}\notag\\
&=\bigl(\indic{a=a'}-\pi_{\mathbf{W}}(a\mid\mathbf{x})\bigr)\mathbf{x}.
\label{eq:app-softmax-log-derivative}
\end{align}
Therefore gradient descent on $\ell_{\mathrm{SFT}}$ gives
\eqref{eq:updates}.

For RFT, expand \eqref{eq:rft-loss} as
\[
\mathcal{L}_{\mathrm{RFT}}(\mathbf{W};\mathbf{x})
=-\sum_{a'\in\cA}\pi_{\mathbf{W}}(a'\mid\mathbf{x})R(a',\mathbf{x}).
\]
We have
\begin{align}
\nabla_{\mathbf{w}_a}\pi_{\mathbf{W}}(a'\mid\mathbf{x})
&=\nabla_{\mathbf{w}_a}
\frac{\exp\!\bigl(\inner{\mathbf{w}_{a'}}{\mathbf{x}}\bigr)}
{\sum_{\widetilde a\in\cA}
\exp\!\bigl(\inner{\mathbf{w}_{\widetilde a}}{\mathbf{x}}\bigr)}
\notag\\
&=\frac{
\indic{a=a'}\exp\!\bigl(\inner{\mathbf{w}_{a'}}{\mathbf{x}}\bigr)
\left(\sum_{\widetilde a\in\cA}
\exp\!\bigl(\inner{\mathbf{w}_{\widetilde a}}{\mathbf{x}}\bigr)\right)
}{
\left(\sum_{\widetilde a\in\cA}
\exp\!\bigl(\inner{\mathbf{w}_{\widetilde a}}{\mathbf{x}}\bigr)\right)^2
}\mathbf{x}\notag\\[-2pt]
&\quad-
\frac{
\exp\!\bigl(\inner{\mathbf{w}_{a'}}{\mathbf{x}}\bigr)
\exp\!\bigl(\inner{\mathbf{w}_{a}}{\mathbf{x}}\bigr)
}{
\left(\sum_{\widetilde a\in\cA}
\exp\!\bigl(\inner{\mathbf{w}_{\widetilde a}}{\mathbf{x}}\bigr)\right)^2
}\mathbf{x}\notag\\
&=\frac{\exp\!\bigl(\inner{\mathbf{w}_{a'}}{\mathbf{x}}\bigr)}
{\sum_{\widetilde a\in\cA}
\exp\!\bigl(\inner{\mathbf{w}_{\widetilde a}}{\mathbf{x}}\bigr)}
\left(
\indic{a=a'}-
\frac{\exp\!\bigl(\inner{\mathbf{w}_{a}}{\mathbf{x}}\bigr)}
{\sum_{\widetilde a\in\cA}
\exp\!\bigl(\inner{\mathbf{w}_{\widetilde a}}{\mathbf{x}}\bigr)}
\right)\mathbf{x}\notag\\
&=\pi_{\mathbf{W}}(a'\mid\mathbf{x})
\bigl(\indic{a=a'}-\pi_{\mathbf{W}}(a\mid\mathbf{x})\bigr)\mathbf{x}.
\label{eq:app-softmax-jacobian-direct}
\end{align}
Consequently,
\begin{align}
-\nabla_{\mathbf{w}_a}\mathcal{L}_{\mathrm{RFT}}
&=\sum_{a'\in\cA}R(a',\mathbf{x})
\nabla_{\mathbf{w}_a}\pi_{\mathbf{W}}(a'\mid\mathbf{x})\notag\\
&=\sum_{a'\in\cA}R(a',\mathbf{x})
\pi_{\mathbf{W}}(a'\mid\mathbf{x})
\bigl(\indic{a=a'}-\pi_{\mathbf{W}}(a\mid\mathbf{x})\bigr)
\mathbf{x}\notag\\
&=\pi_{\mathbf{W}}(a\mid\mathbf{x})R(a,\mathbf{x})\mathbf{x}
-\pi_{\mathbf{W}}(a\mid\mathbf{x})
\sum_{a'\in\cA}\pi_{\mathbf{W}}(a'\mid\mathbf{x})
R(a',\mathbf{x})\mathbf{x}\notag\\
&=\pi_{\mathbf{W}}(a\mid\mathbf{x})
\bigl(R(a,\mathbf{x})-\bar R_{\mathbf{W}}(\mathbf{x})\bigr)\mathbf{x}.
\label{eq:app-rft-loss-gradient}
\end{align}
Gradient descent on $\mathcal{L}_{\mathrm{RFT}}$ therefore gives
\eqref{eq:updates}.
\end{proof}

\begin{lemma}[Correct classification with on-axis channels]
\label{lem:on-axis-classification}
Suppose the channels in \eqref{eq:channels} satisfy
$\Delta\boldsymbol{\alpha}=\gamma\mathbf{u}$ with $\gamma>0$ and
$\boldsymbol{\beta}_a=c_a\mathbf{u}$ for every $a\in\cA$.
Set $B:=\max_{a\in\cA}|c_a|$. If $\gamma>2B$, then the decision rule
\eqref{eq:decision} satisfies
\[
\hat b_{\mathbf{W}}(\mathbf{x})=b^\star(\mathbf{x})
\qquad\text{for every }\mathbf{x}\in\R^p.
\]
In particular, its decision boundary is
$\{\mathbf{x}:\inner{\mathbf{u}}{\mathbf{x}}=0\}$, and
$\mathcal{L}_t(\mathbf{W})=0$ for every task $t$.
\end{lemma}

\begin{proof}
Write $z:=\inner{\mathbf{u}}{\mathbf{x}}$.
Since $\inner{\boldsymbol{\beta}_a}{\mathbf{x}}=c_a z$,
\eqref{eq:decision} gives
\[
D_{\mathbf{W}}(\mathbf{x})=\gamma z+\xi_{\boldsymbol{\beta}}(\mathbf{x}),
\qquad
|\xi_{\boldsymbol{\beta}}(\mathbf{x})|
\le 2\max_{a\in\cA}|c_a z|=2B|z|.
\]
For $z\ne0$, it follows that
\[
zD_{\mathbf{W}}(\mathbf{x})
\ge\gamma z^2-|z|\,|\xi_{\boldsymbol{\beta}}(\mathbf{x})|
\ge(\gamma-2B)z^2>0.
\]
Thus $D_{\mathbf{W}}(\mathbf{x})$ has the same sign as $z$.
When $z=0$, all style scores and $D_{\mathbf{W}}(\mathbf{x})$ vanish,
so both \eqref{eq:decision} and \eqref{eq:reward} assign class $0$.
Hence the two rules agree everywhere and have the stated boundary.
\end{proof}

\begin{lemma}[Gaussian projection decomposition]\label{lem:gaussian-projection}
Under the setup of Section~\ref{sec:setting}, fix a task $t$ and let
$\mathbf{x}\sim N(\boldsymbol{\mu}_t,\mathbf{I}_p)$. Define
\[
z=\inner{\mathbf{u}}{\mathbf{x}},
\qquad
\boldsymbol{\xi}^{\perp}
=\bigl(\mathbf{I}_p-\mathbf{u}\mathbf{u}^\top\bigr)
(\mathbf{x}-\boldsymbol{\mu}_t).
\]
Then
\[
z\sim N(0,1),
\qquad
\boldsymbol{\xi}^{\perp}\sim
N(\mathbf{0},\mathbf{I}_p-\mathbf{u}\mathbf{u}^\top),
\qquad
z\indep\boldsymbol{\xi}^{\perp},
\]
and
\[
\mathbf{x}=z\mathbf{u}+\boldsymbol{\mu}_t+\boldsymbol{\xi}^{\perp}.
\]
\end{lemma}

\begin{proof}
Let $\boldsymbol{\varepsilon}=\mathbf{x}-\boldsymbol{\mu}_t
\sim N(\mathbf{0},\mathbf{I}_p)$. Since
$\inner{\boldsymbol{\mu}_t}{\mathbf{u}}=0$,
\[
z=\mathbf{u}^\top\boldsymbol{\varepsilon},
\qquad
\boldsymbol{\xi}^{\perp}
=\bigl(\mathbf{I}_p-\mathbf{u}\mathbf{u}^\top\bigr)
\boldsymbol{\varepsilon}.
\]
Thus $z$ and $\boldsymbol{\xi}^{\perp}$ are jointly Gaussian. Using
$\norm{\mathbf{u}}=1$,
\[
\operatorname{Var}(z)=\mathbf{u}^\top\mathbf{u}=1,
\qquad
\operatorname{Cov}(\boldsymbol{\xi}^{\perp})
=\mathbf{I}_p-\mathbf{u}\mathbf{u}^\top,
\]
and their cross-covariance is
\[
\operatorname{Cov}(\boldsymbol{\xi}^{\perp},z)
=\bigl(\mathbf{I}_p-\mathbf{u}\mathbf{u}^\top\bigr)\mathbf{u}
=\mathbf{0}.
\]
Joint Gaussianity therefore implies
$z\indep\boldsymbol{\xi}^{\perp}$. Finally,
\[
\boldsymbol{\varepsilon}
=\mathbf{u}\mathbf{u}^\top\boldsymbol{\varepsilon}
+\bigl(\mathbf{I}_p-\mathbf{u}\mathbf{u}^\top\bigr)
\boldsymbol{\varepsilon}
=z\mathbf{u}+\boldsymbol{\xi}^{\perp},
\]
which gives the stated decomposition of $\mathbf{x}$.
\end{proof}

\begin{lemma}[Vectors invariant under $G_{\mathbf u}$]
\label{lem:fixed-vectors}
Let $\mathbf{u}\in\mathbb{R}^p$ be nonzero and define
\begin{equation}\label{eq:app-Gu}
G_{\mathbf{u}}=\{\mathbf{\mathbf{Q}}\in\mathrm{O}(p):\mathbf{\mathbf{Q}}\mathbf{u}=\mathbf{u}\}.
\end{equation}
Then $\mathbf{\mathbf{Q}}\mathbf{v}=\mathbf{v}$ for every $\mathbf{\mathbf{Q}}\in G_{\mathbf{u}}$ if and only if
$\mathbf{v}\in\operatorname{span}(\mathbf{u})$.
\end{lemma}

\begin{proof}
First suppose $\mathbf{v}=c\mathbf{u}$. For every
$\mathbf{\mathbf{Q}}\in G_{\mathbf{u}}$,
\[
\mathbf{\mathbf{Q}}\mathbf{v}=c\mathbf{Q}\mathbf{u}=c\mathbf{u}=\mathbf{v},
\]
which proves one direction.

Conversely, suppose $\mathbf{\mathbf{Q}}\mathbf{v}=\mathbf{v}$ for every
$\mathbf{\mathbf{Q}}\in G_{\mathbf{u}}$. Decompose $\mathbf{v}$ as
\[
\mathbf{v}=c\mathbf{u}+\mathbf{v}^\perp,
\qquad
c=\frac{\inner{\mathbf{u}}{\mathbf{v}}}{\norm{\mathbf{u}}^2},
\qquad
\mathbf{v}^\perp\perp\mathbf{u}.
\]
Suppose, toward a contradiction, that
$\mathbf{v}^\perp\ne\mathbf{0}$. Define
\[
\mathbf{e}=\frac{\mathbf{v}^\perp}{\norm{\mathbf{v}^\perp}},
\qquad
\mathbf{\mathbf{Q}}=\mathbf{I}_p-2\mathbf{e}\mathbf{e}^\top.
\]
Because $\norm{\mathbf{e}}=1$,
\begin{align*}
\mathbf{\mathbf{Q}}^\top \mathbf{\mathbf{Q}}
&=(\mathbf{I}_p-2\mathbf{e}\mathbf{e}^\top)^2\\
&=\mathbf{I}_p-4\mathbf{e}\mathbf{e}^\top
+4\mathbf{e}(\mathbf{e}^\top\mathbf{e})\mathbf{e}^\top\\
&=\mathbf{I}_p.
\end{align*}
Hence $\mathbf{\mathbf{Q}}\in\mathrm{O}(p)$. Since $\mathbf{e}\perp\mathbf{u}$, we also have
\[
\mathbf{\mathbf{Q}}\mathbf{u}
=\mathbf{u}-2\mathbf{e}\mathbf{e}^\top\mathbf{u}
=\mathbf{u},
\]
and therefore $\mathbf{\mathbf{Q}}\in G_{\mathbf{u}}$. On the other hand,
\begin{align*}
\mathbf{\mathbf{Q}}\mathbf{v}^\perp
&=\mathbf{v}^\perp
-2\mathbf{e}\mathbf{e}^\top\mathbf{v}^\perp\\
&=\mathbf{v}^\perp-2\mathbf{v}^\perp
=-\mathbf{v}^\perp.
\end{align*}
It follows that
\[
\mathbf{\mathbf{Q}}\mathbf{v}
=c\mathbf{Q}\mathbf{u}+\mathbf{\mathbf{Q}}\mathbf{v}^\perp
=c\mathbf{u}-\mathbf{v}^\perp
\ne c\mathbf{u}+\mathbf{v}^\perp
=\mathbf{v},
\]
contradicting the assumption that every $\mathbf{\mathbf{Q}}\in G_{\mathbf{u}}$ fixes
$\mathbf{v}$. Thus $\mathbf{v}^\perp=\mathbf{0}$ and
$\mathbf{v}\in\operatorname{span}(\mathbf{u})$.
\end{proof}

\begin{lemma}[Rotation equivariance of zero-mean trajectories]
\label{lem:rotation-equivariance}
Assume $\boldsymbol{\mu}_t=\mathbf{0}$ for every task and
$\mathbf{W}(0)=\mathbf{0}$. Fix a task sequence and suppose each step uses
a fresh prompt independent of the training history and, for SFT, a fresh
teacher draw conditional on that prompt. Let $G_{\mathbf{u}}$ be as defined in
\eqref{eq:app-Gu}; these are the orthogonal transformations that
leave $\mathbf{u}$ fixed and act only on its orthogonal complement
$\mathbf{u}^\perp$. Let $G_{\mathbf{u}}$ act rowwise on the parameter matrix,
$(\mathbf{\mathbf{Q}}\!\cdot\!\mathbf{W})_a=\mathbf{\mathbf{Q}}\mathbf{w}_a$.
Under either the SFT or RFT update, for every $\mathbf{\mathbf{Q}}\in G_{\mathbf{u}}$ and every
step $\tau$,
\begin{equation}\label{eq:app-trajectory-invariance}
\mathbf{\mathbf{Q}}\!\cdot\!\mathbf{W}(\tau)\stackrel{d}{=}\mathbf{W}(\tau).
\end{equation}
Consequently, for every token $a$,
\begin{equation}\label{eq:app-on-axis-expectations}
\E[\mathbf{w}_a(\tau)],\quad
\E[\Delta\boldsymbol{\alpha}(\tau)],\quad
\E[\boldsymbol{\beta}_a(\tau)]
\in\operatorname{span}(\mathbf{u}).
\end{equation}
\end{lemma}

\begin{proof}
Let $\omega$ denote the random input used in one update. For SFT, write
$\omega=(\mathbf{x},a^\star)$ and define
$\mathbf{\mathbf{Q}}\star\omega=(\mathbf{\mathbf{Q}}\mathbf{x},a^\star)$; for RFT, write $\omega=\mathbf{x}$ and
define $\mathbf{\mathbf{Q}}\star\omega=\mathbf{\mathbf{Q}}\mathbf{x}$. Since
$\mathbf{x}\sim N(\mathbf{0},\mathbf{I}_p)$, the linear transformation rule
for Gaussian vectors and $\mathbf{Q}\mathbf{Q}^\top=\mathbf{I}_p$ give
\[
\mathbf{\mathbf{Q}}\mathbf{x}
\sim N\!\left(\mathbf{\mathbf{Q}}\mathbf{0},\mathbf{\mathbf{Q}}\mathbf{I}_p\mathbf{Q}^\top\right)
=N(\mathbf{0},\mathbf{I}_p).
\]
Thus $\mathbf{\mathbf{Q}}\mathbf{x}\stackrel{d}{=}\mathbf{x}$, so the prompt law is invariant
under $\mathbf{\mathbf{Q}}$. Moreover, since $\mathbf{\mathbf{Q}}$ is
orthogonal and $\mathbf{\mathbf{Q}}\mathbf{u}=\mathbf{u}$, we also have
$\mathbf{\mathbf{Q}}^\top\mathbf{u}=\mathbf{u}$. Therefore,
\[
b^\star(\mathbf{\mathbf{Q}}\mathbf{x})
=\indic{\inner{\mathbf{u}}{\mathbf{\mathbf{Q}}\mathbf{x}}>0}
=\indic{\inner{\mathbf{\mathbf{Q}}^\top\mathbf{u}}{\mathbf{x}}>0}
=\indic{\inner{\mathbf{u}}{\mathbf{x}}>0}
=b^\star(\mathbf{x}).
\]
In addition, 
\begin{align*}
\pi_{\mathbf{\mathbf{Q}}\cdot\mathbf{W}}(a\mid \mathbf{\mathbf{Q}}\mathbf{x})
&=\frac{\exp\!\left(\inner{\mathbf{\mathbf{Q}}\mathbf{w}_a}{\mathbf{\mathbf{Q}}\mathbf{x}}\right)}
{\sum_{a'\in\cA}\exp\!\left(\inner{\mathbf{\mathbf{Q}}\mathbf{w}_{a'}}{\mathbf{\mathbf{Q}}\mathbf{x}}\right)}\\
&=\frac{\exp\!\left(\mathbf{w}_a^\top \mathbf{\mathbf{Q}}^\top \mathbf{\mathbf{Q}}\mathbf{x}\right)}
{\sum_{a'\in\cA}\exp\!\left(\mathbf{w}_{a'}^\top \mathbf{\mathbf{Q}}^\top \mathbf{\mathbf{Q}}\mathbf{x}\right)}\\
&=\frac{\exp\!\left(\inner{\mathbf{w}_a}{\mathbf{x}}\right)}
{\sum_{a'\in\cA}\exp\!\left(\inner{\mathbf{w}_{a'}}{\mathbf{x}}\right)}
=\pi_{\mathbf{W}}(a\mid\mathbf{x}).
\end{align*}
The SFT teacher is invariant as well, since \eqref{eq:teacher} depends on the
prompt only through $b^\star(\mathbf{x})$; the RFT reward is invariant for the
same reason. Let $\Phi(\mathbf{W},\omega)$ denote the updated parameter matrix
$\mathbf{W}^+$ obtained by applying the SFT or RFT update in
\eqref{eq:updates}, and let
$\Phi_a$ denote its $a$th row. For SFT, using the invariance of the policy
shown above, we have
\begin{align*}
\Phi_a(\mathbf{\mathbf{Q}}\!\cdot\!\mathbf{W},\mathbf{\mathbf{Q}}\star\omega)
&=\mathbf{\mathbf{Q}}\mathbf{w}_a
+\eta\bigl(\indic{a=a^\star}
-\pi_{\mathbf{\mathbf{Q}}\cdot\mathbf{W}}(a\mid \mathbf{\mathbf{Q}}\mathbf{x})\bigr)\mathbf{\mathbf{Q}}\mathbf{x}\\
&=\mathbf{\mathbf{Q}}\mathbf{w}_a
+\eta\bigl(\indic{a=a^\star}
-\pi_{\mathbf{W}}(a\mid\mathbf{x})\bigr)\mathbf{\mathbf{Q}}\mathbf{x}\\
&=\mathbf{\mathbf{Q}}\Phi_a(\mathbf{W},\omega).
\end{align*}
For RFT, $b^\star(\mathbf{\mathbf{Q}}\mathbf{x})=b^\star(\mathbf{x})$ implies
$R(a,\mathbf{\mathbf{Q}}\mathbf{x})=R(a,\mathbf{x})$. Together with the policy invariance, this
also gives
\begin{align*}
\bar R_{\mathbf{\mathbf{Q}}\cdot\mathbf{W}}(\mathbf{\mathbf{Q}}\mathbf{x})
&=\sum_{a'\in\cA}\pi_{\mathbf{\mathbf{Q}}\cdot\mathbf{W}}(a'\mid \mathbf{\mathbf{Q}}\mathbf{x})
R(a',\mathbf{\mathbf{Q}}\mathbf{x})\\
&=\sum_{a'\in\cA}\pi_{\mathbf{W}}(a'\mid\mathbf{x})R(a',\mathbf{x})
=\bar R_{\mathbf{W}}(\mathbf{x}).
\end{align*}
Consequently, the RFT update satisfies
\begin{align*}
\Phi_a(\mathbf{\mathbf{Q}}\!\cdot\!\mathbf{W},\mathbf{\mathbf{Q}}\star\omega)
&=\mathbf{\mathbf{Q}}\mathbf{w}_a
+\eta\,\pi_{\mathbf{\mathbf{Q}}\cdot\mathbf{W}}(a\mid \mathbf{\mathbf{Q}}\mathbf{x})
\bigl(R(a,\mathbf{\mathbf{Q}}\mathbf{x})
-\bar R_{\mathbf{\mathbf{Q}}\cdot\mathbf{W}}(\mathbf{\mathbf{Q}}\mathbf{x})\bigr)\mathbf{\mathbf{Q}}\mathbf{x}\\
&=\mathbf{\mathbf{Q}}\left[\mathbf{w}_a
+\eta\,\pi_{\mathbf{W}}(a\mid\mathbf{x})
\bigl(R(a,\mathbf{x})-\bar R_{\mathbf{W}}(\mathbf{x})\bigr)\mathbf{x}\right]\\
&=\mathbf{\mathbf{Q}}\Phi_a(\mathbf{W},\omega).
\end{align*}
Thus the following identity holds rowwise under either update rule:
\begin{equation}\label{eq:app-equivariance}
\Phi\bigl(\mathbf{\mathbf{Q}}\!\cdot\!\mathbf{W},\mathbf{\mathbf{Q}}\star\omega\bigr)
=\mathbf{\mathbf{Q}}\!\cdot\!\Phi(\mathbf{W},\omega).
\end{equation}

We prove \eqref{eq:app-trajectory-invariance} by induction. It holds at
$\tau=0$ because $\mathbf{W}(0)=\mathbf{0}$, and hence
$\mathbf{\mathbf{Q}}\!\cdot\!\mathbf{W}(0)=\mathbf{W}(0)$. Now suppose
$\mathbf{\mathbf{Q}}\!\cdot\!\mathbf{W}(\tau)\stackrel{d}{=}\mathbf{W}(\tau)$. Let
$\omega_\tau$ be the fresh sample used at step $\tau$. It is independent of
$\mathbf{W}(\tau)$, and the invariance of the sample law proved above gives
$\mathbf{\mathbf{Q}}\star\omega_\tau\stackrel{d}{=}\omega_\tau$. Therefore, the induction
hypothesis and independence imply the joint distributional identity
\begin{equation}\label{eq:app-joint-invariance}
\bigl(\mathbf{\mathbf{Q}}\!\cdot\!\mathbf{W}(\tau),\mathbf{\mathbf{Q}}\star\omega_\tau\bigr)
\stackrel{d}{=}
\bigl(\mathbf{W}(\tau),\omega_\tau\bigr).
\end{equation}
By the definition of the one-step update map,
$\mathbf{W}(\tau+1)=\Phi(\mathbf{W}(\tau),\omega_\tau)$. Together with
\eqref{eq:app-equivariance} and \eqref{eq:app-joint-invariance}, we have
\begin{align*}
\mathbf{\mathbf{Q}}\!\cdot\!\mathbf{W}(\tau+1)
&=\mathbf{\mathbf{Q}}\!\cdot\!\Phi\bigl(\mathbf{W}(\tau),\omega_\tau\bigr)\\
&=\Phi\bigl(\mathbf{\mathbf{Q}}\!\cdot\!\mathbf{W}(\tau),\mathbf{\mathbf{Q}}\star\omega_\tau\bigr)\\
&\stackrel{d}{=}\Phi\bigl(\mathbf{W}(\tau),\omega_\tau\bigr)\\
&=\mathbf{W}(\tau+1).
\end{align*}
This completes the induction and proves
\eqref{eq:app-trajectory-invariance}.

It follows that $\mathbf{\mathbf{Q}}\E[\mathbf{w}_a(\tau)]=\E[\mathbf{w}_a(\tau)]$ for every
$\mathbf{\mathbf{Q}}\in G_{\mathbf{u}}$. Lemma~\ref{lem:fixed-vectors} therefore implies that
$\E[\mathbf{w}_a(\tau)]\in\operatorname{span}(\mathbf{u})$. The remaining
claims follow because
$\Delta\boldsymbol{\alpha}$ and $\boldsymbol{\beta}_a$ are fixed linear
combinations of the rows.
\end{proof}

\begin{lemma}[Lipschitz continuity of softmax]\label{lem:softmax-lipschitz}
The softmax map is $1$-Lipschitz in the Euclidean norm: for all
$\mathbf{h},\mathbf{h}'\in\R^{2M}$,
\[
\norm{\softmax(\mathbf{h})-\softmax(\mathbf{h}')}_2
\le\norm{\mathbf{h}-\mathbf{h}'}_2.
\]
\end{lemma}

\begin{proof}
For $\boldsymbol{\pi}=\softmax(\mathbf{h})$, differentiating gives the
Jacobian
\[
J(\mathbf{h})=\operatorname{diag}(\boldsymbol{\pi})
-\boldsymbol{\pi}\boldsymbol{\pi}^{\top}.
\]
For every $\mathbf{v}\in\R^{2M}$,
\[
0\le\mathbf{v}^{\top}J(\mathbf{h})\mathbf{v}
=\sum_{a=1}^{2M}\pi_a v_a^2-\left(\sum_{a=1}^{2M}\pi_a v_a\right)^2
\le\sum_{a=1}^{2M}\pi_a v_a^2
\le\norm{\mathbf{v}}_2^2.
\]
The first inequality follows because the expression is a variance under
$\boldsymbol{\pi}$, and the last uses $0\le\pi_a\le1$. Thus
$J(\mathbf{h})$ is symmetric positive semidefinite with eigenvalues at
most $1$, so $\norm{J(\mathbf{h})}_{\mathrm{op}}\le1$ for every $\mathbf{h}$.
Integrating the derivative along the line segment from $\mathbf{h}'$
to $\mathbf{h}$ therefore gives
\[
\begin{aligned}
\norm{\softmax(\mathbf{h})-\softmax(\mathbf{h}')}_2
&\le\int_0^1\norm{J(\mathbf{h}'+t(\mathbf{h}-\mathbf{h}'))}_{\mathrm{op}}
\norm{\mathbf{h}-\mathbf{h}'}_2\,\dd t\\
&\le\norm{\mathbf{h}-\mathbf{h}'}_2.
\end{aligned}
\]
\end{proof}

\section{Proofs of the main results}\label{app:proofs-main}

\subsection{Proof of Theorem~\ref{thm:rate}}\label{app:proof-rate}

\begin{proof}
\noindent\emph{Part (i): the per-sample rate identity.}
For an arbitrary $\mathbf{W}$ and prompt $\mathbf{x}$, by
\eqref{eq:class-prob},
$P_0(\mathbf{W},\mathbf{x})+P_1(\mathbf{W},\mathbf{x})=1$.

First consider the SFT step. If the demonstration is $a^\star$ with
$\sigma(a^\star)=b^\star(\mathbf{x})$, then the row update is
\begin{align}
\mathbf{w}_a^+-\mathbf{w}_a
&= \eta\bigl(\indic{a=a^\star}
-\pi_{\mathbf{W}}(a\mid \mathbf{x})\bigr)\mathbf{x}.
\label{eq:app-sft-row}
\end{align}
Averaging \eqref{eq:app-sft-row} over a class and using the definition of $\boldsymbol{\alpha}^{(b)}$
in \eqref{eq:channels} gives
\begin{align}
\boldsymbol{\alpha}^{(b)+}-\boldsymbol{\alpha}^{(b)}
&= \frac{1}{M}\sum_{a\in\cA_b}
\eta\bigl(\indic{a=a^\star}
-\pi_{\mathbf{W}}(a\mid \mathbf{x})\bigr)\mathbf{x} \notag\\
&= \frac{\eta}{M}
\Bigl(\indic{b=b^\star(\mathbf{x})}
- \sum_{a\in\cA_b}\pi_{\mathbf{W}}(a\mid \mathbf{x})\Bigr)\mathbf{x}
\notag\\
&= \frac{\eta}{M}\bigl(\indic{b=b^\star(\mathbf{x})}
-P_b(\mathbf{W},\mathbf{x})\bigr)\mathbf{x},
\label{eq:app-sft-mean}
\end{align}
where the last equality is \eqref{eq:class-prob}. Subtracting the
$b=0$ mean update from the $b=1$ mean update yields
\begin{align}
\Delta\boldsymbol{\alpha}^+_{\mathrm{SFT}}
-\Delta\boldsymbol{\alpha}
&= \eta\left[
\frac{\indic{b^\star(\mathbf{x})=1}-P_1(\mathbf{W},\mathbf{x})}{M}
-\frac{\indic{b^\star(\mathbf{x})=0}-P_0(\mathbf{W},\mathbf{x})}{M}
\right]\mathbf{x}.
\label{eq:app-sft-delta-raw}
\end{align}
If $b^\star(\mathbf{x})=1$, then $2b^\star(\mathbf{x})-1=1$,
$P_{b^\star(\mathbf{x})}(\mathbf{W},\mathbf{x})
=P_1(\mathbf{W},\mathbf{x})$, and
$P_0(\mathbf{W},\mathbf{x})
=1-P_{b^\star(\mathbf{x})}(\mathbf{W},\mathbf{x})$; if
$b^\star(\mathbf{x})=0$, then $2b^\star(\mathbf{x})-1=-1$,
$P_{b^\star(\mathbf{x})}(\mathbf{W},\mathbf{x})
=P_0(\mathbf{W},\mathbf{x})$, and
$P_1(\mathbf{W},\mathbf{x})
=1-P_{b^\star(\mathbf{x})}(\mathbf{W},\mathbf{x})$.
Therefore \eqref{eq:app-sft-delta-raw} becomes the SFT line of
\eqref{eq:semantic-steps}.

For RFT, the reward is
$R(a,\mathbf{x})=\indic{\sigma(a)=b^\star(\mathbf{x})}$, so
\eqref{eq:updates} gives
\begin{align}
\bar R_{\mathbf{W}}(\mathbf{x})
&= \sum_a \pi_{\mathbf{W}}(a\mid \mathbf{x})
\indic{\sigma(a)=b^\star(\mathbf{x})}
= P_{b^\star(\mathbf{x})}(\mathbf{W},\mathbf{x}),
\label{eq:app-rft-reward-avg}\\
\mathbf{w}_a^+-\mathbf{w}_a
&= \eta\,\pi_{\mathbf{W}}(a\mid\mathbf{x})
\bigl(\indic{\sigma(a)=b^\star(\mathbf{x})}
-P_{b^\star(\mathbf{x})}(\mathbf{W},\mathbf{x})\bigr)
\mathbf{x}.
\label{eq:app-rft-row}
\end{align}
Averaging \eqref{eq:app-rft-row} over a class gives
\begin{align}
\boldsymbol{\alpha}^{(b)+}-\boldsymbol{\alpha}^{(b)}
&= \frac{\eta}{M}\sum_{a\in\cA_b}
\pi_{\mathbf{W}}(a\mid\mathbf{x})
\bigl(\indic{b=b^\star(\mathbf{x})}
-P_{b^\star(\mathbf{x})}(\mathbf{W},\mathbf{x})\bigr)\mathbf{x}
\notag\\
&= \eta\frac{P_b(\mathbf{W},\mathbf{x})}{M}
\bigl(\indic{b=b^\star(\mathbf{x})}
-P_{b^\star(\mathbf{x})}(\mathbf{W},\mathbf{x})\bigr)\mathbf{x}.
\label{eq:app-rft-mean}
\end{align}
Subtracting the class-$0$ update in \eqref{eq:app-rft-mean} from the
class-$1$ update gives
\begin{align}
\Delta\boldsymbol{\alpha}^+_{\mathrm{RFT}}
-\Delta\boldsymbol{\alpha}
&= \eta\Biggl[
\frac{P_1(\mathbf{W},\mathbf{x})}{M}
\bigl(\indic{b^\star(\mathbf{x})=1}
-P_{b^\star(\mathbf{x})}(\mathbf{W},\mathbf{x})\bigr)\notag\\
&\hspace{2.5cm}
-\frac{P_0(\mathbf{W},\mathbf{x})}{M}
\bigl(\indic{b^\star(\mathbf{x})=0}
-P_{b^\star(\mathbf{x})}(\mathbf{W},\mathbf{x})\bigr)
\Biggr]\mathbf{x}.
\label{eq:app-rft-delta-raw}
\end{align}
Again splitting on $b^\star(\mathbf{x})$ and using
$P_0(\mathbf{W},\mathbf{x})+P_1(\mathbf{W},\mathbf{x})=1$, the right side of
\eqref{eq:app-rft-delta-raw} is
\begin{align}
&\frac{2\eta}{M}\,\bigl(2b^\star(\mathbf{x})-1\bigr)
P_{b^\star(\mathbf{x})}(\mathbf{W},\mathbf{x})
\bigl(1-P_{b^\star(\mathbf{x})}(\mathbf{W},\mathbf{x})\bigr)\mathbf{x},
\label{eq:app-rft-delta-simplified}
\end{align}
which is the RFT line of \eqref{eq:semantic-steps}.

\medskip
\noindent\emph{Part (ii): population semantic update.}
Now assume additionally that $\boldsymbol{\beta}=\mathbf{0}$ and
$\Delta\boldsymbol{\alpha}=c\mathbf{u}$ for some $c\in\mathbb{R}$. Then
\eqref{eq:channel-policy} implies that the class probabilities
$P_b(\mathbf{W},\mathbf{x})$ depend on $\mathbf{W}$ only through
$\Delta\boldsymbol{\alpha}$.
Fix task $t$ and use the notation of
Lemma~\ref{lem:gaussian-projection}:
\[
z=\inner{\mathbf{u}}{\mathbf{x}},
\qquad
\boldsymbol{\xi}^{\perp}
=\bigl(\mathbf{I}_p-\mathbf{u}\mathbf{u}^\top\bigr)
(\mathbf{x}-\boldsymbol{\mu}_t).
\]
Lemma~\ref{lem:gaussian-projection} gives $z\sim N(0,1)$,
$z\indep\boldsymbol{\xi}^{\perp}$, and
\[
\mathbf{x}=z\mathbf{u}+\boldsymbol{\mu}_t+\boldsymbol{\xi}^{\perp} .
\]
It follows that
\[
\inner{\Delta\boldsymbol{\alpha}}{\mathbf{x}}
=c\,\inner{\mathbf{u}}{z\mathbf{u}+\boldsymbol{\mu}_t+\boldsymbol{\xi}^{\perp}}
=c\,z,
\]
because both $\boldsymbol{\mu}_t$ and $\boldsymbol{\xi}^{\perp}$ are orthogonal to
$\mathbf{u}$. Hence \eqref{eq:channel-policy} implies that
$P_1(\mathbf{W},\mathbf{x})=(1+e^{-cz})^{-1}$ and
$P_0(\mathbf{W},\mathbf{x})=1-P_1(\mathbf{W},\mathbf{x})$. Since
$b^\star(\mathbf{x})=b^\star(z):=\indic{z>0}$, the correct-class probability
has the one-dimensional representation
\[
P_{b^\star(\mathbf{x})}(\mathbf{W},\mathbf{x})
=P_\star(z):=\frac{1}{1+e^{-c|z|}}.
\]
This function is even in $z$.
Also
$2b^\star(\mathbf{x})-1=\sign(z)$ almost surely.

The SFT rate identity in \eqref{eq:semantic-steps} becomes
\begin{equation}\label{eq:app-rate-sft-task-mean}
\Delta\boldsymbol{\alpha}^+_{\mathrm{SFT}}-\Delta\boldsymbol{\alpha}
=\frac{2\eta}{M}
\sign(z)\,\bigl(1-P_\star(z)\bigr)\,
\bigl(z\mathbf{u}+\boldsymbol{\mu}_t+\boldsymbol{\xi}^{\perp}\bigr).
\end{equation}
Taking the task-$t$ population expectation gives
\begin{align}
\E_t\bigl[\Delta\boldsymbol{\alpha}^+_{\mathrm{SFT}}-\Delta\boldsymbol{\alpha}\bigr]
&=\frac{2\eta}{M}
\Bigl(
\E[\sign(z)(1-P_\star(z))z]\,\mathbf{u}
\notag\\
&\hspace{3.6cm}
+\E[\sign(z)(1-P_\star(z))]\,\boldsymbol{\mu}_t
\notag\\
&\hspace{3.6cm}
+\E[\sign(z)(1-P_\star(z))\boldsymbol{\xi}^{\perp}]
\Bigr).
\label{eq:app-pop-alpha-sft-expanded}
\end{align}
The last term vanishes because
$\sign(z)(1-P_\star(z))$ is a function of $z$
alone, $z$ is independent of $\boldsymbol{\xi}^{\perp}$, and
$\E[\boldsymbol{\xi}^{\perp}]=\mathbf{0}$. The $\boldsymbol{\mu}_t$ term also
vanishes: $1-P_\star(z)$ is even, so
$\sign(z)(1-P_\star(z))$ is odd, and its expectation under the symmetric
$z\sim N(0,1)$ is zero. Since $\sign(z)z=|z|$, the remaining scalar
expectation is
\[
\begin{aligned}
\E\bigl[|z|(1-P_\star(z))\bigr]
&=\E\left[\frac{|z|}{1+e^{c|z|}}\right]\\
&=\sqrt{\frac{2}{\pi}}\int_0^\infty
\frac{z e^{-z^2/2}}{1+e^{cz}}\,\dd z.
\end{aligned}
\]
The last equality uses the standard normal density and the evenness of
the integrand. Substituting into
\eqref{eq:app-pop-alpha-sft-expanded} gives the SFT line of
\eqref{eq:pop-alpha-task-mean}.

Similarly, the RFT rate identity in \eqref{eq:semantic-steps} becomes
\begin{equation}\label{eq:app-rate-rft-task-mean}
\Delta\boldsymbol{\alpha}^+_{\mathrm{RFT}}-\Delta\boldsymbol{\alpha}
=\frac{2\eta}{M}
\sign(z)\,P_\star(z)
\bigl(1-P_\star(z)\bigr)\,
\bigl(z\mathbf{u}+\boldsymbol{\mu}_t+\boldsymbol{\xi}^{\perp}\bigr).
\end{equation}
Taking the task-$t$ population expectation gives
\begin{align}
\E_t\bigl[\Delta\boldsymbol{\alpha}^+_{\mathrm{RFT}}-\Delta\boldsymbol{\alpha}\bigr]
&=\frac{2\eta}{M}
\Bigl(
\E[\sign(z)P_\star(z)
(1-P_\star(z))z]\,\mathbf{u}
\notag\\
&\hspace{3.6cm}
+\E[\sign(z)P_\star(z)
(1-P_\star(z))]\,\boldsymbol{\mu}_t
\notag\\
&\hspace{3.6cm}
+\E[\sign(z)P_\star(z)
(1-P_\star(z))\boldsymbol{\xi}^{\perp}]
\Bigr).
\label{eq:app-pop-alpha-rft-expanded}
\end{align}
Again the last term vanishes by independence and
$\E[\boldsymbol{\xi}^{\perp}]=\mathbf{0}$. The $\boldsymbol{\mu}_t$ term vanishes
as well because
$P_\star(z)
(1-P_\star(z))$ is even, so its product with
$\sign(z)$ is odd. Using $\sign(z)z=|z|$ and again integrating against
the standard normal density gives
\[
\begin{aligned}
\E\bigl[|z|P_\star(z)(1-P_\star(z))\bigr]
&=\E\left[\frac{|z|e^{c|z|}}{(1+e^{c|z|})^2}\right]\\
&=\sqrt{\frac{2}{\pi}}\int_0^\infty
\frac{z e^{-z^2/2}e^{cz}}{(1+e^{cz})^2}\,\dd z.
\end{aligned}
\]
Substituting into \eqref{eq:app-pop-alpha-rft-expanded} yields the RFT
line of \eqref{eq:pop-alpha-task-mean}.
\end{proof}

\subsection{Proof of Theorem~\ref{thm:style}}\label{app:proof-style}

\begin{proof}
\leavevmode
\paragraph{Part (i): RFT style invariance.}
Fix a prompt $\mathbf{x}$ and write the RFT update direction from
\eqref{eq:updates} as
\begin{equation}\label{eq:app-rft-direction}
\mathbf{g}_a
=\pi_{\mathbf{W}}(a\mid\mathbf{x})
\bigl(R(a,\mathbf{x})-\bar R_{\mathbf{W}}(\mathbf{x})\bigr)\mathbf{x}.
\end{equation}
Assume $\boldsymbol{\beta}=\mathbf{0}$. By \eqref{eq:channels}, every row
in a common semantic class is equal to that class's mean. Thus, for any two tokens
$a,a'$ in the same class,
\begin{align}
\mathbf{w}_a &= \mathbf{w}_{a'}, \label{eq:app-same-row}\\
\inner{\mathbf{w}_a}{\mathbf{x}}
&= \inner{\mathbf{w}_{a'}}{\mathbf{x}}, \notag\\
\pi_{\mathbf{W}}(a\mid \mathbf{x})
&= \pi_{\mathbf{W}}(a'\mid \mathbf{x}).
\label{eq:app-same-policy}
\end{align}
The reward also agrees inside the class because, by hypothesis, $R$ depends on
$a$ only through its class $\sigma(a)$, which is constant on a semantic class:
\begin{align}
\sigma(a)=\sigma(a')
&\Longrightarrow
R(a,\mathbf{x})=R(a',\mathbf{x}).
\label{eq:app-same-reward}
\end{align}
Combining \eqref{eq:app-rft-direction}, \eqref{eq:app-same-policy}, and
\eqref{eq:app-same-reward}, the update vector is identical for every row in
the same class. Therefore if $a,a'$ lie in the same class,
\begin{align}
\mathbf{w}_a^+ - \mathbf{w}_{a'}^+
&= \bigl(\mathbf{w}_a-\mathbf{w}_{a'}\bigr)
+ \eta\bigl(\mathbf{g}_a-\mathbf{g}_{a'}\bigr) \notag\\
&= \mathbf{0}+\eta\cdot\mathbf{0}=\mathbf{0},
\label{eq:app-cell-const-preserved}
\end{align}
Hence an update preserves $\boldsymbol{\beta}=\mathbf{0}$.
Induction over the prompt sequence gives
$\boldsymbol{\beta}_{\mathrm{RFT}}(\tau)=\mathbf{0}$ at every step. 
\paragraph{Part (ii): SFT style updates.}
For convenience, define the within-class probability deviation
\begin{equation}\label{eq:app-style-deviation}
\delta_a(\mathbf{W},\mathbf{x})
:=\pi_{\mathbf{W}}(a\mid\mathbf{x})
-\frac{1}{M}P_{\sigma(a)}(\mathbf{W},\mathbf{x}).
\end{equation}
For an arbitrary current matrix $\mathbf{W}$, the SFT row update and its
class average are exactly \eqref{eq:app-sft-row} and
\eqref{eq:app-sft-mean}, respectively. Subtracting
\eqref{eq:app-sft-mean} at $b=\sigma(a)$ from \eqref{eq:app-sft-row} gives
\begin{align}
\boldsymbol{\beta}_a^+-\boldsymbol{\beta}_a
&= \eta\bigl(\indic{a=a^\star}
-\pi_{\mathbf{W}}(a\mid\mathbf{x})\bigr)\mathbf{x}
-\frac{\eta}{M}
\bigl(\indic{\sigma(a)=b^\star}
-P_{\sigma(a)}(\mathbf{W},\mathbf{x})\bigr)\mathbf{x}
\notag\\
&= \eta\left[
\indic{a=a^\star}
-\frac{1}{M}\indic{\sigma(a)=b^\star}
-\left(\pi_{\mathbf{W}}(a\mid\mathbf{x})
-\frac{P_{\sigma(a)}(\mathbf{W},\mathbf{x})}{M}\right)
\right]\mathbf{x}\\
&=\eta\left[
\indic{a=a^\star}
-\frac{1}{M}\indic{\sigma(a)=b^\star}
-\delta_a(\mathbf{W},\mathbf{x})
\right]\mathbf{x},
\label{eq:app-beta-step}
\end{align}
which is exactly \eqref{eq:betastep}.

Now fix a task $t$, a token $a\in\cA_b$, and a current parameter matrix
$\mathbf{W}$ with $\boldsymbol{\beta}=\mathbf{0}$. Since $\mathbf{W}$ has identical rows within each class, all rows in the same
semantic class have the same logit at every prompt. Therefore the policy is
uniform within each semantic class:
\[
\pi_{\mathbf{W}}(a\mid\mathbf{x})
=\frac{1}{M}\sum_{a'\in\cA_b}\pi_{\mathbf{W}}(a'\mid\mathbf{x})
=\frac1M P_b(\mathbf{W},\mathbf{x})
\qquad (a\in\cA_b).
\]
Substituting this identity into the definition of
$\delta_a$ gives
\[
\delta_a(\mathbf{W},\mathbf{x})
=\pi_{\mathbf{W}}(a\mid\mathbf{x})
-\frac1M P_{\sigma(a)}(\mathbf{W},\mathbf{x})
=0.
\]
Thus the policy-centering feedback in the exact style update
\eqref{eq:betastep} vanishes. For this fixed prompt and teacher draw,
\eqref{eq:betastep} becomes
\[
\boldsymbol{\beta}_a^+-\boldsymbol{\beta}_a
=
\eta\left[
\indic{a=a^\star}
-\frac1{M}\indic{b^\star(\mathbf{x})=b}
\right]\mathbf{x}.
\]
Now take expectation only over the teacher demonstration, conditional on the
prompt $\mathbf{x}$. By the teacher model \eqref{eq:teacher}, if
$b^\star(\mathbf{x})\ne b$ then the teacher places no mass on token $a\in\cA_b$,
and both terms inside the brackets have conditional mean zero. If
$b^\star(\mathbf{x})=b$, then
$\Prob(a^\star=a\mid\mathbf{x})=\rho_{t,b}(a)$, so the conditional mean of the
bracket is $\rho_{t,b}(a)-1/M$. Combining the two cases,
\[
\E_{a^\star\mid\mathbf{x}}
\left[
\indic{a=a^\star}
-\frac1{M}\indic{b^\star(\mathbf{x})=b}
\right]
=
\indic{b^\star(\mathbf{x})=b}
\Bigl(\rho_{t,b}(a)-\frac1{M}\Bigr).
\]
Therefore the SFT population update of the style coordinate is
\begin{equation}\label{eq:app-pop-drift-reduce}
\E[\boldsymbol{\beta}_a^+-\boldsymbol{\beta}_a]
=
\eta\Bigl(\rho_{t,b}(a)-\frac1{M}\Bigr)
\E[\indic{b^\star(\mathbf{x})=b}\,\mathbf{x}].
\end{equation}

It remains to compute the last prompt moment. Apply
Lemma~\ref{lem:gaussian-projection} and write
$z=\inner{\mathbf{u}}{\mathbf{x}}$ and
$\boldsymbol{\xi}^{\perp}
=\bigl(\mathbf{I}_p-\mathbf{u}\mathbf{u}^\top\bigr)
(\mathbf{x}-\boldsymbol{\mu}_t)$. Then
$z\indep\boldsymbol{\xi}^{\perp}$ and
\[
\mathbf{x}=z\mathbf{u}+\boldsymbol{\mu}_t+\boldsymbol{\xi}^{\perp},
\]
while $z\sim N(0,1)$. The event $b^\star=b$ depends only on
$z$: it is $\{z>0\}$ for $b=1$ and $\{z<0\}$ for $b=0$. Hence
\[
\E[\indic{b^\star=b}\,\mathbf{x}]
=
\E[\indic{b^\star=b}z]\,\mathbf{u}
+\E[\indic{b^\star=b}(\boldsymbol{\mu}_t+\boldsymbol{\xi}^{\perp})].
\]
For the off-axis term, independence gives
\[
\E[\indic{b^\star=b}(\boldsymbol{\mu}_t+\boldsymbol{\xi}^{\perp})]
=\Prob(b^\star=b)\boldsymbol{\mu}_t
+\Prob(b^\star=b)\E[\boldsymbol{\xi}^{\perp}]
=\frac12\boldsymbol{\mu}_t.
\]
For the on-axis term,
\[
\E[\indic{z>0}z]
=\int_0^\infty z\,\frac{1}{\sqrt{2\pi}}e^{-z^2/2}\,dz
=\frac1{\sqrt{2\pi}},
\qquad
\E[\indic{z<0}z]=-\frac1{\sqrt{2\pi}},
\]
so in either class
\[
\E[\indic{b^\star=b}z]\,\mathbf{u}
=\frac{2b-1}{\sqrt{2\pi}}\mathbf{u}.
\]
Substituting these two moment identities into
\eqref{eq:app-pop-drift-reduce} proves \eqref{eq:pop-drift-full}.

Finally, project \eqref{eq:pop-drift-full} orthogonally to
$\spn(\mathbf{u})$. The term proportional to $\mathbf{u}$ is
removed, while $\boldsymbol{\mu}_t$ is unchanged because
$\boldsymbol{\mu}_t\perp\mathbf{u}$. This gives the off-axis identity in
\eqref{eq:pop-drift-full}.
\end{proof}

\subsection{Proof of Theorem~\ref{thm:separation}}\label{app:proof-separation}

\begin{proof}
\leavevmode
\paragraph{Step 1: perfect initial fit and RFT retention.}
Since $\boldsymbol{\beta}(\mathbf{W}_1)=\mathbf{0}$, the class-score difference is
$D_{\mathbf{W}_1}(\mathbf{x})
=\inner{\Delta\boldsymbol{\alpha}(\mathbf{W}_1)}{\mathbf{x}}$.
Since the old Gaussian law has positive density everywhere,
$\mathcal{L}_1(\mathbf{W}_1)=0$ means that
\[
\Delta\boldsymbol{\alpha}(\mathbf{W}_1)
=\norm{\Delta\boldsymbol{\alpha}(\mathbf{W}_1)}\mathbf{u},
\qquad \norm{\Delta\boldsymbol{\alpha}(\mathbf{W}_1)}>0.
\]
Subtracting the common row average
$\mathbf{c}:=(2M)^{-1}\sum_a\mathbf{w}_{1,a}$ from every row preserves
policy probabilities, predictions, and both update increments.
We therefore use these centered weights throughout the proof, so that
\[
\mathbf{w}_{1,a}=\frac{2\sigma(a)-1}{2}\norm{\Delta\boldsymbol{\alpha}(\mathbf{W}_1)}\mathbf{u},
\qquad \sum_a\mathbf{w}_{1,a}=\mathbf{0}.
\]
Write $z:=\inner{\mathbf{u}}{\mathbf{x}}$.
Both tasks share the semantic rule $b^\star(\mathbf{x})=\indic{z>0}$.
Lemma~\ref{lem:on-axis-classification}, applied with $B=0$, therefore
shows that $\mathbf{W}_1$ classifies both tasks perfectly:
\begin{equation}\label{eq:sep-early-perfect}
a_1(\mathbf{W}_1)=a_2(\mathbf{W}_1)=1.
\end{equation}
Theorem~\ref{thm:style}(i) ensures $\boldsymbol{\beta}=\mathbf{0}$ throughout RFT training, both updates
preserve $\sum_a\mathbf{w}_a=\mathbf{0}$, and
Theorem~\ref{thm:rate}(ii) gives strictly positive on-axis RFT increments.
Thus every RFT iterate $\mathbf{W}$ satisfies
\[
\begin{aligned}
\mathbf{w}_a
&=\frac{2\sigma(a)-1}{2}\norm{\Delta\boldsymbol{\alpha}(\mathbf{W})}\mathbf{u},\\
\norm{\Delta\boldsymbol{\alpha}(\mathbf{W})}
&\ge\norm{\Delta\boldsymbol{\alpha}(\mathbf{W}_1)}>0.
\end{aligned}
\]
Applying Lemma~\ref{lem:on-axis-classification} again with $B=0$ shows
that every RFT iterate classifies both tasks correctly, proving part (i).

\paragraph{Step 2: quantifying the style drift under SFT.}
Let $\hat{\boldsymbol{\mu}}_2:=\boldsymbol{\mu}_2/\norm{\boldsymbol{\mu}_2}$ and
$y:=\inner{\hat{\boldsymbol{\mu}}_2}{\mathbf{x}}$.
By Lemma~\ref{lem:gaussian-projection}, under task $2$ we can write
$\mathbf{x}=z\mathbf{u}+\boldsymbol{\mu}_2+\boldsymbol{\xi}^{\perp}$,
where $z\sim N(0,1)$ is independent of
$\boldsymbol{\xi}^{\perp}\sim N(\mathbf{0},\mathbf{I}_p-\mathbf{u}\mathbf{u}^\top)$.
Since $\hat{\boldsymbol{\mu}}_2\perp\mathbf{u}$ and
$\norm{\hat{\boldsymbol{\mu}}_2}=1$,
\[
y=\inner{\hat{\boldsymbol{\mu}}_2}{\mathbf{x}}
=\norm{\boldsymbol{\mu}_2}+\inner{\hat{\boldsymbol{\mu}}_2}{\boldsymbol{\xi}^{\perp}}
\sim N(\norm{\boldsymbol{\mu}_2},1).
\]
Moreover, $y$ depends only on $\boldsymbol{\xi}^{\perp}$, so it is
independent of $z$.

Define
$V:=\spn\{\mathbf{u},\hat{\boldsymbol{\mu}}_2\}$.
We will show that every parameter row remains in $V$ throughout task-$2$ training.
Define the residual
$\boldsymbol{\xi}_V:=\mathbf{x}-z\mathbf{u}-y\hat{\boldsymbol{\mu}}_2$,
which is the projection of $\mathbf{x}$ onto $V^\perp$.
Since $\boldsymbol{\mu}_2\in V$, the residual has mean zero:
\[
\E_2[\boldsymbol{\xi}_V]
=\boldsymbol{\mu}_2-0\cdot\mathbf{u}-\norm{\boldsymbol{\mu}_2}\hat{\boldsymbol{\mu}}_2
=\mathbf{0}.
\]
Orthogonality of the Gaussian projections also makes
$\boldsymbol{\xi}_V$ independent of $(z,y)$.
Initially, $\mathbf{w}_{1,a}=\frac{2\sigma(a)-1}{2}\norm{\Delta\boldsymbol{\alpha}(\mathbf{W}_1)}\mathbf{u}\in V$.
If all current rows lie in $V$, write
\begin{equation}\label{eq:sep-row-decomposition}
\mathbf{w}_a=\zeta_a\mathbf{u}+\omega_a\hat{\boldsymbol{\mu}}_2.
\end{equation}
Their logits and policy probabilities are then
\[
\inner{\mathbf{w}_a}{\mathbf{x}}=\zeta_a z+\omega_a y,
\qquad
\pi_a=\pi_{\mathbf{W}}(a\mid\mathbf{x})
=\frac{e^{\zeta_a z+\omega_a y}}
{\sum_{a'\in\cA}e^{\zeta_{a'}z+\omega_{a'}y}}.
\]
Thus, for fixed current parameters, $\pi_a$ depends on $\mathbf{x}$
only through $(z,y)$.
Meanwhile, the teacher definition \eqref{eq:teacher} gives
\[
q_2(a\mid\mathbf{x})
=\indic{\sigma(a)=b^\star(\mathbf{x})}\rho_{2,b^\star(\mathbf{x})}(a),
\qquad b^\star(\mathbf{x})=\indic{z>0}.
\]
By the teacher specification in Section~\ref{sec:setting}, the within-class
distribution $\rho_{2,b}$ is fixed for each class and does not depend on $\mathbf{x}$, so
$q_2(a\mid\mathbf{x})$ depends on the input only through the sign of $z$.
Consequently, $q_2(a\mid\mathbf{x})-\pi_a$ is a function of $(z,y)$
alone. By \eqref{eq:pop-updates}, the $V^\perp$ component of the SFT
parameter increment is
\[
\eta_2\E_2[(q_2(a\mid\mathbf{x})-\pi_a)\boldsymbol{\xi}_V]
=\eta_2\E_2[(q_2(a\mid\mathbf{x})-\pi_a)
\E_2[\boldsymbol{\xi}_V\mid z,y]]=\mathbf{0}.
\]
The last equality follows because $\boldsymbol{\xi}_V$ is independent
of $(z,y)$ and has mean zero.
Each updated row therefore remains in $V$. By induction,
\eqref{eq:sep-row-decomposition} holds at every task-$2$ SFT iterate,
so it suffices to track the two scalar
coefficients $\zeta_a$ and $\omega_a$ for each token.

Since $\norm{\mathbf{u}}=1$ and $\mathbf{u}\perp\hat{\boldsymbol{\mu}}_2$,
the decomposition \eqref{eq:sep-row-decomposition} gives
$\zeta_a=\inner{\mathbf{w}_a}{\mathbf{u}}$.
Taking the inner product of the SFT population update
\eqref{eq:pop-updates} with $\mathbf{u}$ and using
$z=\inner{\mathbf{u}}{\mathbf{x}}$ yields
\begin{gather*}
    \zeta_a^+-\zeta_a
=\inner{\mathbf{w}_a^+-\mathbf{w}_a}{\mathbf{u}}
=\eta_2\E_2[(q_2(a\mid\mathbf{x})-\pi_a)z],\\
|\zeta_a^+-\zeta_a|\le\eta_2\E|z|\le\eta_2,  
\end{gather*}

since $|q_2(a\mid\mathbf{x})-\pi_a|\le1$ and $z\sim N(0,1)$.
At initialization, $\zeta_a^{(0)}=\frac{2\sigma(a)-1}{2}\norm{\Delta\boldsymbol{\alpha}(\mathbf{W}_1)}$.
After $j\le k_2$ updates, the triangle inequality gives
\begin{equation}\label{eq:sep-zeta-crude}
\begin{aligned}
|\zeta_a^{(j)}|
&\le|\zeta_a^{(0)}|+\sum_{\ell=0}^{j-1}
|\zeta_a^{(\ell+1)}-\zeta_a^{(\ell)}|\\
&\le\tfrac12\norm{\Delta\boldsymbol{\alpha}(\mathbf{W}_1)}+j\eta_2\E|z|\\
&\le\tfrac12\norm{\Delta\boldsymbol{\alpha}(\mathbf{W}_1)}+\tau_2.
\end{aligned}
\end{equation}

Collect the coefficients in \eqref{eq:sep-row-decomposition} into a
matrix $\boldsymbol{\Theta}_j\in\R^{2M\times2}$, whose row indexed by token $a$ is
$(\zeta_a^{(j)},\omega_a^{(j)})$ after $j$ task-$2$ steps.
For $\mathbf{r}:=(z,y)^\top$, each entry of $\boldsymbol{\Theta}_j\mathbf{r}$ is the
corresponding token's logit:
\[
(\boldsymbol{\Theta}_j\mathbf{r})_a=\zeta_a^{(j)}z+\omega_a^{(j)}y
=\inner{\mathbf{w}_a^{(j)}}{\mathbf{x}}.
\]
Thus $\softmax(\boldsymbol{\Theta}_j\mathbf{r})$ is the model's probability vector.
The SFT population update \eqref{eq:pop-updates}, expressed in these
coordinates, becomes
\begin{equation}\label{eq:sep-update-map}
\boldsymbol{\Theta}_{j+1}=\boldsymbol{\Theta}_j+\eta_2\Delta(\boldsymbol{\Theta}_j),
\qquad
\Delta(\boldsymbol{\Theta}):=\E_2\bigl[(q_2(\cdot\mid\mathbf{x})-\softmax(\boldsymbol{\Theta}\mathbf{r}))
\mathbf{r}^\top\bigr].
\end{equation}
Here $\Delta(\boldsymbol{\Theta})$ denotes the population SFT update direction
at $\boldsymbol{\Theta}$. It has the same shape as $\boldsymbol{\Theta}$:
its first and second columns give the update directions for $\zeta_a$ and $\omega_a$,
respectively. Explicitly, its row indexed by $a$ has entries
\[
\Delta_{\zeta_a}(\boldsymbol{\Theta})=\E_2[(q_2(a\mid\mathbf{x})-\pi_a)z],
\qquad
\Delta_{\omega_a}(\boldsymbol{\Theta})=\E_2[(q_2(a\mid\mathbf{x})-\pi_a)y],
\]
where $\pi_a=[\softmax(\boldsymbol{\Theta}\mathbf{r})]_a$.
At $\mathbf{W}_1$ (represented by $\boldsymbol{\Theta}_0$), recall that
$\boldsymbol{\alpha}^{(b)}=M^{-1}\sum_{a\in\cA_b}\mathbf{w}_a$
is the class-$b$ semantic channel; see \eqref{eq:channels}.
Theorem~\ref{thm:rate}(ii) shows that the population update of
$\boldsymbol{\alpha}^{(1)}-\boldsymbol{\alpha}^{(0)}$ is parallel to
$\mathbf{u}$. From \eqref{eq:app-sft-mean}, a single SFT step on a prompt
$\mathbf{x}$ updates the two components of the semantic channel by
\[
\boldsymbol{\alpha}^{(b)+}-\boldsymbol{\alpha}^{(b)}
=\frac{\eta}{M}\bigl(\indic{b=b^\star(\mathbf{x})}-P_b(\mathbf{W},\mathbf{x})\bigr)\mathbf{x},
\qquad b\in\{0,1\}.
\]
Summing over $b$, we have
\begin{align*}
\boldsymbol{\alpha}^{(0)+}+\boldsymbol{\alpha}^{(1)+}
&=\boldsymbol{\alpha}^{(0)}+\boldsymbol{\alpha}^{(1)}
+\frac{\eta}{M}\sum_{b\in\{0,1\}}
\bigl(\indic{b=b^\star(\mathbf{x})}-P_b(\mathbf{W},\mathbf{x})\bigr)\mathbf{x}\\
&=\boldsymbol{\alpha}^{(0)}+\boldsymbol{\alpha}^{(1)}
+\frac{\eta}{M}\left(\indic{b^\star(\mathbf{x})=0}
+\indic{b^\star(\mathbf{x})=1}-P_0(\mathbf{W},\mathbf{x})-P_1(\mathbf{W},\mathbf{x})\right)\mathbf{x}\\
&=\boldsymbol{\alpha}^{(0)}+\boldsymbol{\alpha}^{(1)},
\end{align*}

where the indicators add to $1$ because exactly one class is
correct, and $P_0+P_1=1$ is by definition \eqref{eq:class-prob}, so $\boldsymbol{\alpha}^{(0)}+\boldsymbol{\alpha}^{(1)}$ is
unchanged by every SFT step. Writing
$\boldsymbol{\alpha}^{(b)}=\tfrac12\bigl[(\boldsymbol{\alpha}^{(0)}+\boldsymbol{\alpha}^{(1)})
\pm(\boldsymbol{\alpha}^{(1)}-\boldsymbol{\alpha}^{(0)})\bigr]$, the sum does
not move and the difference moves only along $\mathbf{u}$, so each semantic
channel's population update at $\mathbf{W}_1$ is parallel to $\mathbf{u}$.
Since $\mathbf{u}\perp\hat{\boldsymbol{\mu}}_2$, its projection onto
$\hat{\boldsymbol{\mu}}_2$ is zero.
By \eqref{eq:channels}, $\boldsymbol{\alpha}^{(b)}$ is the average of
$\mathbf{w}_a$ over $a\in\cA_b$, so its update is the average of the
updates of those token parameters. Thus, for the first task-$2$ step
from $\mathbf{W}_1$,
\[
\begin{aligned}
\inner{\E_2[\boldsymbol{\alpha}^{(b)+}-\boldsymbol{\alpha}^{(b)}]}{\hat{\boldsymbol{\mu}}_2}
&=\frac1M\sum_{a\in\cA_b}
\inner{\E_2[\mathbf{w}_a^+-\mathbf{w}_a]}{\hat{\boldsymbol{\mu}}_2}\\
&=\frac{\eta_2}{M}\sum_{a\in\cA_b}
\Delta_{\omega_a}(\boldsymbol{\Theta}_0).
\end{aligned}
\]
Here each $\Delta_{\omega_a}(\boldsymbol{\Theta}_0)$ describes the update
direction of one token's $\omega_a$ coefficient; summing and dividing
by $M$ gives the corresponding direction for the semantic channel.
The left-hand side is zero, so
\begin{equation}\label{eq:sep-alpha-omega-zero}
\frac1M\sum_{a\in\cA_b}\Delta_{\omega_a}(\boldsymbol{\Theta}_0)=0,
\qquad b\in\{0,1\}.
\end{equation}
Individual rows can still move along $\hat{\boldsymbol{\mu}}_2$ through
their within-class deviations
$\boldsymbol{\beta}_a=\mathbf{w}_a-\boldsymbol{\alpha}^{(\sigma(a))}$, and
Theorem~\ref{thm:style}(ii) quantifies this: since
$\boldsymbol{\beta}(\mathbf{W}_1)=\mathbf{0}$,
\eqref{eq:pop-drift-full} with $t=2$ gives, for $a\in\cA_b$,
\[
\E_2[\boldsymbol{\beta}_a^+-\boldsymbol{\beta}_a]
=\eta_2\Bigl(\rho_{2,b}(a)-\frac1M\Bigr)
\Bigl(\frac{\norm{\boldsymbol{\mu}_2}}{2}\hat{\boldsymbol{\mu}}_2
+\frac{2b-1}{\sqrt{2\pi}}\mathbf{u}\Bigr).
\]
As $\mathbf{w}_a=\boldsymbol{\alpha}^{(\sigma(a))}+\boldsymbol{\beta}_a$ and
the $\boldsymbol{\alpha}$-step has no $\hat{\boldsymbol{\mu}}_2$-component,
the $\hat{\boldsymbol{\mu}}_2$-component of the row step is that of the
$\boldsymbol{\beta}$-step. Using that the initial
rows are parallel to $\mathbf{u}$,
\begin{equation}\label{eq:sep-omega-init}
\omega_a(\boldsymbol{\Theta}_0)=0,
\qquad
\Delta_{\omega_a}(\boldsymbol{\Theta}_0)=\frac{\norm{\boldsymbol{\mu}_2}}{2}
\left(\rho_{2,\sigma(a)}(a)-\frac1M\right),
\qquad a\in\cA.
\end{equation}

Iterating \eqref{eq:sep-update-map} and using $\tau_2=k_2\eta_2$,
\begin{equation}\label{eq:sep-telescoped}
\boldsymbol{\Theta}_{k_2}-\boldsymbol{\Theta}_0-\tau_2\Delta(\boldsymbol{\Theta}_0)
=\eta_2\sum_{j=0}^{k_2-1}\bigl(\Delta(\boldsymbol{\Theta}_j)-\Delta(\boldsymbol{\Theta}_0)\bigr).
\end{equation}
Two bounds on the update map $\Delta$ defined in \eqref{eq:sep-update-map}
control the right-hand side. First,
$|q_2(a\mid\mathbf{x})-\pi_a|\le1$, $\E_2|z|\le1$, and
$\E_2|y|\le\norm{\boldsymbol{\mu}_2}+1$ (as $y\sim N(\norm{\boldsymbol{\mu}_2},1)$)
bound each entry of $\Delta$, so summing squares over the $2M$ rows,
\begin{equation}\label{eq:sep-update-norm-bound}
\sup_{\boldsymbol{\Theta}}\norm{\Delta(\boldsymbol{\Theta})}_F
\le\sqrt{2M\bigl(1+(\norm{\boldsymbol{\mu}_2}+1)^2\bigr)}
\le\sqrt{2M}\,(\norm{\boldsymbol{\mu}_2}+2).
\end{equation}
Second, the teacher term $q_2(\cdot\mid\mathbf{x})$ does not depend on
$\boldsymbol{\Theta}$ and therefore cancels when we subtract the two update maps:
\[
\Delta(\boldsymbol{\Theta})-\Delta(\boldsymbol{\Theta}')
=\E_2\!\left[
\bigl(\softmax(\boldsymbol{\Theta}'\mathbf{r})
-\softmax(\boldsymbol{\Theta}\mathbf{r})\bigr)\mathbf{r}^\top
\right].
\]
Taking the Frobenius norm, with Jensen's inequality we have
\[
\begin{aligned}
\norm{\Delta(\boldsymbol{\Theta})-\Delta(\boldsymbol{\Theta}')}_F
&\le\E_2\!\left[
\norm{\bigl(\softmax(\boldsymbol{\Theta}'\mathbf{r})
-\softmax(\boldsymbol{\Theta}\mathbf{r})\bigr)\mathbf{r}^\top}_F
\right]\\
&=\E_2\!\left[
\norm{\softmax(\boldsymbol{\Theta}'\mathbf{r})
-\softmax(\boldsymbol{\Theta}\mathbf{r})}\,\norm{\mathbf{r}}
\right]\\
&\le\E_2\bigl[\norm{(\boldsymbol{\Theta}-\boldsymbol{\Theta}')\mathbf{r}}\,\norm{\mathbf{r}}\bigr]\\
&\le\norm{\boldsymbol{\Theta}-\boldsymbol{\Theta}'}_F\,\E_2\norm{\mathbf{r}}^2,
\end{aligned}
\]
where the third inequality uses the $1$-Lipschitz property of softmax
(Lemma~\ref{lem:softmax-lipschitz}).
Finally, since $\mathbf{r}=(z,y)^\top$, with $z\sim N(0,1)$ and
$y\sim N(\norm{\boldsymbol{\mu}_2},1)$ under task $2$,
\[
\E_2\norm{\mathbf{r}}^2
=\E_2z^2+\E_2y^2
=1+\bigl(1+\norm{\boldsymbol{\mu}_2}^2\bigr)
=\norm{\boldsymbol{\mu}_2}^2+2,
\]
where we used $\E X^2=\operatorname{Var}(X)+(\E X)^2$ for each coordinate.
Thus
\begin{equation}\label{eq:sep-lipschitz-bound}
\norm{\Delta(\boldsymbol{\Theta})-\Delta(\boldsymbol{\Theta}')}_F
\le (\norm{\boldsymbol{\mu}_2}^2+2)\norm{\boldsymbol{\Theta}-\boldsymbol{\Theta}'}_F.
\end{equation}
Summing the update recursion in \eqref{eq:sep-update-map} over the first $j$ steps gives
\[
\boldsymbol{\Theta}_j-\boldsymbol{\Theta}_0
=\eta_2\sum_{\ell=0}^{j-1}\Delta(\boldsymbol{\Theta}_\ell).
\]
Using the triangle inequality and \eqref{eq:sep-update-norm-bound}, we obtain
\begin{equation}\label{eq:sep-displacement-bound}
\norm{\boldsymbol{\Theta}_j-\boldsymbol{\Theta}_0}_F
\le\eta_2\sum_{\ell=0}^{j-1}\norm{\Delta(\boldsymbol{\Theta}_\ell)}_F
\le j\eta_2\sqrt{2M}\,(\norm{\boldsymbol{\mu}_2}+2).
\end{equation}
Substituting \eqref{eq:sep-lipschitz-bound} and \eqref{eq:sep-displacement-bound} into
\eqref{eq:sep-telescoped} gives
\begin{align*}
&\norm{\boldsymbol{\Theta}_{k_2}-\boldsymbol{\Theta}_0-\tau_2\Delta(\boldsymbol{\Theta}_0)}_F\\
&\quad\le\eta_2\sum_{j=0}^{k_2-1}(\norm{\boldsymbol{\mu}_2}^2+2)\,j\eta_2\sqrt{2M}\,(\norm{\boldsymbol{\mu}_2}+2)\\
&\quad\le\tfrac12(\norm{\boldsymbol{\mu}_2}^2+2)\sqrt{2M}\,(\norm{\boldsymbol{\mu}_2}+2)\tau_2^2
:=C_{\mathrm{pop}}\tau_2^2,
\end{align*}
where $C_{\mathrm{pop}}:=\tfrac12(\norm{\boldsymbol{\mu}_2}^2+2)\sqrt{2M}\,(\norm{\boldsymbol{\mu}_2}+2)$.
For any matrix $\mathbf{A}$, each entry satisfies
$|A_{a,2}|\le\bigl(\sum_{b,c}|A_{b,c}|^2\bigr)^{1/2}=\norm{\mathbf{A}}_F$.
Applying this inequality to the $(a,2)$ entry of
$\boldsymbol{\Theta}_{k_2}-\boldsymbol{\Theta}_0-\tau_2\Delta(\boldsymbol{\Theta}_0)$ gives
\[
\left|\omega_a(\mathbf{W}_2)-\omega_a(\boldsymbol{\Theta}_0)
-\tau_2\Delta_{\omega_a}(\boldsymbol{\Theta}_0)\right|
\le C_{\mathrm{pop}}\tau_2^2.
\]
Substituting \eqref{eq:sep-omega-init} yields
\begin{equation}\label{eq:sep-phase-expansion}
\left|\omega_a(\mathbf{W}_2)
-\frac{\tau_2\norm{\boldsymbol{\mu}_2}}{2}
\left(\rho_{2,\sigma(a)}(a)-\frac1M\right)\right|
\le C_{\mathrm{pop}}\tau_2^2,
\qquad a\in\cA.
\end{equation}

Choose $a^\circ\in\argmax_{a\in\cA_1}\rho_{2,1}(a)$. For $a'\in\cA_0$,
assumption (b) gives $\rho_{2,1}(a^\circ)-\rho_{2,0}(a')\ge\Delta_\rho$.
Since $\sigma(a^\circ)=1$ and $\sigma(a')=0$, \eqref{eq:sep-phase-expansion}
gives a lower bound for $\omega_{a^\circ}(\mathbf{W}_2)$ and an upper bound for
$\omega_{a'}(\mathbf{W}_2)$:
\begin{equation}\label{eq:sep-token-style-bounds}
\begin{aligned}
\omega_{a^\circ}(\mathbf{W}_2)
&\ge\frac{\tau_2\norm{\boldsymbol{\mu}_2}}{2}
\left(\rho_{2,1}(a^\circ)-\frac1M\right)-C_{\mathrm{pop}}\tau_2^2,\\
\omega_{a'}(\mathbf{W}_2)
&\le\frac{\tau_2\norm{\boldsymbol{\mu}_2}}{2}
\left(\rho_{2,0}(a')-\frac1M\right)+C_{\mathrm{pop}}\tau_2^2.
\end{aligned}
\end{equation}
Subtracting the upper bound from the lower bound in \eqref{eq:sep-token-style-bounds} gives
\[
\begin{aligned}
&\omega_{a^\circ}(\mathbf{W}_2)-\omega_{a'}(\mathbf{W}_2)\\
&\quad\ge\frac{\tau_2\norm{\boldsymbol{\mu}_2}}{2}
\left[\left(\rho_{2,1}(a^\circ)-\frac1M\right)
-\left(\rho_{2,0}(a')-\frac1M\right)\right]-2C_{\mathrm{pop}}\tau_2^2\\
&\quad=\frac{\tau_2\norm{\boldsymbol{\mu}_2}}{2}
\bigl(\rho_{2,1}(a^\circ)-\rho_{2,0}(a')\bigr)-2C_{\mathrm{pop}}\tau_2^2\\
&\quad\ge\frac{\tau_2\norm{\boldsymbol{\mu}_2}\Delta_\rho}{2}-2C_{\mathrm{pop}}\tau_2^2.
\end{aligned}
\]
Set
\begin{equation}\label{eq:tau0}
\tau_0:=\frac{\norm{\boldsymbol{\mu}_2}\Delta_\rho}{8C_{\mathrm{pop}}}
=\frac{\norm{\boldsymbol{\mu}_2}\Delta_\rho}{4(\norm{\boldsymbol{\mu}_2}^2+2)\sqrt{2M}\,(\norm{\boldsymbol{\mu}_2}+2)}.
\end{equation}
For $\tau_2\le\tau_0$,
$2C_{\mathrm{pop}}\tau_2^2\le2C_{\mathrm{pop}}\tau_0\tau_2
=\tau_2\norm{\boldsymbol{\mu}_2}\Delta_\rho/4$, so
for every $a'\in\cA_0$,
\begin{equation}\label{eq:sep-transverse-gap}
\omega_{a^\circ}(\mathbf{W}_2)-\omega_{a'}(\mathbf{W}_2)
\ge\frac{\tau_2\norm{\boldsymbol{\mu}_2}\Delta_\rho}{4}>0.
\end{equation}

\paragraph{Step 3: an error region for both tasks.}
We construct a single error region and evaluate its probability under
both task distributions. Let
\begin{equation}\label{eq:sep-error-region}
\kappa:=\frac{4(\norm{\Delta\boldsymbol{\alpha}(\mathbf{W}_1)}+2\tau_2)}
{\norm{\boldsymbol{\mu}_2}\Delta_\rho\tau_2},
\qquad
E:=\{z<0,\;y>\kappa|z|\}.
\end{equation}
Fix $\mathbf{x}\in E$ and $a'\in\cA_0$, and evaluate all coefficients at
$\mathbf{W}_2$. By the triangle inequality and \eqref{eq:sep-zeta-crude}
with $j=k_2$,
\begin{equation}\label{eq:sep-zeta-difference}
\begin{aligned}
|\zeta_{a^\circ}-\zeta_{a'}|
&\le|\zeta_{a^\circ}|+|\zeta_{a'}|\\
&\le2\left(\tfrac12\norm{\Delta\boldsymbol{\alpha}(\mathbf{W}_1)}+\tau_2\right)
=\norm{\Delta\boldsymbol{\alpha}(\mathbf{W}_1)}+2\tau_2.
\end{aligned}
\end{equation}
Since $\kappa>0$ and $\mathbf{x}\in E$, we have $y>\kappa|z|>0$.
Using \eqref{eq:sep-row-decomposition}, \eqref{eq:sep-transverse-gap},
\eqref{eq:sep-zeta-difference}, and \eqref{eq:sep-error-region},
together with the definitions of $z$ and $y$,
\[
\begin{aligned}
\inner{\mathbf{w}_{a^\circ}-\mathbf{w}_{a'}}{\mathbf{x}}
&=(\zeta_{a^\circ}-\zeta_{a'})z+(\omega_{a^\circ}-\omega_{a'})y\\
&\ge-|\zeta_{a^\circ}-\zeta_{a'}|\,|z|+(\omega_{a^\circ}-\omega_{a'})y\\
&\ge-|\zeta_{a^\circ}-\zeta_{a'}|\,|z|
+\frac{\tau_2\norm{\boldsymbol{\mu}_2}\Delta_\rho}{4}\,y\\
&\ge-(\norm{\Delta\boldsymbol{\alpha}(\mathbf{W}_1)}+2\tau_2)|z|
+\frac{\tau_2\norm{\boldsymbol{\mu}_2}\Delta_\rho}{4}\,y\\
&=\frac{\tau_2\norm{\boldsymbol{\mu}_2}\Delta_\rho}{4}\bigl(y-\kappa|z|\bigr)>0.
\end{aligned}
\]
So on $E$ the highest-scoring token lies in $\cA_1$ and
$\hat b_{\mathbf{W}_2}(\mathbf{x})=1$, while $z<0$ means
$b^\star(\mathbf{x})=0$. This pointwise conclusion holds under either
task distribution, so $\mathcal{L}_r(\mathbf{W}_2)\ge\Prob_r(E)$ for
$r\in\{1,2\}$.
Under task $r$, Lemma~\ref{lem:gaussian-projection} gives independent
coordinates $z\sim N(0,1)$ and
$y\sim N(\inner{\boldsymbol{\mu}_r}{\hat{\boldsymbol{\mu}}_2},1)$.
Integrating over $z<0$ with the substitution $z=-t$, $t>0$, yields
\begin{equation}\label{eq:sep-task-error}
\begin{aligned}
\mathcal{L}_r(\mathbf{W}_2)\ge\Prob_r(E)
&=\int_0^\infty\phi(t)\,\Prob_r\bigl(y>\kappa t\bigr)\,\dd t\\
&=\int_0^\infty\phi(t)\Phi\!\left(
\inner{\boldsymbol{\mu}_r}{\hat{\boldsymbol{\mu}}_2}-\kappa t\right)\,\dd t>0,
\qquad r\in\{1,2\},
\end{aligned}
\end{equation}
where strict positivity follows because the integrand is positive for
every $t>0$.

Substituting the definitions of $\kappa$ and
$\hat{\boldsymbol{\mu}}_2=\boldsymbol{\mu}_2/\norm{\boldsymbol{\mu}_2}$
into \eqref{eq:sep-task-error} proves part (ii) for both tasks.
\end{proof}

\newpage
\subsection{Proof of Theorem~\ref{thm:offaxis-energy}}\label{app:proof-offaxis-energy}

\begin{proof}[Proof of Theorem~\ref{thm:offaxis-energy}]
\leavevmode
\paragraph{Step 1: conditional sampling and integrability.}
For either method, let $\mathcal F_s$ denote its training history
strictly before step $s$, including the previous teacher draws for SFT.
For the trajectory under consideration, write $\mathbf W_s:=\mathbf W(s)$ and
$\mathbf m_s:=\boldsymbol\mu_{t(s)}$. Set
$\mathbf P_\perp:=\mathbf I_p-\mathbf u\mathbf u^\top$ and define
\begin{equation}\label{eq:app-energy-projections}
\mathbf R_s:=\mathbf W_s^\perp=\mathbf W_s\mathbf P_\perp,
\qquad \mathbf y_s:=\mathbf P_\perp\mathbf x_s,
\qquad z_s:=\inner{\mathbf u}{\mathbf x_s},
\qquad
\mathbf v_s:=(\Delta\boldsymbol\alpha(s))^\perp.
\end{equation}
Thus $\norm{\mathbf R_s}_F^2=\norm{\mathbf W_s^\perp}_F^2$.
The matrix $\mathbf W_s$ is $\mathcal F_s$-measurable, whereas the
current prompt is fresh. By Lemma~\ref{lem:gaussian-projection} and
$\mathbf m_s\perp\mathbf u$, conditionally on $\mathcal F_s$,
\begin{equation}\label{eq:app-energy-sampling}
\begin{aligned}
&z_s\sim N(0,1),\qquad
\mathbf y_s\sim N(\mathbf m_s,\mathbf P_\perp),
\qquad z_s\indep\mathbf y_s,\\
&\E[\mathbf y_s\mid\mathcal F_s]=\mathbf m_s,
\qquad
\E[\norm{\mathbf y_s}^2\mid\mathcal F_s]
=\operatorname{tr}(\mathbf P_\perp)+\norm{\mathbf m_s}^2
=p-1+\norm{\mathbf m_s}^2.
\end{aligned}
\end{equation}
All expectations used below are finite. Indeed, for SFT the coefficient
vector in \eqref{eq:updates} has norm at most $\sqrt2$, since it is
the difference of two probability vectors. For RFT, each coefficient
has absolute value at most $\pi_{\mathbf W_s}(a\mid\mathbf x_s)$,
because \eqref{eq:reward} gives $0\le R,\bar R\le1$; hence the
coefficient vector has norm at most $1$. In both cases,
\[
\norm{\mathbf W_s}_F
\le\sqrt2\eta\sum_{j=0}^{s-1}\norm{\mathbf x_j}.
\]
The right-hand side has finite moments of every order for finite $s$,
as the prompts are Gaussian and the task sequence is fixed.

\paragraph{Step 2: the exact semantic--style energy decomposition.}
For either update in \eqref{eq:updates}, the sum of all token parameters
is preserved. In SFT,
\[
\sum_{a\in\cA}(\mathbf w_a^+-\mathbf w_a)
=\eta\left(\sum_a\indic{a=a^\star}-\sum_a\pi_{\mathbf W}(a\mid\mathbf x)\right)
\mathbf x=\mathbf0.
\]
In RFT, the definition of $\bar R_{\mathbf W}$ following
\eqref{eq:updates} gives
\[
\begin{aligned}
\sum_{a\in\cA}(\mathbf w_a^+-\mathbf w_a)
&=\eta\left(\sum_a\pi_{\mathbf W}(a\mid\mathbf x)R(a,\mathbf x)
-\bar R_{\mathbf W}(\mathbf x)\sum_a\pi_{\mathbf W}(a\mid\mathbf x)\right)
\mathbf x\\
&=\eta(\bar R_{\mathbf W}(\mathbf x)-\bar R_{\mathbf W}(\mathbf x))\mathbf x
=\mathbf0.
\end{aligned}
\]
Since $\mathbf W_0=\mathbf0$, the definition of class means in
\eqref{eq:channels} therefore implies
\begin{equation}\label{eq:app-energy-centered-means}
M\bigl(\boldsymbol\alpha^{(0)}(s)+\boldsymbol\alpha^{(1)}(s)\bigr)
=\sum_a\mathbf w_a(s)=\mathbf0,
\qquad
\boldsymbol\alpha^{(1)\perp}(s)=\frac{\mathbf v_s}{2},
\quad
\boldsymbol\alpha^{(0)\perp}(s)=-\frac{\mathbf v_s}{2}.
\end{equation}
Furthermore, \eqref{eq:channels} gives
$\mathbf w_a^\perp=\boldsymbol\alpha^{(b)\perp}+\boldsymbol\beta_a^\perp$
for $a\in\cA_b$ and
$\sum_{a\in\cA_b}\boldsymbol\beta_a^\perp=\mathbf0$.
Consequently, for each class,
\[
\begin{aligned}
\sum_{a\in\cA_b}\norm{\mathbf w_a^\perp}^2
&=M\norm{\boldsymbol\alpha^{(b)\perp}}^2
+2\inner{\boldsymbol\alpha^{(b)\perp}}
{\sum_{a\in\cA_b}\boldsymbol\beta_a^\perp}
+\sum_{a\in\cA_b}\norm{\boldsymbol\beta_a^\perp}^2\\
&=M\norm{\boldsymbol\alpha^{(b)\perp}}^2
+\sum_{a\in\cA_b}\norm{\boldsymbol\beta_a^\perp}^2.
\end{aligned}
\]
Summing over the two classes and using
\eqref{eq:app-energy-centered-means} yields the pathwise identity
\begin{equation}\label{eq:app-transverse-channel-energy}
\begin{aligned}
\norm{\mathbf R_s}_F^2
&=M\sum_{b=0}^1\norm{\boldsymbol\alpha^{(b)\perp}(s)}^2
+\sum_{a\in\cA}\norm{\boldsymbol\beta_a^\perp(s)}^2\\
&=\frac M2\norm{\mathbf v_s}^2
+\norm{\boldsymbol\beta^\perp(s)}_F^2.
\end{aligned}
\end{equation}
This proves \eqref{eq:offaxis-energy-decomposition} for both methods.

\paragraph{Step 3: the SFT update and the teacher's conditional mean.}
We first prove part (ii). Write
$\boldsymbol\pi_s=(\pi_{\mathbf W_s}(a\mid\mathbf x_s))_{a\in\cA}$
as a column vector, and let $\mathbf e_{a_s^\star}\in\R^{2M}$ be the
one-hot vector of the teacher token. Right-multiplying the SFT update
in \eqref{eq:updates} by $\mathbf P_\perp$ gives
\begin{equation}\label{eq:app-full-transverse-update}
\mathbf R_{s+1}
=\mathbf R_s+\eta(\mathbf e_{a_s^\star}-\boldsymbol\pi_s)\mathbf y_s^\top.
\end{equation}
Under \eqref{eq:teacher}, the vector
$\mathbf q_{t(s)}(z_s)=(q_{t(s)}(a\mid\mathbf x_s))_{a\in\cA}$ depends
on the prompt only through the sign of $z_s$. The fresh teacher draw
therefore satisfies
\[
\E[\mathbf e_{a_s^\star}\mid\mathcal F_s,\mathbf x_s]
=\mathbf q_{t(s)}(z_s).
\]
Using iterated conditional expectation and then
\eqref{eq:app-energy-sampling}, we obtain
\begin{equation}\label{eq:app-transverse-teacher-mean}
\begin{aligned}
\E[\mathbf e_{a_s^\star}\mathbf y_s^\top\mid\mathcal F_s]
&=\E[\mathbf q_{t(s)}(z_s)\mathbf y_s^\top\mid\mathcal F_s]
=\bar{\mathbf q}_s\mathbf m_s^\top,\\
\bar{\mathbf q}_s&:=\E[\mathbf q_{t(s)}(z_s)],
\qquad (\bar{\mathbf q}_s)_a=\frac12\rho_{t(s),\sigma(a)}(a).
\end{aligned}
\end{equation}
The factor $1/2$ follows from $\Prob(z_s>0)=\Prob(z_s<0)=1/2$.

\paragraph{Step 4: Gaussian integration by parts for the SFT policy term.}
For a fixed $\mathbf W$ and coordinate $j$, differentiating the softmax gives
\[
\frac{\partial\pi_{\mathbf W}(a\mid\mathbf x)}{\partial x_j}
=\pi_{\mathbf W}(a\mid\mathbf x)
\left(w_{a,j}-\sum_{a'}\pi_{\mathbf W}(a'\mid\mathbf x)w_{a',j}\right).
\]
Thus its Jacobian, with one row per token, is
$(\operatorname{diag}(\boldsymbol\pi)-\boldsymbol\pi\boldsymbol\pi^\top)\mathbf W$.
For $\mathbf x\sim N(\mathbf m_s,\mathbf I_p)$, integration by parts
coordinatewise yields $\E[(x_j-m_{s,j})\pi_{\mathbf W}(a\mid\mathbf x)]
=\E[\partial\pi_{\mathbf W}(a\mid\mathbf x)/\partial x_j]$;
the boundary term vanishes because the softmax is bounded, and its
derivatives are bounded for fixed $\mathbf W$.
Define $\bar{\boldsymbol\pi}_s:=\E[\boldsymbol\pi_s\mid\mathcal F_s]$.
Conditional on $\mathcal F_s$, this identity and right multiplication
by $\mathbf P_\perp$ give, since $\mathbf P_\perp\mathbf m_s=\mathbf m_s$,
\begin{equation}\label{eq:app-transverse-policy-stein}
\begin{aligned}
\E[\boldsymbol\pi_s\mathbf y_s^\top\mid\mathcal F_s]
&=\bar{\boldsymbol\pi}_s\mathbf m_s^\top+\mathbf H_s\mathbf R_s,\\
\mathbf H_s&:=\E\left[
\operatorname{diag}(\boldsymbol\pi_s)-\boldsymbol\pi_s\boldsymbol\pi_s^\top
\,\middle|\,\mathcal F_s\right]\succeq0.
\end{aligned}
\end{equation}
Indeed, for any $\mathbf v\in\R^{2M}$ and probability vector
$\boldsymbol\pi$,
\[
\mathbf v^\top
(\operatorname{diag}(\boldsymbol\pi)-\boldsymbol\pi\boldsymbol\pi^\top)
\mathbf v
=\sum_a\pi_a v_a^2-\left(\sum_a\pi_a v_a\right)^2\ge0,
\]
so taking the conditional expectation preserves positive semidefiniteness.

\paragraph{Step 5: the SFT energy recursion.}
The Frobenius inner product is
$\langle\mathbf U,\mathbf V\rangle_F=\operatorname{tr}(\mathbf U^\top\mathbf V)$,
and $\norm{\mathbf v\mathbf y^\top}_F^2=\norm{\mathbf v}^2\norm{\mathbf y}^2$.
Expanding \eqref{eq:app-full-transverse-update} therefore gives
\[
\begin{aligned}
\norm{\mathbf R_{s+1}}_F^2
={}&\norm{\mathbf R_s}_F^2
+2\eta\langle\mathbf R_s,
(\mathbf e_{a_s^\star}-\boldsymbol\pi_s)\mathbf y_s^\top\rangle_F\\
&+\eta^2\norm{\mathbf e_{a_s^\star}-\boldsymbol\pi_s}^2
\norm{\mathbf y_s}^2.
\end{aligned}
\]
Equations~\eqref{eq:app-transverse-teacher-mean}
and~\eqref{eq:app-transverse-policy-stein}, and the
$\mathcal F_s$-measurability of $\mathbf R_s$, yield the exact identity
\begin{equation}\label{eq:app-full-transverse-energy}
\begin{aligned}
\E[\norm{\mathbf R_{s+1}}_F^2\mid\mathcal F_s]
={}&\norm{\mathbf R_s}_F^2
-2\eta\operatorname{tr}(\mathbf R_s^\top\mathbf H_s\mathbf R_s)\\
&+2\eta\langle\mathbf R_s,
(\bar{\mathbf q}_s-\bar{\boldsymbol\pi}_s)\mathbf m_s^\top\rangle_F\\
&+\eta^2\E\left[
\norm{\mathbf e_{a_s^\star}-\boldsymbol\pi_s}^2\norm{\mathbf y_s}^2
\,\middle|\,\mathcal F_s\right].
\end{aligned}
\end{equation}
Since $\mathbf H_s\succeq0$, each column $\mathbf r$ of $\mathbf R_s$
satisfies $\mathbf r^\top\mathbf H_s\mathbf r\ge0$; summing over columns
shows that the trace in \eqref{eq:app-full-transverse-energy} is nonnegative. Both
$\bar{\mathbf q}_s$ and $\bar{\boldsymbol\pi}_s$ are probability vectors,
so their squared norms are at most $1$ and their inner product is
nonnegative. Thus
$\norm{\bar{\mathbf q}_s-\bar{\boldsymbol\pi}_s}\le\sqrt2$, and
Cauchy--Schwarz gives
\[
\langle\mathbf R_s,
(\bar{\mathbf q}_s-\bar{\boldsymbol\pi}_s)\mathbf m_s^\top\rangle_F
\le\sqrt2\,\norm{\mathbf R_s}_F\norm{\mathbf m_s}.
\]
Moreover, for
every demonstrated token $a$ and probability vector $\boldsymbol\pi$,
\[
\norm{\mathbf e_a-\boldsymbol\pi}^2
=(1-\pi_a)^2+\sum_{a'\ne a}\pi_{a'}^2
\le(1-\pi_a)^2+\left(\sum_{a'\ne a}\pi_{a'}\right)^2
=2(1-\pi_a)^2\le2.
\]
Substituting these estimates and \eqref{eq:app-energy-sampling} into
\eqref{eq:app-full-transverse-energy} gives
\begin{equation}\label{eq:app-full-transverse-bound-step}
\begin{aligned}
\E[\norm{\mathbf R_{s+1}}_F^2\mid\mathcal F_s]
\le{}&\norm{\mathbf R_s}_F^2
+2\sqrt2\eta\norm{\mathbf m_s}\norm{\mathbf R_s}_F\\
&+2\eta^2\bigl(p-1+\norm{\mathbf m_s}^2\bigr).
\end{aligned}
\end{equation}
\paragraph{Step 6: iterate the SFT recursion.}
Set $e_s:=\E\norm{\mathbf R_s}_F^2$. The means $\mathbf m_s$ are
deterministic. Taking expectations in
\eqref{eq:app-full-transverse-bound-step} and using
$\E\norm{\mathbf R_s}_F\le\sqrt{e_s}$ yields
\begin{equation}\label{eq:app-sft-energy-scalar-recursion}
\begin{aligned}
e_{s+1}
&\le e_s+2\sqrt2\eta\norm{\mathbf m_s}\sqrt{e_s}
+2\eta^2\norm{\mathbf m_s}^2+2\eta^2(p-1)\\
&=\left(\sqrt{e_s}+\sqrt2\eta\norm{\mathbf m_s}\right)^2
+2\eta^2(p-1).
\end{aligned}
\end{equation}
To iterate this bound, let
\[
A_s:=\sqrt2\eta\sum_{j=0}^{s-1}\norm{\mathbf m_j},
\qquad V_s:=2\eta^2s(p-1).
\]
We claim $\sqrt{e_s}\le A_s+\sqrt{V_s}$. The claim holds at $s=0$.
For the induction step, use the elementary inequality
\[
\sqrt{(A+\sqrt V+c)^2+d}
\le A+c+\sqrt{V+d},\qquad A,V,c,d\ge0.
\]
Indeed, the square of the right-hand side minus the square of the
left-hand side is
$2(A+c)(\sqrt{V+d}-\sqrt V)\ge0$, and both sides are nonnegative.
Apply it with $A=A_s$, $V=V_s$,
$c=\sqrt2\eta\norm{\mathbf m_s}$, and $d=2\eta^2(p-1)$.
By \eqref{eq:app-sft-energy-scalar-recursion} and the induction hypothesis,
\[
\begin{aligned}
\sqrt{e_{s+1}}
&\le\sqrt{(A_s+\sqrt{V_s}+c)^2+d}\\
&\le A_s+c+\sqrt{V_s+d}
=A_{s+1}+\sqrt{V_{s+1}}.
\end{aligned}
\]
Evaluating the resulting bound at $s=N$ gives
\begin{equation}\label{eq:app-full-transverse-bound}
\E\norm{\mathbf R_N}_F^2
\le2\eta^2\left(\sqrt{N(p-1)}
+\sum_{s=0}^{N-1}\norm{\boldsymbol\mu_{t(s)}}\right)^2.
\end{equation}
By \eqref{eq:app-energy-projections}, this is precisely
\eqref{eq:offaxis-energy}, proving part (ii).

\paragraph{Step 7: reduce RFT to its semantic channel.}
We now consider the RFT trajectory and prove part (i). Since
$\mathbf W_0=\mathbf0$, Theorem~\ref{thm:style}(i) gives
$\boldsymbol\beta(s)=\mathbf0$ for every step, for each realization
of the prompts. By \eqref{eq:policy} and \eqref{eq:class-prob}, equal
class sizes then give
\[
P_1(\mathbf W_s,\mathbf x)
=\frac{M e^{\inner{\boldsymbol\alpha^{(1)}(s)}{\mathbf x}}}
{M e^{\inner{\boldsymbol\alpha^{(0)}(s)}{\mathbf x}}
+M e^{\inner{\boldsymbol\alpha^{(1)}(s)}{\mathbf x}}}
=\frac{1}{1+e^{-\inner{\Delta\boldsymbol\alpha(s)}{\mathbf x}}}.
\]
Define
\begin{equation}\label{eq:app-rft-energy-scalar-defs}
c_s:=\inner{\Delta\boldsymbol\alpha(s)}{\mathbf u},
\qquad \lambda:=\frac{2\eta}{M},
\qquad
f(r):=\frac{e^r}{(1+e^r)^2}
=\frac{1}{4\cosh^2(r/2)}.
\end{equation}
In particular, $f$ is even and $0<f(r)\le1/4$ for every $r$.
Since $P_0=1-P_1$, either value of the correct class satisfies
\begin{equation}\label{eq:app-rft-energy-prob-product}
P_{b^\star}(\mathbf W_s,\mathbf x_s)
\bigl(1-P_{b^\star}(\mathbf W_s,\mathbf x_s)\bigr)
=f\!\left(c_sz_s+\inner{\mathbf v_s}{\mathbf y_s}\right).
\end{equation}
Here we used \eqref{eq:app-energy-projections} to write
$\Delta\boldsymbol\alpha(s)=c_s\mathbf u+\mathbf v_s$ and
$\mathbf x_s=z_s\mathbf u+\mathbf y_s$.
Let $\varepsilon_s:=2b^\star(\mathbf x_s)-1$.
By \eqref{eq:reward}, $\varepsilon_s=\sign(z_s)$ almost surely and
$\varepsilon_s z_s=|z_s|$; the event $z_s=0$ has probability zero.
Substituting \eqref{eq:app-rft-energy-prob-product} into the exact
semantic update \eqref{eq:semantic-steps}, and projecting onto
$\mathbf u$ and $\mathbf u^\perp$, gives
\begin{equation}\label{eq:app-rft-energy-updates}
\begin{aligned}
c_{s+1}
&=c_s+\lambda|z_s|
f\!\left(c_sz_s+\inner{\mathbf v_s}{\mathbf y_s}\right),\\
\mathbf v_{s+1}
&=\mathbf v_s+\lambda\varepsilon_s
f\!\left(c_sz_s+\inner{\mathbf v_s}{\mathbf y_s}\right)\mathbf y_s.
\end{aligned}
\end{equation}
The first line and $c_0=0$ imply $c_{s+1}\ge c_s\ge0$ for all $s$,
without any upper restriction on $\eta$.

\paragraph{Step 8: the conditional RFT drift has a nonpositive radial component.}
Condition on $\mathcal F_s$, so that $c_s\ge0$ and $\mathbf v_s$
are fixed. For $c\ge0$ and $t\in\R$, define
\begin{equation}\label{eq:app-rft-energy-paired-drift}
\begin{aligned}
h_c(t)
&:=\E_{Z\sim N(0,1)}[\sign(Z)f(cZ+t)]\\
&=\int_0^\infty\phi(z)\bigl[f(cz+t)-f(cz-t)\bigr]\,\dd z,
\end{aligned}
\end{equation}
where $\phi$ is the standard normal density. The second line follows
by splitting the integral at zero, substituting $z\mapsto-z$ in
the negative half, and using the evenness of both $\phi$ and $f$.
Differentiating \eqref{eq:app-rft-energy-scalar-defs} gives
$f'(r)=-f(r)\tanh(r/2)$, so $f$ is nonincreasing on $[0,\infty)$.
For $c,z,t\ge0$, the inequality $|cz+t|\ge|cz-t|$ therefore implies
$f(cz+t)\le f(cz-t)$. Equation~\eqref{eq:app-rft-energy-paired-drift}
then yields $h_c(t)\le0$ for $t\ge0$.
The same equation gives $h_c(-t)=-h_c(t)$, hence
\begin{equation}\label{eq:app-rft-energy-drift-sign}
t\,h_c(t)\le0
\qquad\text{for every }c\ge0\text{ and }t\in\R.
\end{equation}
Write the conditional mean of the unscaled off-axis increment as
\[
\mathbf g_s:=\E\!\left[
\varepsilon_s f\!\left(c_sz_s+\inner{\mathbf v_s}{\mathbf y_s}\right)
\mathbf y_s\,\middle|\,\mathcal F_s\right].
\]
By the conditional independence in \eqref{eq:app-energy-sampling},
we may first average over $z_s$ while holding $\mathbf y_s$ fixed.
Using \eqref{eq:app-rft-energy-paired-drift} then gives
\begin{equation}\label{eq:app-rft-energy-radial-drift}
\begin{aligned}
\inner{\mathbf v_s}{\mathbf g_s}
&=\E\!\left[
\inner{\mathbf v_s}{\mathbf y_s}\,
\varepsilon_s f\!\left(c_sz_s+\inner{\mathbf v_s}{\mathbf y_s}\right)
\,\middle|\,\mathcal F_s\right]\\
&=\E\!\left[
\inner{\mathbf v_s}{\mathbf y_s}\,
h_{c_s}\!\left(\inner{\mathbf v_s}{\mathbf y_s}\right)
\,\middle|\,\mathcal F_s\right]\le0,
\end{aligned}
\end{equation}
where the last inequality applies \eqref{eq:app-rft-energy-drift-sign}
pointwise in $\mathbf y_s$. This argument allows any
$\mathbf m_s\perp\mathbf u$; it does not require zero task means.

\paragraph{Step 9: the RFT energy recursion and final bound.}
Expanding the second line of \eqref{eq:app-rft-energy-updates},
taking conditional expectations, and using $\varepsilon_s^2=1$ gives
\begin{equation}\label{eq:app-rft-energy-recursion}
\begin{aligned}
\E[\norm{\mathbf v_{s+1}}^2\mid\mathcal F_s]
={}&\norm{\mathbf v_s}^2+2\lambda\inner{\mathbf v_s}{\mathbf g_s}\\
&+\lambda^2\E\!\left[
f\!\left(c_sz_s+\inner{\mathbf v_s}{\mathbf y_s}\right)^2
\norm{\mathbf y_s}^2\,\middle|\,\mathcal F_s\right]\\
\le{}&\norm{\mathbf v_s}^2
+\frac{\lambda^2}{16}(p-1+\norm{\mathbf m_s}^2)\\
={}&\norm{\mathbf v_s}^2
+\frac{\eta^2}{4M^2}(p-1+\norm{\mathbf m_s}^2).
\end{aligned}
\end{equation}
The inequality uses \eqref{eq:app-rft-energy-radial-drift},
$f^2\le1/16$ from \eqref{eq:app-rft-energy-scalar-defs}, and the
second moment in \eqref{eq:app-energy-sampling}. Taking expectations
and summing \eqref{eq:app-rft-energy-recursion} from $s=0$ to $N-1$,
with $\mathbf v_0=\mathbf0$, yields
\begin{equation}\label{eq:app-rft-semantic-energy-bound}
\E\norm{\mathbf v_N}^2
\le\frac{\eta^2}{4M^2}
\sum_{s=0}^{N-1}(p-1+\norm{\mathbf m_s}^2).
\end{equation}
Finally, $\boldsymbol\beta(N)=\mathbf0$ and
\eqref{eq:app-transverse-channel-energy} give
\[
\begin{aligned}
\E\norm{\mathbf W_{\mathrm{RFT}}^\perp(N)}_F^2
&=\frac M2\E\norm{\mathbf v_N}^2\\
&\le\frac{\eta^2}{8M}
\left(N(p-1)+\sum_{s=0}^{N-1}\norm{\boldsymbol\mu_{t(s)}}^2\right),
\end{aligned}
\]
which is \eqref{eq:rft-offaxis-energy} and proves part (i).
Both parts include $N=0$, with empty sums and zero initial energy.
Fresh sampling is essential to the conditional calculations in
\eqref{eq:app-energy-sampling} and the subsequent expectations;
the proof does not directly apply to reuse of a fixed prompt sample.
For RFT, exact per-prompt gradients are also essential to the pathwise
style invariance used in Step~7; sampled action gradients are discussed
separately in Appendix~\ref{app:sampled}.
\end{proof}

\section{Additional results}\label{app:additional}

\subsection{Exact and sampled policy gradients}\label{app:sampled}

The invariance in Theorem~\ref{thm:style}(i) concerns the exact
per-prompt gradient \eqref{eq:updates}. Here ``exact'' means that the
expectation over actions is evaluated by summing over all tokens at the
current policy and prompt. When this summation is computationally
impractical, REINFORCE provides an unbiased gradient estimate using
actions sampled from the current policy. This subsection explains how action sampling may change the
style dynamics.

\paragraph{An unbiased gradient estimate.}
Fix $\mathbf{W}$ and $\mathbf{x}$, and draw
$A\sim\pi_{\mathbf{W}}(\cdot\mid\mathbf{x})$. The REINFORCE estimator
without a baseline \citep{williams1992simple} gives the update
\begin{equation}\label{eq:app-sampled-update}
\mathbf{w}_a^+=\mathbf{w}_a+\eta\,\widehat{\mathbf{g}}_a(A),
\qquad
\widehat{\mathbf{g}}_a(A)
=R(A,\mathbf{x})
\bigl(\indic{a=A}-\pi_{\mathbf{W}}(a\mid\mathbf{x})\bigr)\mathbf{x}.
\end{equation}
All expectations below are conditional on the current
$\mathbf{W},\mathbf{x}$ and average only over the sampled action. Since
$\E_A[R(A,\mathbf{x})\indic{a=A}]
=\pi_{\mathbf{W}}(a\mid\mathbf{x})R(a,\mathbf{x})$, we have
\begin{equation}\label{eq:app-sampled-unbiased}
\E_A[\widehat{\mathbf{g}}_a(A)]
=\pi_{\mathbf{W}}(a\mid\mathbf{x})
\bigl(R(a,\mathbf{x})-\bar R_{\mathbf{W}}(\mathbf{x})\bigr)\mathbf{x}.
\end{equation}
Thus the sampled update agrees with the RFT update in \eqref{eq:updates}
in conditional mean.

\paragraph{Style noise from a single sampled action.}
Subtracting the mean update within $a$'s class gives
\begin{equation}\label{eq:app-sampled-style}
\begin{aligned}
\Delta\boldsymbol{\beta}_a
&:=\boldsymbol{\beta}_a^+-\boldsymbol{\beta}_a\\
&=\eta R(A,\mathbf{x})\Bigl[
\indic{a=A}-\frac1M\indic{\sigma(A)=\sigma(a)}
-\delta_a(\mathbf{W},\mathbf{x})\Bigr]\mathbf{x},
\end{aligned}
\end{equation}
where $\delta_a$ is defined in \eqref{eq:app-style-deviation}. When $\boldsymbol{\beta}=\mathbf{0}$,
$\delta_a=0$ and \eqref{eq:app-sampled-unbiased} implies
$\E_A[\Delta\boldsymbol{\beta}_a]=\mathbf{0}$. Nevertheless, if $M>1$
and the sampled token is rewarded, its style increment is
$\eta(1-1/M)\mathbf{x}$, while every other token in the same class receives
$-\eta\mathbf{x}/M$. Tokens in the other class receive zero style
increment. For $\eta>0$ and $\mathbf{x}\ne\mathbf{0}$, a rewarded step
produces $\boldsymbol{\beta}^+\ne\mathbf{0}$.
An unrewarded sample produces no update for this estimator.

Let $\Delta\boldsymbol{\beta}$ collect the style increments over all
tokens. When $\boldsymbol{\beta}=\mathbf{0}$, the conditional squared style increment is
\begin{equation}\label{eq:app-sampled-style-variance}
\E_A\bigl[\norm{\Delta\boldsymbol{\beta}}_F^2\bigr]
=\eta^2 P_{b^\star(\mathbf{x})}(\mathbf{W},\mathbf{x})
\left(1-\frac1M\right)\norm{\mathbf{x}}^2.
\end{equation}
Indeed, a rewarded sample occurs with probability
$P_{b^\star(\mathbf{x})}(\mathbf{W},\mathbf{x})$, and its
squared coefficients sum to
$(1-1/M)^2+(M-1)/M^2=1-1/M$. Since the conditional mean increment is
zero when $\boldsymbol{\beta}=\mathbf{0}$, \eqref{eq:app-sampled-style-variance} is also the total
conditional variance. The corresponding off-axis expression replaces
$\mathbf{x}$ by $\mathbf{x}^{\perp}$. Hence a zero mean style increment
can coexist with positive expected style energy after one step.

\paragraph{Averaging actions and following the trajectory.}
If each update averages $K$ independent action-gradient estimates
computed at the same $\mathbf{W},\mathbf{x}$, its conditional mean is the
RFT update in \eqref{eq:updates}, while the variance in
\eqref{eq:app-sampled-style-variance} is divided by $K$. Finite $K$
therefore reduces style noise but does not guarantee that every update
preserves $\boldsymbol{\beta}=\mathbf{0}$.
Once a sampled trajectory develops $\boldsymbol{\beta}\ne\mathbf{0}$, the zero mean
style-increment statement above no longer applies: subsequent
conditional means follow the exact gradient at the new, generally
nonuniform policy. Theorem~\ref{thm:style}(i) consequently does not imply
zero style drift along sampled trajectories.

We compare REINFORCE updates with different numbers of sampled actions per
prompt against exact RFT in Appendix~\ref{app:robustness-ablations}
(Figure~\ref{fig:ablation-forgetting}(b,d)).

\subsection{Limiting behavior under continued task-2 training}
\label{app:limiting-behavior}

Theorem~\ref{thm:separation} establishes positive SFT semantic error on
both tasks over a finite training interval. We now consider continued training on task $2$.
In the unregularized population setting, sufficiently small fixed steps
eventually drive the semantic errors on both tasks to zero, even though
the teacher's style preferences can cause errors at intermediate times.

\begin{proposition}[Vanishing semantic error under continued population SFT]
\label{prop:limiting-error}
Under the setup of Section~\ref{sec:setting}, let
$\mathbf{W}^{(0)}=\mathbf{W}_1$ be any finite checkpoint, and run population
SFT on task $2$ with a fixed step size satisfying
\begin{equation}\label{eq:limiting-step-size}
0<\eta_2\le\frac{2}{1+\norm{\boldsymbol{\mu}_2}^2}.
\end{equation}
Then, as $k\to\infty$ ,
\begin{equation}\label{eq:limiting-errors}
\mathcal{L}_1(\mathbf{W}^{(k)})\longrightarrow0,
\qquad
\mathcal{L}_2(\mathbf{W}^{(k)})\longrightarrow0.
\end{equation}
In particular, from the checkpoint in Theorem~\ref{thm:separation},
\[
\cF^{(2)}_{\mathrm{SFT}}\longrightarrow0,
\qquad
\cG^{(2)}_{\mathrm{SFT}}\longrightarrow0,
\]
whereas RFT keeps both quantities identically zero throughout training.
\end{proposition}

\begin{proof}
Let $f_2$ denote the task-$2$ population token cross-entropy, and let
$f_{2,*}$ be its infimum over finite parameter matrices:
\[
f_2(\mathbf{W})
:=\E_2\!\left[-\sum_{a\in\cA}q_2(a\mid\mathbf{x})
\log\pi_{\mathbf{W}}(a\mid\mathbf{x})\right],
\qquad
f_{2,*}:=\inf_{\mathbf{V}} f_2(\mathbf{V}).
\]
We first show that approaching this infimum forces semantic error to
vanish, and then verify that population optimization approaches it.

\paragraph{Separating class and style losses.}
For $a\in\cA_b$, write
$\pi_{\mathbf{W}}(a\mid\mathbf{x},b)
:=\pi_{\mathbf{W}}(a\mid\mathbf{x})/P_b(\mathbf{W},\mathbf{x})$.
This conditional policy depends only on the style channel
$\boldsymbol{\beta}$. Since the teacher is always semantically correct,
\begin{equation}\label{eq:limiting-loss-decomposition}
\begin{aligned}
f_2(\mathbf{W})&=C_2(\mathbf{W})+S_2(\boldsymbol{\beta}),\\
C_2(\mathbf{W})&:=\E_2[-\log P_{b^\star(\mathbf{x})}(\mathbf{W},\mathbf{x})],\\
S_2(\boldsymbol{\beta})&:=\E_2\!\left[
-\sum_{a\in\cA_{b^\star(\mathbf{x})}}\rho_{2,b^\star(\mathbf{x})}(a)
\log\pi_{\mathbf{W}}(a\mid\mathbf{x},b^\star(\mathbf{x}))\right].
\end{aligned}
\end{equation}
Let $S_{2,*}$ be the infimum of $S_2$ over finite style channels whose
rows sum to zero within each class. We claim that $f_{2,*}=S_{2,*}$.
The inequality $f_{2,*}\ge S_{2,*}$ follows from $C_2\ge0$.
Conversely, fix such a style channel and set
$\mathbf{w}_a^{[s]}=\boldsymbol{\beta}_a
+(2\sigma(a)-1)s\mathbf{u}/2$ for $s\ge0$.
This leaves $S_2$ unchanged. With $z=\inner{\mathbf{u}}{\mathbf{x}}$
and $b=b^\star(\mathbf{x})$, the class loss is
\[
-\log P_b(\mathbf{W}^{[s]},\mathbf{x})
=\log\!\left(1+e^{-s|z|}
\frac{\sum_{a\in\cA_{1-b}}e^{\inner{\boldsymbol{\beta}_a}{\mathbf{x}}}}
{\sum_{a\in\cA_b}e^{\inner{\boldsymbol{\beta}_a}{\mathbf{x}}}}\right).
\]
It tends to zero almost surely and is bounded by its value at $s=0$,
which is integrable because the fixed linear logits have finite Gaussian
first moments. Dominated convergence yields
$C_2(\mathbf{W}^{[s]})\to0$, hence
$f_{2,*}\le S_2(\boldsymbol{\beta})$ for every finite style channel.
Taking the infimum proves the claim and gives
\begin{equation}\label{eq:limiting-class-gap}
0\le C_2(\mathbf{W})\le f_2(\mathbf{W})-f_{2,*}.
\end{equation}

\paragraph{From objective values to semantic errors.}
If greedy decoding selects a wrong-class token, every correct-class
token has probability at most that selected token's probability, and
thus at most the total wrong-class mass. Consequently,
$P_b(\mathbf{W},\mathbf{x})\le M(1-P_b(\mathbf{W},\mathbf{x}))$
on the error event, including ties resolved to the wrong class.
The class loss there is at least $\log(1+1/M)$, so
\begin{equation}\label{eq:limiting-semantic-bound}
\mathcal{L}_2(\mathbf{W})
\le\frac{C_2(\mathbf{W})}{\log(1+1/M)}
\le\frac{f_2(\mathbf{W})-f_{2,*}}{\log(1+1/M)}.
\end{equation}
Both tasks share the same semantic rule and hence the same error event
$E_{\mathbf{W}}=\{\mathbf{x}:\hat b_{\mathbf{W}}(\mathbf{x})
\ne b^\star(\mathbf{x})\}$.
Writing $\boldsymbol{\delta}=\boldsymbol{\mu}_1-\boldsymbol{\mu}_2$,
the density ratio of their prompt distributions is
\[
r(\mathbf{x})
=\exp\!\left(\inner{\boldsymbol{\delta}}{\mathbf{x}-\boldsymbol{\mu}_2}
-\tfrac12\norm{\boldsymbol{\delta}}^2\right),
\qquad
\E_2[r(\mathbf{x})^2]=e^{\norm{\boldsymbol{\delta}}^2}.
\]
Cauchy--Schwarz therefore gives, uniformly in $\mathbf{W}$,
\begin{equation}\label{eq:limiting-task-transfer}
\mathcal{L}_1(\mathbf{W})
=\E_2[r(\mathbf{x})\indic{\mathbf{x}\in E_{\mathbf{W}}}]
\le e^{\norm{\boldsymbol{\mu}_1-\boldsymbol{\mu}_2}^2/2}
\sqrt{\mathcal{L}_2(\mathbf{W})}.
\end{equation}
Thus any sequence with $f_2(\mathbf{W})\to f_{2,*}$ has vanishing
semantic error on both tasks.

\paragraph{Convergence of the population objective.}
The function $f_2$ is convex. For any perturbation matrix $\mathbf{A}$,
its Hessian satisfies
\[
\begin{aligned}
D^2f_2(\mathbf{W})[\mathbf{A},\mathbf{A}]
&=\E_2\!\left[(\mathbf{A}\mathbf{x})^\top
(\operatorname{diag}\boldsymbol{\pi}
-\boldsymbol{\pi}\boldsymbol{\pi}^\top)(\mathbf{A}\mathbf{x})\right]\\
&\le\tfrac12\E_2\norm{\mathbf{A}\mathbf{x}}^2
\le\tfrac12(1+\norm{\boldsymbol{\mu}_2}^2)\norm{\mathbf{A}}_F^2,
\end{aligned}
\]
where $\boldsymbol{\pi}=\pi_{\mathbf{W}}(\cdot\mid\mathbf{x})$.
Here the categorical covariance has operator norm at most $1/2$, since
its absolute row sums are $2\pi_a(1-\pi_a)\le1/2$, and
$\E_2[\mathbf{x}\mathbf{x}^\top]
=\mathbf{I}_p+\boldsymbol{\mu}_2\boldsymbol{\mu}_2^\top$.
Hence $f_2$ is $L$-smooth for
$L=(1+\norm{\boldsymbol{\mu}_2}^2)/2$.
Under \eqref{eq:limiting-step-size}, convexity and the descent inequality
give, for every finite comparator $\mathbf{V}$,
\[
f_2(\mathbf{W}^{(j+1)})-f_2(\mathbf{V})
\le\frac{\norm{\mathbf{W}^{(j)}-\mathbf{V}}_F^2
-\norm{\mathbf{W}^{(j+1)}-\mathbf{V}}_F^2}{2\eta_2}.
\]
Summing and using monotonicity of the objective yields
\begin{equation}\label{eq:limiting-comparator}
f_2(\mathbf{W}^{(k)})-f_2(\mathbf{V})
\le\frac{\norm{\mathbf{W}_1-\mathbf{V}}_F^2}{2k\eta_2}.
\end{equation}
First letting $k\to\infty$ and then taking the infimum over $\mathbf{V}$
proves $f_2(\mathbf{W}^{(k)})\to f_{2,*}$ without assuming that a finite
minimizer exists.
Equations~\eqref{eq:limiting-semantic-bound} and
\eqref{eq:limiting-task-transfer} now imply \eqref{eq:limiting-errors}.
When $\mathcal{L}_1(\mathbf{W}_1)=0$, forgetting equals the final
task-$1$ error, and generalization error is the average of the two task
errors, proving the asserted metric limits.
\end{proof}

\paragraph{What converges.}
The objective approaches its infimum and semantic errors vanish, but
the parameters do not converge to a finite minimizer. Indeed, if
$\mathbf{U}$ has rows $(2\sigma(a)-1)\mathbf{u}/2$, then at every finite
$\mathbf{W}$,
\[
\left.\frac{\dd}{\dd s}f_2(\mathbf{W}+s\mathbf{U})\right|_{s=0}
=-\E_2\!\left[|z|\bigl(1-P_{b^\star(\mathbf{x})}
(\mathbf{W},\mathbf{x})\bigr)\right]<0.
\]
Thus no finite matrix attains $f_{2,*}$, and any sequence approaching
this infimum must have $\norm{\mathbf{W}}_F\to\infty$; otherwise a
bounded subsequence would yield a finite minimizer by continuity.
The token cross-entropy need not tend to zero, and exact matching of the
teacher's style preferences is unnecessary for vanishing semantic error.
The finite-time errors in Theorem~\ref{thm:separation} therefore do not
persist at a positive level in this population limit. This does not
assert monotone semantic improvement or zero error after finitely many
steps. 

\subsection{Numerical illustration of finite-time forgetting}\label{app:two-task-dynamics}

We illustrate the finite-time separation in
Theorem~\ref{thm:separation} and the subsequent recovery described in
Appendix~\ref{app:limiting-behavior} by varying training time at a fixed
step size.

\paragraph{Setup.}
We take $p=2$, $M=2$, $\mathbf{u}=(1,0)^\top$, and Gaussian task means
$\boldsymbol{\mu}_1=(0,2)^\top$ and
$\boldsymbol{\mu}_2=(0,3)^\top$. The task-$2$ teachers are
$\rho_{2,0}=(0.5,0.5)$ and $\rho_{2,1}=(0.99,0.01)$, giving
$\Delta_\rho=0.49$.
Both methods start from the constructed checkpoint
$\mathbf{w}_a=(2\sigma(a)-1)\mathbf{u}/2$, so that
$\Delta\boldsymbol{\alpha}=\mathbf{u}$ and
$\boldsymbol{\beta}=\mathbf{0}$.
This checkpoint already classifies both tasks perfectly because they
share the same semantic rule. We then train only on task $2$, using
population updates with fixed step size $\eta=10^{-3}$ for $10^6$ steps,
up to $\tau=k\eta=1000$.
Here $\mathcal{L}_1$ equals forgetting, and
$\cG^{(2)}=(\mathcal{L}_1+\mathcal{L}_2)/2$.
The guaranteed interval in \eqref{eq:tau0} ends at
$\tau_0\approx0.00334$; the chosen step size also satisfies the
sufficient condition \eqref{eq:limiting-step-size} for eventual recovery.

\paragraph{Numerical evaluation.}
We evaluate the SFT teacher term analytically and its policy term by
deterministic Gaussian quadrature, with $12$ Gauss--Legendre nodes on
each of six intervals in positive $z=\inner{\mathbf{u}}{\mathbf{x}}$
and their negative reflections, and $40$ Gauss--Hermite nodes in the
orthogonal coordinate ($5760$ nodes in total; $|z|\le11$).
RFT follows its scalar on-axis semantic recursion from
Theorem~\ref{thm:rate}(ii), preserving zero style and zero semantic error.
SFT errors are computed by conditioning on $z$, evaluating the Gaussian
probability of the greedy decision region in the other coordinate,
and integrating over $z$.

\begin{figure}[htbp]
\centering
\includegraphics[width=0.9\linewidth]{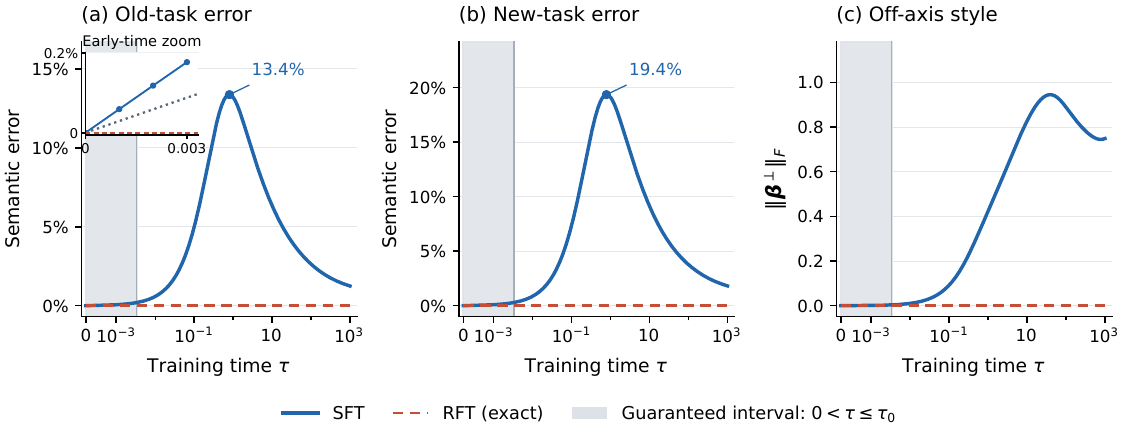}
\caption{\textbf{Two-task population dynamics at a fixed step size.}
\emph{(a)} Old-task error, equal to forgetting from the common perfectly
fitted checkpoint. \emph{(b)} New-task error.
\emph{(c)} Off-axis style norm $\norm{\boldsymbol{\beta}^{\perp}}_F$.
The horizontal axis is logarithmic outside a small linear neighborhood
of zero. Shading marks $0<\tau\le\tau_0$; the inset compares the first
three SFT steps with the old-task error lower bound in
Theorem~\ref{thm:separation} (dotted).
Later times illustrate behavior beyond that guaranteed interval.}
\label{fig:theorem3-timecourse}
\end{figure}

\paragraph{Results.}
In Figure~\ref{fig:theorem3-timecourse}, SFT's old- and new-task errors
rise from zero to approximately $13.4\%$ and $19.4\%$, respectively,
near $\tau=0.8$, then decline to $1.25\%$ and $1.81\%$ at $\tau=1000$.
The shaded steps lie inside the guaranteed interval and have
positive old-task error above the theorem's lower bound.
RFT retains zero error on both tasks throughout.
SFT's error reduction occurs while its off-axis style norm remains
nonzero (as the semantic margin grows).
Thus this example exhibits substantial recovery without eliminating
style drift.

\section{Control and ablation experiments}\label{app:controls}
\renewcommand{\topfraction}{0.9}
\renewcommand{\bottomfraction}{0.5}
\renewcommand{\textfraction}{0.07}
\setcounter{topnumber}{3}

\subsection{Control experiments}
\label{app:control-experiments}

We isolate the roles of task means and teacher preferences in
population style drift.
Unless specified below, all settings follow Section~\ref{sec:experiments}.

\paragraph{Zero task means.}
Setting $\boldsymbol{\mu}_t=\mathbf{0}$ removes population off-axis
style drift. Both population methods then achieve negligible
generalization error and forgetting (Figure~\ref{fig:pareto-pop-zero});
their small residual off-axis energies reflect Monte Carlo prompt
sampling. Finite-sample SFT still develops greater off-axis energy and
a worse generalization--forgetting tradeoff than RFT, illustrating
sampling effects even when population drift is absent.

\paragraph{Uniform teachers.}
Keeping the nonzero task means of the main experiment, we set
$\mathbf{v}_{t,b}=\mathbf{0}$ in \eqref{eq:rho}, giving
$\rho_{t,b}(a)=1/M$. Removing teacher preferences also eliminates the
population performance gap (Figure~\ref{fig:pareto-pop-uniform}).
Population SFT now preserves $\boldsymbol\beta=\mathbf0$ because its
updates are constant within each semantic class, including under
Monte Carlo prompt averaging. Its residual off-axis energy therefore
comes from the semantic component. Finite-sample SFT still introduces
style fluctuations by sampling individual teacher tokens. RFT is
unchanged from Figure~\ref{fig:pareto-pop}, since it does not use the
teacher and the prompts and seeds are shared.

\begin{figure}[htbp]
\centering
\includegraphics[width=\linewidth]{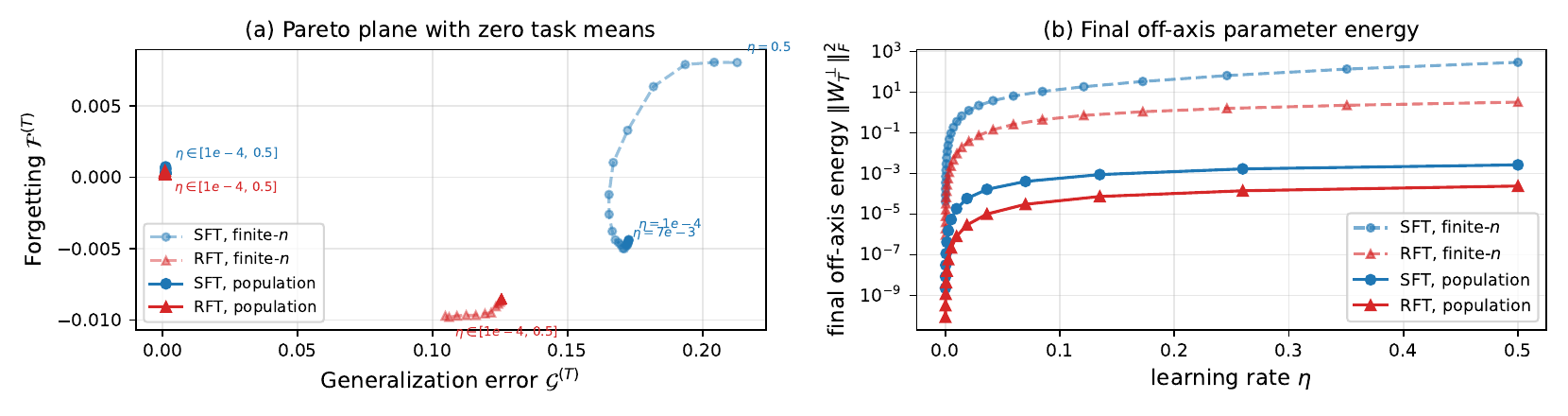}
\caption{\textbf{Zero task means.} We set
$\boldsymbol{\mu}_t=\mathbf0$ for every task. As in
Figure~\ref{fig:pareto-pop}, panel~(a) shows generalization error versus
forgetting, and panel~(b) shows final off-axis energy versus learning
rate. Solid and dashed curves denote population and finite-sample
updates, respectively.}
\label{fig:pareto-pop-zero}
\end{figure}

\begin{figure}[htbp]
\centering
\includegraphics[width=\linewidth]{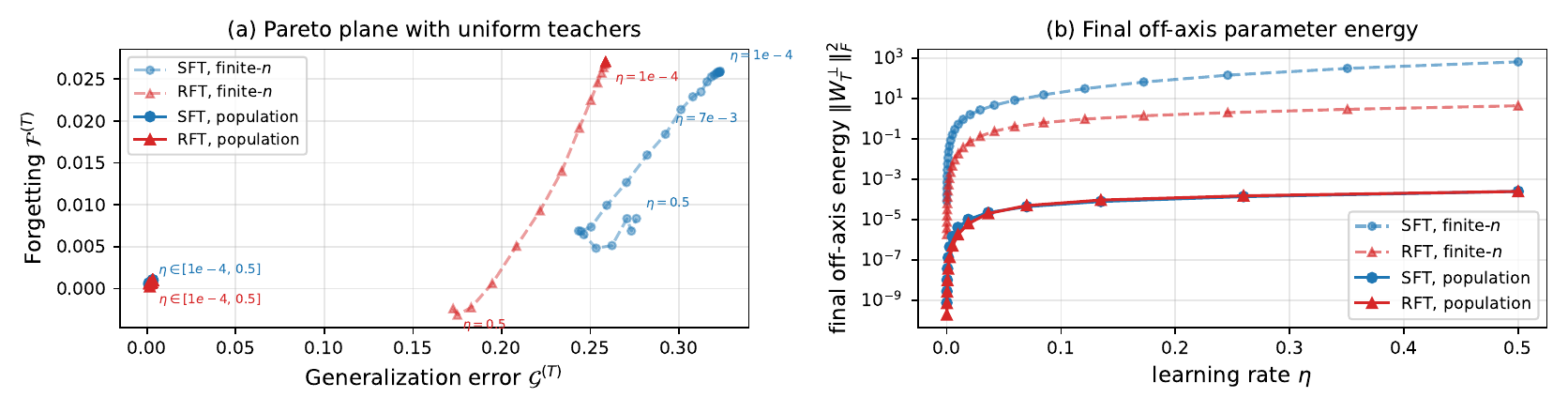}
\caption{\textbf{Uniform teachers.} We retain the nonzero Gaussian task
means of Figure~\ref{fig:pareto-pop} and set $\rho_{t,b}(a)=1/M$ for
every task and class. The layout and line styles match
Figure~\ref{fig:pareto-pop-zero}.}
\label{fig:pareto-pop-uniform}
\end{figure}

\subsection{Ablation experiments}
\label{app:robustness-ablations}

We examine sensitivity to three assumptions: Gaussian prompts, correct
semantic rewards, and exact action expectations. Unless specified otherwise, we follow the setup of
Section~\ref{sec:experiments}.

\paragraph{Non-Gaussian prompts.}
We keep the task means of Figure~\ref{fig:pareto-pop} and replace the
Gaussian prompt noise by independent
$\operatorname{Laplace}(0,1/\sqrt2)$ coordinates,
$\mathbf x=\boldsymbol{\mu}_t+\boldsymbol{\varepsilon}$, which preserves
the means and the identity covariance. The results are nearly unchanged
from Figure~\ref{fig:pareto-pop} (Figure~\ref{fig:pareto-pop-laplace}):
population SFT develops essentially the same off-axis energy and
forgetting, and the finite-sample curves are close. Population RFT keeps
forgetting below $0.3\%$ but acquires a small off-axis energy (about
$3\times10^{-2}$ at $\eta=0.5$, versus $2.5\times10^{-4}$ with Gaussian
prompts) and error up to $1.5\%$: its algebraic style invariance holds
for any prompt law, whereas its exactly on-axis population direction
relies on Gaussian prompts.

\begin{figure}[htbp]
\centering
\includegraphics[width=\linewidth]{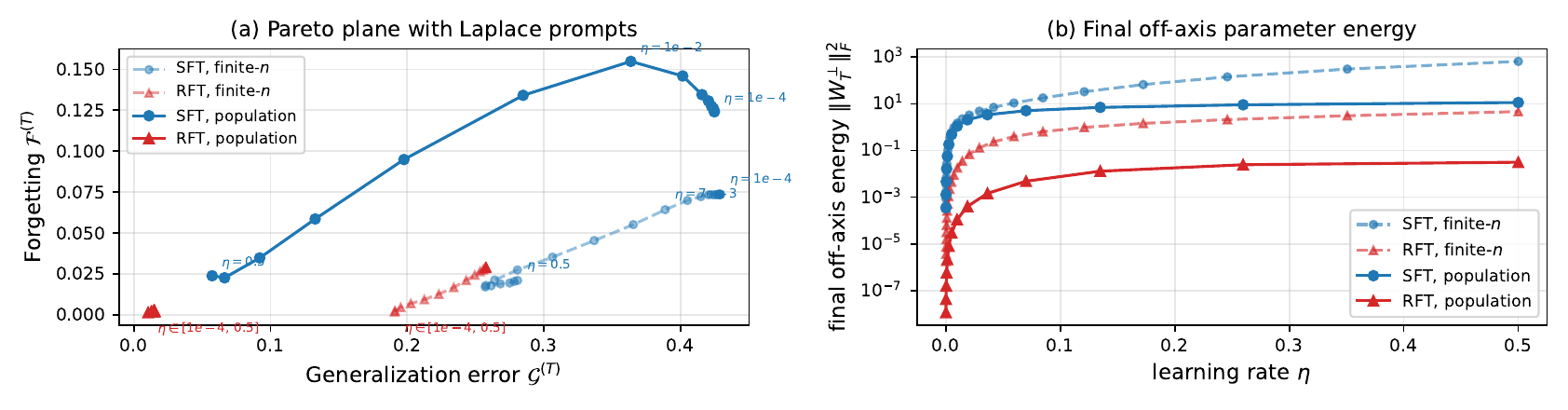}
\caption{\textbf{Non-Gaussian prompts.} Prompt noise coordinates are
drawn independently from $\operatorname{Laplace}(0,1/\sqrt2)$ around the
task means of Figure~\ref{fig:pareto-pop}, giving identity covariance.
The layout and line styles match Figure~\ref{fig:pareto-pop-zero}.}
\label{fig:pareto-pop-laplace}
\end{figure}

\paragraph{Imperfect rewards.}
For $q\in\{0,0.05,0.1,0.2\}$, we apply either a
\emph{class flip}, replacing the entire reward vector by
$1-R(\cdot,\mathbf x)$ with probability $q$, or \emph{token flips},
independently reversing each token's reward with probability $q$.
Corruptions are sampled independently for each prompt, including within
Monte Carlo (MC) batches. RFT sums exactly over actions using the
corrupted rewards; SFT continues to receive correct demonstrations.
Both methods are evaluated using clean semantic labels.

Class flips preserve within-class reward symmetry and hence RFT's
zero-style invariance, whereas token flips introduce style fluctuations.
Across the tested noise levels, MC RFT exhibits little
forgetting, while token flips induce more forgetting than class flips
in finite-sample runs (Figure~\ref{fig:ablation-forgetting}(a,c)).

\begin{figure}[htbp]
\centering
\includegraphics[width=0.5\linewidth]{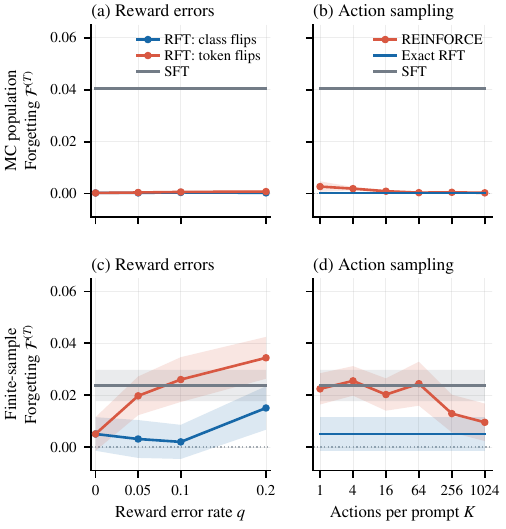}
\caption{\textbf{Forgetting under two assumption ablations.}
Top: MC population updates; bottom: finite-sample updates.
Columns vary reward corruption and the number of REINFORCE actions per
prompt, respectively.
All panels use $\eta\approx0.121$ and the same vertical scale.
Training prompts are paired across methods.
Shading gives pointwise $95\%$ Student-$t$
intervals over $10$ MC seeds and mean $\pm1.96$ standard errors over
$50$ finite-sample seeds, conditional on the fixed world and test sets.
Negative forgetting indicates backward improvement.}
\label{fig:ablation-forgetting}
\end{figure}

\paragraph{Finite action sampling.}
With clean rewards, we replace the exact sum over $2M=12$ actions by
the no-baseline REINFORCE estimator of Appendix~\ref{app:sampled},
averaging $K\in\{1,4,16,64,256,1024\}$ independent actions per prompt at the
current policy. The estimator is unbiased for every $K$ by
\eqref{eq:app-sampled-unbiased}. MC updates additionally average over
$10{,}000$ prompts, so even $K=1$ uses $10{,}000$ action samples per
update. Prompt and update counts are matched within each protocol,
but action-query and computation budgets differ across $K$ and the
exact baseline.

At the reference learning rate, REINFORCE forgetting moves closer to the exact
RFT baseline as $K$ increases in both settings
(Figure~\ref{fig:ablation-forgetting}(b,d)).

\end{document}

%% file: Figures/fig_reward_cells.tex
\begin{tikzpicture}[
    token/.style={draw,rounded corners=2pt,thick,fill=white,inner xsep=4pt,inner ysep=3pt,font=\small},
    cellbox/.style={draw,rounded corners=5pt,thick,inner sep=7pt},
    stylept/.style={circle,draw,thick,inner sep=1.7pt},
    note/.style={align=center,font=\small},
    arrow/.style={->,thick},
    node distance=4pt
]
\node[note,font=\small] (promptleft) at (1.45,3.35) {$\mathbf{x}$: \texttt{"2+2=?"}};
\node[note,font=\small] (promptright) at (5.45,3.35) {$\mathbf{x}$: \texttt{"1+4=?"}};

\node[token,font=\scriptsize] (lc1) at (1.45,2.45) {\texttt{"4"}};
\node[token,font=\scriptsize,below=of lc1] (lc2) {\texttt{"four"}};
\node[token,font=\scriptsize,below=of lc2] (lc3) {\texttt{"The answer is 4"}};
\node[cellbox,draw=linkblue,fit=(lc1)(lc2)(lc3),
      label={[font=\scriptsize]left:{$\cA_0:\ R=1$}}] (lcellzero) {};

\node[token,font=\scriptsize] (l1a) at (1.45,0.18) {\texttt{"5"}};
\node[token,font=\scriptsize,below=of l1a] (l1b) {\texttt{"five"}};
\node[token,font=\scriptsize,below=of l1b] (l1c) {\texttt{"It equals 5"}};
\node[cellbox,draw=red!70!black,fit=(l1a)(l1b)(l1c),
      label={[font=\scriptsize]left:{$\cA_1:\ R=0$}}] (lcellone) {};

\node[token,font=\scriptsize] (rc1) at (5.45,2.45) {\texttt{"4"}};
\node[token,font=\scriptsize,below=of rc1] (rc2) {\texttt{"four"}};
\node[token,font=\scriptsize,below=of rc2] (rc3) {\texttt{"The answer is 4"}};
\node[cellbox,draw=linkblue,fit=(rc1)(rc2)(rc3),
      label={[font=\scriptsize]right:{$\cA_0:\ R=0$}}] (rcellzero) {};

\node[token,font=\scriptsize] (r1a) at (5.45,0.18) {\texttt{"5"}};
\node[token,font=\scriptsize,below=of r1a] (r1b) {\texttt{"five"}};
\node[token,font=\scriptsize,below=of r1b] (r1c) {\texttt{"It equals 5"}};
\node[cellbox,draw=red!70!black,fit=(r1a)(r1b)(r1c),
      label={[font=\scriptsize]right:{$\cA_1:\ R=1$}}] (rcellone) {};

\node[note,font=\small] (rewardlabel) at (3.3,-2.15)
{\textbf{(a) Semantic classes}};

\begin{scope}[yshift=0.35cm]
\draw[->,thin] (-0.55,-6.3) -- (6.9,-6.3);
\node[font=\scriptsize,anchor=north] at (3.2,-6.5) {terse $\to$ explanatory};
\draw[->,thin] (-0.55,-6.3) -- (-0.55,-3.45);
\node[font=\scriptsize,rotate=90,anchor=south] at (-0.85,-4.9) {numeric $\to$ verbal};
\node[stylept,fill=linkblue!25,label={[font=\ttfamily\scriptsize]90:{"4"}}]                at (0.45,-5.95) {};
\node[stylept,fill=linkblue!25,label={[font=\ttfamily\scriptsize,xshift=6pt]135:{"four"}}]            at (0.65,-3.75) {};
\node[stylept,fill=linkblue!25,label={[font=\ttfamily\scriptsize]270:{"The answer is 4"}}] at (5.35,-4.55) {};
\node[stylept,fill=red!20,label={[font=\ttfamily\scriptsize]90:{"5"}}]             at (1.15,-5.95) {};
\node[stylept,fill=red!20,label={[font=\ttfamily\scriptsize]90:{"five"}}]          at (1.55,-3.75) {};
\node[stylept,fill=red!20,label={[font=\ttfamily\scriptsize]90:{"It equals 5"}}]   at (4.65,-4.45) {};
\draw[arrow,linkblue] (2.0,-5.5) -- (3.45,-4.6)
node[midway,above,sloped,font=\scriptsize] {teacher preference};
\draw[thick,dashed,green!60!black] (0.8,-5.85) ellipse (0.72 and 0.42);
\node[font=\scriptsize\bfseries,green!60!black,anchor=west] at (1.6,-5.85) {Task 1};
\draw[thick,dashed,orange!85!black] (1.1,-3.65) ellipse (0.9 and 0.4);
\node[font=\scriptsize\bfseries,orange!85!black,anchor=west] at (2.1,-3.65) {Task 2};
\draw[thick,dashed,linkblue] (5.0,-4.5) ellipse (0.85 and 0.5);
\node[font=\scriptsize\bfseries,linkblue,anchor=west] at (5.9,-4.5) {Task 3};
\node[note,font=\small] (stylelabel) at (3.3,-7.1)
{\textbf{(b) Style space}};
\end{scope}
\end{tikzpicture}

%% file: references.bib
@inproceedings{howard2018universal,
  title     = {Universal Language Model Fine-tuning for Text Classification},
  author    = {Howard, Jeremy and Ruder, Sebastian},
  booktitle = {Proceedings of the 56th Annual Meeting of the Association for Computational Linguistics (Volume 1: Long Papers)},
  pages     = {328--339},
  year      = {2018},
  publisher = {Association for Computational Linguistics},
  doi       = {10.18653/v1/P18-1031},
  url       = {https://aclanthology.org/P18-1031/}
}

@techreport{radford2018improving,
  title       = {Improving Language Understanding by Generative Pre-Training},
  author      = {Radford, Alec and Narasimhan, Karthik and Salimans, Tim and Sutskever, Ilya},
  institution = {OpenAI},
  year        = {2018},
  url         = {https://cdn.openai.com/research-covers/language-unsupervised/language_understanding_paper.pdf}
}

@inproceedings{devlin2019bert,
  title     = {{BERT}: Pre-training of Deep Bidirectional Transformers for Language Understanding},
  author    = {Devlin, Jacob and Chang, Ming-Wei and Lee, Kenton and Toutanova, Kristina},
  booktitle = {Proceedings of the 2019 Conference of the North American Chapter of the Association for Computational Linguistics: Human Language Technologies, Volume 1 (Long and Short Papers)},
  pages     = {4171--4186},
  year      = {2019},
  publisher = {Association for Computational Linguistics},
  doi       = {10.18653/v1/N19-1423},
  url       = {https://aclanthology.org/N19-1423/}
}

@article{raffel2020exploring,
  title   = {Exploring the Limits of Transfer Learning with a Unified Text-to-Text Transformer},
  author  = {Raffel, Colin and Shazeer, Noam and Roberts, Adam and Lee, Katherine and Narang, Sharan and Matena, Michael and Zhou, Yanqi and Li, Wei and Liu, Peter J.},
  journal = {Journal of Machine Learning Research},
  volume  = {21},
  number  = {140},
  pages   = {1--67},
  year    = {2020},
  url     = {https://www.jmlr.org/papers/v21/20-074.html}
}

@article{chung2024scaling,
  title   = {Scaling Instruction-Finetuned Language Models},
  author  = {Chung, Hyung Won and Hou, Le and Longpre, Shayne and Zoph, Barret and Tay, Yi and Fedus, William and Li, Yunxuan and Wang, Xuezhi and Dehghani, Mostafa and Brahma, Siddhartha and others},
  journal = {Journal of Machine Learning Research},
  volume  = {25},
  number  = {70},
  pages   = {1--53},
  year    = {2024},
  url     = {https://www.jmlr.org/papers/v25/23-0870.html}
}

@article{guo2025deepseekr1,
  title   = {{DeepSeek-R1} incentivizes reasoning in {LLMs} through reinforcement learning},
  author  = {Guo, Daya and Yang, Dejian and Zhang, Haowei and Song, Junxiao and Wang, Peiyi and Zhu, Qihao and Xu, Runxin and Zhang, Ruoyu and Ma, Shirong and Bi, Xiao and others},
  journal = {Nature},
  volume  = {645},
  pages   = {633--638},
  year    = {2025},
  doi     = {10.1038/s41586-025-09422-z},
  url     = {https://www.nature.com/articles/s41586-025-09422-z}
}

@inproceedings{rolnick2019experience,
  title     = {Experience Replay for Continual Learning},
  author    = {Rolnick, David and Ahuja, Arun and Schwarz, Jonathan and Lillicrap, Timothy and Wayne, Gregory},
  booktitle = {Advances in Neural Information Processing Systems (NeurIPS)},
  volume    = {32},
  year      = {2019},
  url       = {https://proceedings.neurips.cc/paper/2019/hash/fa7cdfad1a5aaf8370ebeda47a1ff1c3-Abstract.html}
}

@inproceedings{zenke2017continual,
  title     = {Continual Learning Through Synaptic Intelligence},
  author    = {Zenke, Friedemann and Poole, Ben and Ganguli, Surya},
  booktitle = {International Conference on Machine Learning (ICML)},
  series    = {Proceedings of Machine Learning Research},
  volume    = {70},
  pages     = {3987--3995},
  year      = {2017},
  url       = {https://proceedings.mlr.press/v70/zenke17a.html}
}

@inproceedings{mallya2018packnet,
  title     = {{PackNet}: Adding Multiple Tasks to a Single Network by Iterative Pruning},
  author    = {Mallya, Arun and Lazebnik, Svetlana},
  booktitle = {Proceedings of the IEEE Conference on Computer Vision and Pattern Recognition (CVPR)},
  pages     = {7765--7773},
  year      = {2018},
  url       = {https://openaccess.thecvf.com/content_cvpr_2018/html/Mallya_PackNet_Adding_Multiple_CVPR_2018_paper.html}
}

@article{luo2025empirical,
  title   = {An Empirical Study of Catastrophic Forgetting in Large Language Models During Continual Fine-Tuning},
  author  = {Luo, Yun and Yang, Zhen and Meng, Fandong and Li, Yafu and Zhou, Jie and Zhang, Yue},
  journal = {IEEE Transactions on Audio, Speech and Language Processing},
  volume  = {33},
  pages   = {3776--3786},
  year    = {2025},
  doi     = {10.1109/TASLPRO.2025.3606231},
  url     = {https://arxiv.org/abs/2308.08747}
}

@article{wang2024comprehensive,
  title   = {A Comprehensive Survey of Continual Learning: Theory, Method and Application},
  author  = {Wang, Liyuan and Zhang, Xingxing and Su, Hang and Zhu, Jun},
  journal = {IEEE Transactions on Pattern Analysis and Machine Intelligence},
  volume  = {46},
  number  = {8},
  pages   = {5362--5383},
  year    = {2024},
  doi     = {10.1109/TPAMI.2024.3367329},
  url     = {https://arxiv.org/abs/2302.00487}
}

@inproceedings{lin2023theory,
  title     = {Theory on Forgetting and Generalization of Continual Learning},
  author    = {Lin, Sen and Ju, Peizhong and Liang, Yingbin and Shroff, Ness},
  booktitle = {Proceedings of the 40th International Conference on Machine Learning},
  series    = {Proceedings of Machine Learning Research},
  volume    = {202},
  pages     = {21078--21100},
  year      = {2023},
  publisher = {PMLR},
  url       = {https://proceedings.mlr.press/v202/lin23f.html}
}

@inproceedings{ding2024understanding,
  title     = {Understanding Forgetting in Continual Learning with Linear Regression},
  author    = {Ding, Meng and Ji, Kaiyi and Wang, Di and Xu, Jinhui},
  booktitle = {Proceedings of the 41st International Conference on Machine Learning},
  series    = {Proceedings of Machine Learning Research},
  volume    = {235},
  pages     = {10978--11001},
  year      = {2024},
  publisher = {PMLR},
  url       = {https://proceedings.mlr.press/v235/ding24c.html}
}

@inproceedings{evron2023continual,
  title     = {Continual Learning in Linear Classification on Separable Data},
  author    = {Evron, Itay and Moroshko, Edward and Buzaglo, Gon and Khriesh, Maroun and Marjieh, Badea and Srebro, Nathan and Soudry, Daniel},
  booktitle = {Proceedings of the 40th International Conference on Machine Learning},
  series    = {Proceedings of Machine Learning Research},
  volume    = {202},
  pages     = {9440--9484},
  year      = {2023},
  publisher = {PMLR},
  url       = {https://proceedings.mlr.press/v202/evron23a.html}
}

@inproceedings{li2025theory,
  title     = {Theory on Mixture-of-Experts in Continual Learning},
  author    = {Li, Hongbo and Lin, Sen and Duan, Lingjie and Liang, Yingbin and Shroff, Ness B.},
  booktitle = {International Conference on Learning Representations (ICLR)},
  year      = {2025},
  url       = {https://proceedings.iclr.cc/paper_files/paper/2025/hash/17a234c91f746d9625a75cf8a8731ee2-Abstract-Conference.html}
}

@inproceedings{chaudhry2018riemannian,
  title     = {Riemannian Walk for Incremental Learning: Understanding Forgetting and Intransigence},
  author    = {Chaudhry, Arslan and Dokania, Puneet K. and Ajanthan, Thalaiyasingam and Torr, Philip H. S.},
  booktitle = {European Conference on Computer Vision (ECCV)},
  pages     = {532--547},
  year      = {2018},
  url       = {https://openaccess.thecvf.com/content_ECCV_2018/html/Arslan_Chaudhry__Riemannian_Walk_ECCV_2018_paper.html}
}

@inproceedings{chaudhry2019efficient,
  title     = {Efficient Lifelong Learning with {A-GEM}},
  author    = {Chaudhry, Arslan and Ranzato, Marc'Aurelio and Rohrbach, Marcus and Elhoseiny, Mohamed},
  booktitle = {International Conference on Learning Representations (ICLR)},
  year      = {2019},
  url       = {https://openreview.net/forum?id=Hkf2_sC5FX}
}

@article{williams1992simple,
  title   = {Simple statistical gradient-following algorithms for connectionist reinforcement learning},
  author  = {Williams, Ronald J.},
  journal = {Machine Learning},
  volume  = {8},
  number  = {3--4},
  pages   = {229--256},
  year    = {1992},
  url     = {https://doi.org/10.1007/BF00992696}
}

@inproceedings{sutton2000policy,
  title     = {Policy gradient methods for reinforcement learning with function approximation},
  author    = {Sutton, Richard S. and McAllester, David and Singh, Satinder and Mansour, Yishay},
  booktitle = {Advances in Neural Information Processing Systems (NeurIPS)},
  year      = {1999},
  url       = {https://proceedings.nips.cc/paper_files/paper/1999/hash/464d828b85b0bed98e80ade0a5c43b0f-Abstract.html}
}

@article{agarwal2021theory,
  title   = {On the Theory of Policy Gradient Methods: Optimality, Approximation, and Distribution Shift},
  author  = {Agarwal, Alekh and Kakade, Sham M. and Lee, Jason D. and Mahajan, Gaurav},
  journal = {Journal of Machine Learning Research},
  volume  = {22},
  number  = {98},
  pages   = {1--76},
  year    = {2021},
  url     = {https://jmlr.org/papers/v22/19-736.html}
}

@inproceedings{kakade2001natural,
  title     = {A Natural Policy Gradient},
  author    = {Kakade, Sham M.},
  booktitle = {Advances in Neural Information Processing Systems (NeurIPS)},
  volume    = {14},
  pages     = {1531--1538},
  year      = {2001},
  url       = {https://papers.nips.cc/paper_files/paper/2001/hash/4b86abe48d358ecf194c56c69108433e-Abstract.html}
}

@article{shao2024deepseekmath,
  title   = {{DeepSeekMath}: Pushing the Limits of Mathematical Reasoning in Open Language Models},
  author  = {Shao, Zhihong and Wang, Peiyi and Zhu, Qihao and Xu, Runxin and Song, Junxiao and Bi, Xiao and Zhang, Haowei and Zhang, Mingchuan and Li, Y. K. and Wu, Y. and Guo, Daya},
  journal = {arXiv preprint arXiv:2402.03300},
  year    = {2024},
  url     = {https://arxiv.org/abs/2402.03300}
}

@inproceedings{zhou2023lima,
  title     = {{LIMA}: Less Is More for Alignment},
  author    = {Zhou, Chunting and Liu, Pengfei and Xu, Puxin and Iyer, Srini and Sun, Jiao and Mao, Yuning and Ma, Xuezhe and Efrat, Avia and Yu, Ping and Yu, Lili and Zhang, Susan and Ghosh, Gargi and Lewis, Mike and Zettlemoyer, Luke and Levy, Omer},
  booktitle = {Advances in Neural Information Processing Systems},
  volume    = {36},
  year      = {2023},
  url       = {https://proceedings.neurips.cc/paper_files/paper/2023/hash/ac662d74829e4407ce1d126477f4a03a-Abstract-Conference.html}
}

@inproceedings{gudibande2024false,
  title     = {The False Promise of Imitating Proprietary Language Models},
  author    = {Gudibande, Arnav and Wallace, Eric and Snell, Charlie and Geng, Xinyang and Liu, Hao and Abbeel, Pieter and Levine, Sergey and Song, Dawn},
  booktitle = {International Conference on Learning Representations (ICLR)},
  year      = {2024},
  url       = {https://proceedings.iclr.cc/paper_files/paper/2024/hash/4db16435e3a5a2ef3fc39b8f0d12498d-Abstract-Conference.html}
}

@inproceedings{chu2025sft,
  title     = {{SFT} Memorizes, {RL} Generalizes: A Comparative Study of Foundation Model Post-training},
  author    = {Chu, Tianzhe and Zhai, Yuexiang and Yang, Jihan and Tong, Shengbang and Xie, Saining and Schuurmans, Dale and Le, Quoc V. and Levine, Sergey and Ma, Yi},
  booktitle = {International Conference on Machine Learning (ICML)},
  series    = {Proceedings of Machine Learning Research},
  volume    = {267},
  pages     = {10818--10838},
  year      = {2025},
  url       = {https://proceedings.mlr.press/v267/chu25c.html}
}

@inproceedings{shenfeld2026razor,
  title     = {{RL}'s Razor: Why Online Reinforcement Learning Forgets Less},
  author    = {Shenfeld, Idan and Pari, Jyothish and Agrawal, Pulkit},
  booktitle = {International Conference on Learning Representations (ICLR)},
  year      = {2026},
  url       = {https://proceedings.iclr.cc/paper_files/paper/2026/hash/618c95f4557c15b253fb0e6f548ea0c0-Abstract-Conference.html}
}

@inproceedings{shenfeld2026selfdistillation,
  title     = {Self-Distillation Enables Continual Learning},
  author    = {Shenfeld, Idan and Damani, Mehul and H{\"u}botter, Jonas and Agrawal, Pulkit},
  booktitle = {International Conference on Machine Learning (ICML)},
  year      = {2026},
  url       = {https://arxiv.org/abs/2601.19897}
}

@inproceedings{chen2026retaining,
  title     = {Retaining by Doing: The Role of On-Policy Data in Mitigating Forgetting},
  author    = {Chen, Howard and Razin, Noam and Narasimhan, Karthik and Chen, Danqi},
  booktitle = {International Conference on Machine Learning (ICML)},
  year      = {2026},
  url       = {https://arxiv.org/abs/2510.18874}
}

@article{balasubramanian2026quantitative,
  title   = {A Quantitative Characterization of Forgetting in Post-Training},
  author  = {Balasubramanian, Krishnakumar and Kasiviswanathan, Shiva Prasad},
  journal = {arXiv preprint arXiv:2603.12163},
  year    = {2026},
  doi     = {10.48550/arXiv.2603.12163},
  url     = {https://arxiv.org/abs/2603.12163}
}

@inproceedings{lai2026reinforcement,
  title     = {Reinforcement Fine-Tuning Naturally Mitigates Forgetting in Continual Post-Training},
  author    = {Lai, Song and Zhao, Haohan and Feng, Rong and Ma, Changyi and Liu, Wenzhuo and Zhao, Hongbo and Lin, Xi and Yi, Dong and Zhang, Qingfu and Liu, Hongbin and Meng, Gaofeng and Zhu, Fei},
  booktitle = {International Conference on Machine Learning (ICML)},
  year      = {2026},
  url       = {https://arxiv.org/abs/2507.05386}
}

@inproceedings{zhang2026data,
  title     = {Why Reinforcement Fine-Tuning Preserves Prior Knowledge Better: A Data Perspective},
  author    = {Zhang, Zhihao and Dong, Qiaole and Zhang, Qi and Zhou, Enyu and Zhao, Jun and Xi, Zhiheng and Jin, Senjie and Fan, Xiaoran and Zhou, Yuhao and Wu, Mingqi and Fu, Yanwei and Ji, Tao and Gui, Tao and Huang, Xuanjing and Chen, Kai},
  booktitle = {International Conference on Learning Representations (ICLR)},
  year      = {2026},
  url       = {https://openreview.net/pdf/3b5fda354275f890c9404d61311720aac059b79d.pdf}
}

@inproceedings{ouyang2022training,
  title     = {Training language models to follow instructions with human feedback},
  author    = {Ouyang, Long and Wu, Jeff and Jiang, Xu and Almeida, Diogo and Wainwright, Carroll L. and Mishkin, Pamela and Zhang, Chong and Agarwal, Sandhini and Slama, Katarina and Ray, Alex and others},
  booktitle = {Advances in Neural Information Processing Systems (NeurIPS)},
  year      = {2022},
  url       = {https://arxiv.org/abs/2203.02155}
}

@article{kirkpatrick2017overcoming,
  title   = {Overcoming catastrophic forgetting in neural networks},
  author  = {Kirkpatrick, James and Pascanu, Razvan and Rabinowitz, Neil and Veness, Joel and Desjardins, Guillaume and Rusu, Andrei A. and Milan, Kieran and Quan, John and Ramalho, Tiago and Grabska-Barwinska, Agnieszka and Hassabis, Demis and Clopath, Claudia and Kumaran, Dharshan and Hadsell, Raia},
  journal = {Proceedings of the National Academy of Sciences},
  volume  = {114},
  number  = {13},
  pages   = {3521--3526},
  year    = {2017},
  doi     = {10.1073/pnas.1611835114},
  url     = {https://arxiv.org/abs/1612.00796}
}
